\documentclass[twoside,11pt]{article}

\usepackage{blindtext}

\usepackage{jmlr2e}

\makeatletter
\renewcommand{\jmlrheading}[5]{}    %
\renewcommand{\ShortHeadings}[2]{}  %
\renewcommand{\editor}[1]{}         %

\def\@starteditor{}   %
\def\@editor{}        %
\def\@endeditor{}     %
\renewcommand{\firstpageno}[1]{}    %
\makeatother

\makeatletter
\renewcommand{\maketitle}{%
  \par
  \begingroup
    \renewcommand{\thefootnote}{\fnsymbol{footnote}}%
    \def\@makefnmark{\hbox to 0pt{$^{\@thefnmark}$\hss}}%
    \long\def\@makefntext##1{\parindent 1em\noindent
            \hb@xt@1.8em{##1\hss}##1}%
    \if@twocolumn
      \ifnum \col@number=\@ne
        \@maketitle
      \else
        \twocolumn[\@maketitle]%
      \fi
    \else
      \newpage
      \global\@topnum\z@   %
      \@maketitle
    \fi
    \thispagestyle{plain}\@thanks
  \endgroup
  \setcounter{footnote}{0}%
  \renewcommand{\thefootnote}{\arabic{footnote}}%
}
\makeatother

\usepackage{color}

\newcommand{\cA}{\mathcal{A}}

\newcommand{\cD}{\mathcal{D}}
\newcommand{\cE}{\mathcal{E}}
\newcommand{\cF}{\mathcal{F}}
\newcommand{\cG}{\mathcal{G}}
\newcommand{\cH}{\mathcal{H}}

\newcommand{\cK}{\mathcal{K}}

\newcommand{\cN}{\mathcal{N}}

\newcommand{\cS}{\mathcal{S}}

\newcommand{\cX}{\mathcal{X}}
\newcommand{\cY}{\mathcal{Y}}

\newcommand{\clip}[1]{\operatorname{clip_m}\left( #1 \right)}
\newcommand{\clipg}[1]{\operatorname{clip_\gamma}\left( #1 \right)}

\newcommand{\bP}{\mathbb{P}}

\newcommand{\bR}{\mathbb{R}}

\newcommand{\En}{\mathbb{E}}

\newcommand{\wt}[1]{\widetilde{#1}}
\newcommand{\wh}[1]{\widehat{#1}}
\newcommand{\wb}[1]{\widebar{#1}}

\newcommand{\sign}{\operatorname{sign}}
\newcommand{\ind}[1]{\mathbbm{1}(#1)}

\newcommand{\veps}{\varepsilon}
\newcommand{\ldef}{\vcentcolon=}

\newcommand{\Unif}{\mathrm{Unif}}

\newcommand{\KL}{\mathrm{KL}}

\newcommand{\poly}{\mathrm{poly}}

\usepackage{mathtools}
\DeclarePairedDelimiter{\abs}{\lvert}{\rvert} %
\DeclarePairedDelimiter{\brk}{[}{]}

\DeclarePairedDelimiter{\prn}{(}{)}
\DeclarePairedDelimiter{\norm}{\|}{\|}
\DeclarePairedDelimiter{\inner}{\langle}{\rangle}
\DeclarePairedDelimiter{\set}{\{}{\}}

\DeclareMathOperator*{\argmin}{arg\,min} %

\def\medskip{\vskip 10 pt}
\def\bigskip{\vskip 15 pt}

\newcommand{\revindent}[1][1]{\hspace{#1in}&\hspace{-#1in}}
\newcommand{\multiline}[1]{\parbox[t]{\dimexpr\linewidth-\algorithmicindent}{#1}}

\def\texitem#1{\par\vspace{5pt}
\noindent\hangindent 20pt
\hbox to 20pt {\hss #1 ~}\ignorespaces}

\usepackage{dirtytalk}
\usepackage{enumitem}
\newcommand{\fhat}{\wh{f}}

\newcommand{\Ber}{\mathrm{Ber}}

\newcommand{\thstar}{{\theta^\star}}
\newcommand{\Hess}{\mathsf{H}}

\newcommand{\E}{\mathbb{E}}
\newcommand{\thetahat}{\hat{\theta}}

\usepackage{algorithm}
\usepackage{algpseudocode}
\usepackage{algorithmicx}

\input{widebar}

\usepackage{framed}
\usepackage{comment}

\begin{document}

\title{Refined Risk Bounds for Unbounded Losses via\\ Transductive Priors}

\author{\name Jian Qian \email jianqian@mit.edu \\
\addr Department of Electrical Engineering and Computer Science \\
Massachusetts Institute of Technology\\
Cambridge, MA 02139, USA
\AND
\name Alexander Rakhlin \email rakhlin@mit.edu \\
\addr Department of Brain and Cognitive Sciences\\
Massachusetts Institute of Technology\\
Cambridge, MA 02139, USA
\AND 
\name Nikita Zhivotovskiy \email zhivotovskiy@berkeley.edu\\
\addr Department of Statistics\\
University of California, Berkeley\\
Berkeley, CA 94720, USA
}

\maketitle

\begin{abstract}
We revisit the sequential variants of linear regression with the squared loss, classification problems with hinge loss, and logistic regression, all characterized by unbounded losses in the setup where no assumptions are made on the magnitude of design vectors and the norm of the optimal vector of parameters. The key distinction from existing results lies in our assumption that the set of design vectors is known in advance (though their order is not), a setup sometimes referred to as \emph{transductive} online learning. While this assumption might seem similar to fixed design regression or denoising, we demonstrate that the sequential nature of our algorithms allows us to convert our bounds into statistical ones with random design without making any additional assumptions about the distribution of the design vectors—an impossibility for standard denoising results. Our key tools are based on the exponential weights algorithm with carefully chosen transductive (design-dependent) priors, which exploit the full horizon of the design vectors, as well as additional aggregation tools that address the possibly unbounded norm of the vector of the optimal solution.

Our classification regret bounds have a feature that is only attributed to bounded losses in the literature: they depend solely on the dimension of the parameter space and on the number of rounds, independent of the design vectors or the norm of the optimal solution. For linear regression with squared loss, we further extend our analysis to the sparse case, providing sparsity regret bounds that depend additionally only on the magnitude of the response variables. We argue that these improved bounds are specific to the transductive setting and unattainable in the worst-case sequential setup. 
Our algorithms, in several cases, have polynomial time approximations and reduce to sampling with respect to log-concave measures instead of aggregating over hard-to-construct $\varepsilon$-covers of classes.
\end{abstract}


\section{Introduction}
Sequential learning algorithms are widely used in practice, even when the entire training sample is known in advance (\emph{batch} setup), with stochastic gradient descent being one of the most well-known examples. The model we study in this paper represents an intermediate approach, combining sequential procedures with their \emph{batch} counterparts. A standard sequential regression problem can be described as follows: at round $t$, one observes the design vector $x_t$ and makes a prediction $\widehat{y}_t$ based on previous observations, aiming to approximate the true value $y_t$, which is revealed only after the prediction $\widehat{y}_t$ is made. 
Our setup, often referred to as \emph{transductive online learning}, provides the learner with additional information, as the entire set of design vectors $x_1, \ldots, x_T$ (though not their order) is revealed to the learner in advance and can be utilized. 
In the terminology of \cite{vapnik2013nature}, transduction refers to situations where the aim is not to infer a general prediction rule for arbitrary future points, but to predict the labels of a specific, fixed set of points known beforehand. This exactly matches our setting, where the goal is to assign labels to a predetermined sequence rather than to predict each point in an arbitrary sequence as it arrives. 

One of our primary motivations for studying the transductive online model stems from its connection to the batch setting, where the entire training sample is known in advance. It is well established, due to the so-called \emph{online-to-batch} conversion arguments, that regret bounds in sequential settings correspond to risk bounds in statistical settings. We make the following simple observation: instead of proving regret bounds in the more challenging \emph{worst-case sequential setting}, one can instead prove regret bounds in the \emph{easier transductive sequential setting} and obtain statistical risk bounds with provably better guarantees. Our analysis leverages certain versions of the classical exponential weights algorithm, where the prior distributions depend on the entire set of design vectors $x_1, \ldots, x_T$, thus motivating the term \emph{transductive priors}.

Our key focus is on the separation between the regret bounds achievable in the transductive setting compared to the standard online setting, particularly in the case of linear regression with squared loss, including sparse linear regression, as well as classification problems like logistic regression and classification with hinge loss. Some of our key findings are simple enough to be formulated in the introduction. Let us start with classification with logarithmic loss: in the standard online setup \citep{hazan2014logistic, foster2018logistic}, we observe the vectors $x_t \in \mathbb{R}^d$ round by round, where $t = 1, \ldots , T$, and assign the probabilities $\widehat{p}(x_t, 1)$ to label $1$ and $\widehat{p}(x_t, -1)$ to label $-1$, after which we observe the true label $y_t \in \{1, -1\}$. For some $b > 0$, our aim is to minimize the regret with respect to the logarithmic loss, that is,
\begin{equation}
\label{eq:regretdef}
\operatorname{Reg}_T = \sum\limits_{t = 1}^T -\log(\widehat{p}(x_t, y_t)) - \inf\limits_{\theta \in \mathbb{R}^d, \|\theta\|_2 \le b} \sum\limits_{t = 1}^T -\log(\sigma(y_t\langle x_t, \theta\rangle)),
\end{equation}
where $\sigma(z) = 1/(1 + \exp(-z))$ is the sigmoid function. Here, for the sake of example, we assumed $\|x_t\|_2 \le 1$, though importantly, the results in this paper will completely bypass \emph{any boundedness assumptions} on the design vectors. The approaches based on choosing the \emph{proper} predictor, that is, a predictor corresponding to the logistic regression model, $\widehat{p}(x_t, 1) = \sigma(\langle x_t, \widehat{\theta}_t\rangle)$,  where $\widehat{\theta}_t$ is independent of $x_t$, are quite limited. As shown by \citet*{hazan2014logistic}, and rather informally skipping some assumptions, the following lower bound holds:
\[
\operatorname{Reg}_T =  \Omega\left(\min\left\{d\exp(b), b\sqrt{T}\right\}\right).
\]
Motivated by the question of \cite{mcmahan2012open}, a number of works have focused on improving the exponential dependence on $b$ in the above bound. A significant improvement on the regret bound can be achieved by using \emph{improper} predictions, namely mixtures of sigmoid functions, as shown by the results of \citet*{kakade2004online, foster2018logistic} (see also \citep[Section 12.2]{cesa2006prediction}). Their results generate a sequence of probability assignments $\widehat{p}(x_t, \cdot)$ whose regret, as defined by \eqref{eq:regretdef}, satisfies
\begin{equation}
\label{eq:logtregret}
    \operatorname{Reg}_T = O(d\log(b\cdot T)),
\end{equation}
demonstrating only logarithmic dependence on $b$. Moreover, it appears that this result is generally unimprovable, as the lower bound shows that $b < \infty$ is indeed required for regret that scales better than $O(T)$. Remarkably, the constructions by \citet{foster2018logistic} demonstrating the necessity of the dependence on $b$ are related to the classical example in sequential binary classification, due to \cite{littlestone1988learning}, where the thresholds are impossible to learn with a finite number of mistakes. At the same time, thresholds are easily learnable in the batch setup, or alternatively in the sequential setup when the sequence $\{x_1, \ldots, x_T\}$ is known in advance \citep{ben1997online}. As we show in what follows, this favorable phenomenon extends well beyond the example of realizable binary classification and can be applied to classification and regression problems with several loss functions in the agnostic case.

It is important to note that the assumption \(b < \infty\) in \eqref{eq:logtregret} is not natural. For instance, in the favorable scenario where all data points belong to the same class \(+1\), the optimal parameter vector \(\theta^{\star}\) could have an arbitrarily large norm. As discussed earlier, our focus in this paper is on the transductive setup, where the set of vectors \(\{x_1, \ldots, x_T\}\) is known to the learner beforehand. By denoting the regret in this setup as \(\operatorname{Reg}_T^{\operatorname{trd}}\), we demonstrate that\looseness=-1
\begin{equation}
\label{eq:reglogloss}
\operatorname{Reg}_T^{\operatorname{trd}} = O(d\log(T)),
\end{equation}
where there is \emph{no dependency} on the norm of the optimizer \(\theta^\star\) or the norms of the \(x_t\)'s. Notably, the probability assignments generated by our algorithm do not require a computationally prohibitive discretization of the function space, which is often seen in density estimation and regression approaches \citep{yang1999information, vovk2006metric, bilodeau2023minimax}. Instead, we leverage variants of continuous exponential weights, which, while not always practical, can be implemented in polynomial time in certain scenarios.

\begin{framed}
While the distinction between the regret bounds we prove in this paper—comparing the existing $\operatorname{Reg}_T = O(d\log(b\cdot T))$ with $\operatorname{Reg}_T^{\operatorname{trd}} = O(d\log(T))$—might appear quantitatively minor, it is, in fact, qualitatively significant. Our findings show an essential separation: they reveal what is learnable, with sublinear regret achievable in one framework, namely the transductive setup (which also extends to the classical batch setting, as discussed below), versus what remains unlearnable in the traditional online learning setup. \looseness=-1
\end{framed}

Additionally, we introduce a variant of the online-to-batch conversion argument to illustrate the application of our results in the statistical setup. In the statistical learning context, one observes i.i.d. sample \((X_t, Y_t)_{t = 1}^T\) of copies of \((X, Y)\), where \(X \in \mathbb{R}^d\) and \(Y \in \mathbb{R}\). Our analysis shows that there exists a probability assignment \(\widetilde{p}(\cdot, \cdot)\), constructed based on the training sample \((X_t, Y_t)_{t = 1}^T\), such that
\begin{equation}
    \label{eq:excessrisklogistic}
    \E[-\log(\widetilde{p}(X, Y))] - \inf\limits_{\theta \in \mathbb{R}^d}\E[-\log(\sigma(Y\langle X, \theta\rangle))] = O\left(\frac{d\log T}{T}\right),
\end{equation}
with expectations taken over both the training sample \((X_t, Y_t)_{t = 1}^T\) and an independent test observation \((X,Y)\). A key feature of our excess risk bound is that the right-hand side of \eqref{eq:excessrisklogistic} is \emph{completely independent} of the distributions of \(X\) and \(Y\). Moreover, the curvature of the logarithmic loss function enables the \(d\log(T)/T\) convergence rate without requiring additional realizability/low noise assumptions. Despite the extensive literature on logistic regression in the statistical framework (for comparison, see, e.g., \citep{bach2010self, hazan2014logistic, plan2017high, ostrovskii2021finite, puchkin2023exploring, kuchelmeister2024finite, hsu2024sample} for a non-exhaustive list of existing risk upper bounds in this context), this fully assumption-free convergence rate has not been previously established even for mixtures of sigmoids proposed in \citep{kakade2004online, foster2018logistic} and other batch-specific results as in \citep{mourtada2022improper, vijaykumar2021localization, van2023high}. We emphasize again that our approach relies on using variants of the continuous exponential weights algorithm.

Our findings go well beyond the logarithmic loss. For example, we consider the case of sparse linear regression with \emph{squared loss}. Specifically, in the transductive setting with squared loss, assuming $|y_t| \leq m$ for all $t = 1, \ldots, T$ and comparing against any $\theta^{\star}$ with $\|\theta^{\star}\|_0 \leq s$, we prove a regret bound 
of \looseness=-1
\[
\operatorname{Reg}_T^{\operatorname{trd}} = O\left(sm^2\log\left(\frac{dT}{s\kappa_s^2}\right)\right),
\]
where the dependence on the compatibility constant $\kappa_s$ is logarithmic, contrasting with the polynomial dependence found in most known results for predictors that satisfy sparsity regret or risk bounds. 

Furthermore, in the setup of classification with \emph{hinge loss}, we prove the so-called \emph{first-order} regret bounds. In more detail, we observe the vectors \(x_t \in \mathbb{R}^d\) round by round, where \(t = 1, \ldots , T\), and output a real-valued score \(\widehat{f}(x_t)\), trying to minimize, for some \(\gamma > 0\), the total hinge loss:
\[
\sum\nolimits_{t = 1}^T\frac{(\gamma - y_t\widehat{f}(x_t))_{+}}{\gamma}.
\]
In the transductive setup, where we know the set \(\{x_1, \ldots, x_T\}\), if additionally the data is linearly separable with margin \(\gamma\), the result of our \cref{thm:slab-experts-for-hinge} implies, in particular, that
\[
\sum\limits_{t = 1}^T\frac{(\gamma - y_t\widehat{f}(x_t))_{+}}{\gamma} = O\left(d \log(T)\right).
\]
Prediction is again achieved using a version of the continuous exponential weights algorithm, and the right-hand side of the above cumulative error bound does not depend on \(\gamma\), which need not hold in the inductive (unknown-set) setup.

\subsection{Key contributions and structure of the paper}

Our key contributions can be summarized through the regret bounds we derive for the transductive online learning framework, considering the squared loss, logistic loss, and hinge loss. These bounds explicitly reveal distinctions between the transductive and inductive setups across different loss functions. Moreover, our algorithm exhibits computational advantages through sampling from log-concave densities. Concretely, we summarize them as follows:\looseness=-1
\begin{enumerate}[label=$\bullet$,leftmargin=5mm]
    \item In \cref{sec:linearregression}, we revisit the classical problem of online squared loss regression. We recover the optimal rate of $O(dm^2\log T)$ for the transductive setup originally established by \cite{gaillard2019uniform} with a transductive prior without making any assumptions about the boundedness of covariates or the norm of the optimal solution (comparator). 
    In the sparse linear regression setting,  we introduce a sparse prior that achieves regret bounds independent of the norm of $\theta^{\star}$ via the evaluation of Gaussian-type integrals.
    For completeness, we also include results by combining our transductive prior and the discrete sparse prior proposed by \citet{rigollet2011exponential,rigollet2012sparse} in \cref{app:additional-sparse}, where a better regret bound for sparse linear regression is achieved, albeit through a method involving sampling from discrete measures, which may require a prohibitively large number of samples from these distributions in the worst case.

    Then, we present a hard case where the celebrated predictor from \citet{vovk2001competitive,azoury2001relative} suffers $\Omega(T)$ regret, which shows a separation between the transductive setup and the inductive setup. 
    Finally, we show that such a separation exists for any prediction strategy when comparing with the clipped version of the squared loss. 
    \item In \cref{sec:logistic}, we cover transductive online logistic regression. In this section, we present a construction of slab-experts that combine predictions of exponential weights algorithms with a transductive Gaussian prior and linear separators. We show that the optimal slab-expert always obtains the optimal \emph{assumption-free} bound of $O(d\log T)$, and by mixability, such a bound is achievable in general with the exponential weights algorithm over the slab-experts. We also mention that such assumption-free bounds are not obtainable in the inductive setup, establishing a desired separation.
    \item In \cref{sec:non-curve}, we investigate the hinge loss in both the inductive and transductive setup. The hinge loss is unbounded and not mixable. Nevertheless, we present the first-order regret bound scaling as $O(d\log(bT/\gamma))$ in the realizable case, where $\gamma$ is the margin of separation in the inductive case and $b$ is the norm of the optimal weight vector. This regret bound is the parametric counterpart of the celebrated $O(1/\gamma^2)$ mistake bound for perceptron. We then combine the slab-expert idea with the aforementioned clipping trick to obtain an assumption-free first-order regret bound of $O(d\log T)$ in the transductive setup. The separation between the inductive and transductive setups also exists for the hinge loss.
    \item In \cref{sec:computational}, we highlight our computational benefit through sampling from log-concave densities. More concretely, since the squared loss and the logistic loss are convex, the exponential weights algorithm, therefore, predicts log-concave densities. We showcase how to compute the exponential weights algorithm for sparse linear regression and logistic regression. For the latter, we demonstrate how to sample the slab-experts, though our constructions lead to computational complexities exponential in dimension in the worst case.
    \item In \cref{sec:online-to-batch}, the online-to-batch conversion is restated with the reduction from the transductive online learning setup to the batch setup. Our results \cref{sec:linearregression,sec:logistic,sec:non-curve}, together with this reduction, bring forth new batch results for corresponding losses. Moreover, for the special case of logistic regression and under mild additional assumptions on the design, we obtain a polynomial time algorithm that gives a $O((d+\log T)\log T/T)$ excess risk upper bound.
\end{enumerate}

Alongside, in \cref{sec:final}, we discuss several directions of future work. In \cref{sec:linearrevisit}, we revisit sequential linear regression and discuss short proofs of several existing bounds.

\subsection{Contextual review of related literature}

In this section, we provide an overview focused on the most relevant papers. Given the vast volume of existing literature, we concentrate only on the most pertinent directions, with additional references typically found in the respective papers. Further discussion of related literature is deferred to the corresponding sections.

Our results cover three unbounded loss functions, namely squared loss, logarithmic loss, and hinge loss, using a unified approach. It is worth mentioning that in the case of the squared loss, the results similar to ours are relatively developed. In the context of squared loss, the closest papers to our work are \citep{bartlett2015minimax, gaillard2019uniform}, which were the first to emphasize the improved rates for the celebrated sequential online regression predictor of \citet{vovk2001competitive, azoury2001relative}.
In fact, the paper \citep{gaillard2019uniform} achieves the same regret bounds we achieve with our methods. However, in the case of the squared loss, we make some further extensions and work with the sparse case. 
When working with the squared loss,
a result that is somewhat close in spirit to our result is the sequential bound of \cite{gerchinovitz2011sparsity}, which is a version of the earlier non-sequential fixed design sparse regression papers of \cite{dalalyan2008aggregation,dalalyan2012mirror,dalalyan2012sparse}. We make a more detailed comparison in \cref{sec:sparselinear}.

The case of logistic regression is more challenging and requires further work. Analyzing this problem requires some form of boundedness for the vector of optimal solutions $\theta^\star$. Because this limitation is mainly recognized in the case of proper probability assignments, where the assigned probability corresponds to the sigmoid function \citep{hazan2014logistic}, our approach builds on \citep{kakade2004online, foster2018logistic} (see also \cite[Section 11.10]{cesa2006prediction}) and uses the mixture approach or improper predictors, where assigned probabilities are not necessarily given by sigmoid functions. This technique allows for improving the dependence on the norm of the optimal solution from being \emph{exponential} with respect to the norm of $\theta^\star$ to a \emph{logarithmic} dependence on this norm. However, without additional assumptions, the lower bounds in \citep{foster2018logistic} require that the norm of $\theta^\star$ is bounded. Our results show that if the design vectors are known in advance, an exponential weighting algorithm-based predictor can achieve a regret bound dependent only on the dimension and the number of rounds.
The only studies we know that address unbounded optimal solutions are the recent works by \citep{drmota2024unbounded, drmota2024precise}. These works consider a similar transductive online setup with $x_t$ generated uniformly at random on the sphere and compute the precise asymptotics of the first term of the expected regret and leave the additive asymptotic term in their bounds. While asymptotic results are standard in information theory, our results provide deterministic regret bounds without structural assumptions on design vectors, are finite samples, but do not recover the asymptotically optimal constant in front of the leading term $d\log(T)$. The methods differ, as we focus on adapting exponential weights algorithms with transductive priors, which also usually leads to polynomial time implementation, whereas \citet{drmota2024unbounded, drmota2024precise} analyze Shtarkov’s sum using analytic combinatorics tools. We note that the analysis of Shtarkov's sum is not known to lead to polynomial time predictors. We also refer to the paper \citep{wu2022sequential}, where connections between online transductive setup and the result of \citet{shtar1987universal} are made.\looseness=-1

Choosing the right prior for exponential weights is a classical question studied extensively in the information theory literature, primarily focusing on logarithmic loss. When computing the regret with respect to realizable i.i.d. data, the works of \cite{clarke1990information, clarke1994jeffreys} show that the \emph{Jeffreys} prior, the prior induced by the Fisher information, provides asymptotically optimal expected regret bounds. The classical work of \cite{rissanen1996fisher} demonstrates that the i.i.d. data assumption is not necessary to achieve optimal asymptotic regret with the Jeffreys prior. This is done under the assumption that the maximum likelihood estimates satisfy the Central Limit Theorem, a weaker assumption than i.i.d. data. These important results are asymptotic in nature and do not fully align with our goals. In some special cases, such as probability assignment, the analysis of the renowned estimator by \cite{krichevsky1981performance} can be seen as follows. First, one computes the \emph{asymptotically} optimal Jeffreys prior, but then uses it to obtain sharp non-asymptotic regret bounds for the exponential weights algorithm (see the derivations in \cite[Section 9.7]{cesa2006prediction}). Another approach to non-asymptotic bounds with the Jeffreys prior in exponential weights algorithms with i.i.d. data is discussed in \citep{barron1999information}. This approach derives non-asymptotic expected regret bounds under a local quadratic behavior of the Fisher information. Unfortunately, the method as is does not work in the context of logistic regression with unbounded target parameters due to possible very imprecise quadratic approximations when the target vector has coordinates that are too large. However, a modification of this idea, involving control over large coordinates of the target vector of parameters, is used in our analysis.

In the machine learning literature, transductive online learning has been explored in several papers \citep{ben1997online, kakade2005batch, cesa2013efficient}. These studies primarily demonstrate an equivalence between batch learning in classification tasks using binary loss---assuming random and i.i.d. samples---and transductive settings, focusing on \emph{expected regret} rather than worst-case scenarios. Simplifying these results, they suggest that a computationally efficient i.i.d. learner can be converted to an efficient transductive online learner. Our work, in contrast, focuses on establishing deterministic regret bounds for several unbounded losses, which in some cases can lead to computationally efficient learners in the batch setup.

Results closely related to our regret bounds are known in the batch setup, where one observes the full sample of i.i.d. points and focuses on bounding the \emph{excess risk}. The result of \cite{forster2002relative} complements the sequential regression bound of \cite{vovk2001competitive} (see also \citep{vavskevivcius2023suboptimality} for the relation of this result to the standard least squares estimator, and \citep{mourtada2022distribution} for the high probability version of their excess risk bound) by focusing on the batch setup. The excess risk bound of \cite{forster2002relative} shares similar properties with our regret bound: it depends neither on the properties of the design vectors nor on the norm of the optimal parameter $\theta^\star$. In the context of removing dependencies on the norm of $\theta^\star$, the recent work on \emph{localized priors} (originally introduced in \citep{catoni2007pacbayes}) in the PAC-Bayesian scenario \cite{mourtada2023local} is particularly relevant. The advantage of that work is that, specifically in the batch setting, it not only removes the dependence on the norm of the optimal parameter $\theta^\star$ but also eliminates superfluous logarithmic factors related to the sample size, although it still requires additional boundedness assumptions that are easier to circumvent in our sequential setup. \looseness=-1

Another closely related area of literature is the PAC-Bayesian analysis and its specific realization with \emph{exchangeable priors}. Originally introduced in \citep{catoni2007pacbayes, audibert2004phd}, this approach was primarily developed for the analysis of exponential-weights type prediction rules in classification problems, particularly when learning VC classes of classifiers. While selecting a prior distribution for exponential weights can be problematic for discrete but infinite VC classes, the method becomes feasible if one introduces a ghost sample and projects the VC class onto it, rendering the class finite and allowing a uniform distribution over the projection to serve as a prior. We refer to the recent works \citep{grunwald2021pac, mourtada2023local} that utilize exchangeable priors, as well as the survey \citep{alquier2024user} for an extensive discussion of PAC-Bayesian results. Our online learning-based approach is more flexible: it operates in the sequential setup instead of relying on i.i.d. data, and it allows us to easily incorporate improvements due to transductive priors as well as the mixability of the loss. This flexibility is reflected in the difference between $O(\log T)$ and $O(\sqrt{T})$ regret bounds, without the need to assume the boundedness of the loss, which is typically required in PAC-Bayesian analysis.

Finally, we consider the hinge loss. To the best of our knowledge, except for the recent work of \citet{mai2024high} (see also \citep{alquier2016properties}) focusing on sparsity and PAC-Bayesian tools, studies exploring the connections between exponential weights and the hinge loss are limited. In this context, even without access to $x_1, \ldots, x_T$, i.e., in the standard online learning setup, we derive a mistake bound complementary to the classical perceptron error bound. While the classic perceptron algorithm incurs $O(1/\gamma^2)$ mistakes, as demonstrated by the result of \citet{novikoff1962convergence}, our bound of $O(d\log(T/\gamma))$ can be tighter in low-dimensional settings when the margin is particularly small. The closest in spirit result in the literature is the mistake bound of \citet*{gilad2004bayes}, which applies specifically to the realizable case, using a method based on predictions from the center of gravity of the current version space.
Furthermore, in the transductive online setup, where all $x_1, \ldots, x_T$ are known, we show that, in the realizable case, the total normalized hinge error will scale as $d\log(T)$, independent of the margin $\gamma$. This phenomenon is similar to what we obtain for logistic regression, where the regret depends only on the dimension and the number of rounds.\looseness=-1

\paragraph{Notation.}
In what follows, $I_d$ stands for the $d \times d$ identity matrix. We use $\|\cdot\|_{p}$ for $\ell_p$ norm of the vector, $p \ge 1$.
The symbol $\|v\|_0$ denotes the number of non-zero coordinates of $v \in \mathbb{R}^d$.
In particular, a vector $v \in \mathbb{R}^d$ is $s$-sparse if $\|v\|_0 = s$. For a pair of measures $\varphi, \mu$ the symbol $\mathcal{KL}(\varphi||\mu)$ denotes the Kullback-Leibler divergence between $\varphi$ and $\mu$. The symbol $\mathcal{N}(\nu, \sigma^2)$ stands for the Gaussian distribution with mean $\nu$ and variance $\sigma^2$. For any real number $x$, the symbol $\lfloor x \rfloor$ denotes the largest integer smaller or equal to $x$. For any matrix $A$, $\det(A)$ denotes the determinant of the matrix, and $A^\top$ denotes its transpose. For any set $\cX$, the symbol $|\cX|$ denotes its cardinality. We define $O(\cdot), \Omega(\cdot)$, and $\poly(\cdot)$ following standard non-asymptotic big-O notation. For any positive real number $m$, we define the clipping operator \(\clip{x} = \min\{m, \max\{-m, x\}\}\). We use $\ind{\cE}$ to denote the indicator function for event $\cE$. We use $\cS^{d-1}$ to denote the unit sphere in $\bR^d$, i.e., $\cS^{d-1} = \set{x\in \bR^d\mid{} \norm{x}=1}$. For any positive number $a>0$, we use $a\cS^{d-1}$ to denote the unit sphere scaled by $a$, that is, $a\cS^{d-1}= \set{ax\mid{}x\in \cS^{d-1}}$. For any bounded measurable set $\cX$ in $\bR^d$, we use $\Unif(X)$ to denote the uniform distribution on $X$.

\section{Linear Regression with Transductive Priors}
\label{sec:linearregression}
We begin this section by recalling the regret bound in \citep[Theorem 7]{gaillard2019uniform}, which itself is an improvement over the earlier result in \citep{bartlett2015minimax}. This result is the closest in spirit to the bounds presented in this paper and serves as an existing benchmark. \looseness=-1

\begin{theorem}[\citet*{gaillard2019uniform}]
\label{thm:gaillard2019}
Assume that the sequence $\{(x_t, y_t)\}_{t = 1}^T$, where $(x_t,y_t)\in \bR^d\times \bR$ for all $1\leq t\leq T$, is such that $\max_{t}|y_t| \le m$ and that the matrix $\sum\nolimits_{t = 1}^T x_t x_t^{\top}$ is invertible\footnote{This assumption can be easily bypassed. One way around this is by restricting the problem to the linear span of $x_1, \ldots, x_T$.}. For $x \in \mathbb{R}^d$, consider the sequence of estimators 
\[
\widehat{\theta}_{t, x} = \arg\min_{\theta \in \mathbb{R}^d}\left\{\sum_{i = 1}^{t - 1} (y_i - \langle x_i, \theta \rangle)^2 + (\langle x, \theta \rangle)^2 + \lambda \sum\limits_{i = 1}^T (\langle x_i, \theta \rangle)^2 \right\}.
\]
Then, the following regret bound holds:
\[
\sum\limits_{t = 1}^T (y_t - \langle x_t, \widehat{\theta}_{t, x_t} \rangle)^2 \le \inf_{\theta \in \mathbb{R}^d}\left\{\sum_{t = 1}^{T} (y_t - \langle x_t, \theta \rangle)^2 \right\} + \lambda m^2 T + d m^2 \log\left(1 + \frac{1}{\lambda}\right).
\]
In particular, choosing $\lambda = 1/T$ yields
\begin{equation}
\label{eq:gaillardlinear}
\sum\limits_{t = 1}^T (y_t - \langle x_t, \widehat{\theta}_{t, x_t} \rangle)^2 \le \inf_{\theta \in \mathbb{R}^d}\left\{\sum_{t = 1}^{T} (y_t - \langle x_t, \theta \rangle)^2 \right\} + m^2 (1 + d \log(1 + T)).
\end{equation}
\end{theorem}

The key aspect of this regret bound is that the term $m^2 (1 + d \log(1 + T))$ depends neither on the $x_t$-s nor on the norm of $\theta^{\star}$, where $\theta^{\star} = \arg\min\nolimits_{\theta \in \mathbb{R}^d}\sum_{t = 1}^{T} (y_t - \langle x_t, \theta \rangle)^2$. Throughout the paper, we are aiming for results of this type.

The proof of \cref{thm:gaillard2019} in \citep{gaillard2019uniform} employs a linear transformation of the feature vectors, thereby reducing the problem to the classical regret bounds of \citet{vovk2001competitive} and \citet{azoury2001relative} using linear algebraic arguments. Our approach, instead, follows the original path of \citet{vovk2001competitive} by exploiting the mixability of the loss (details to follow), with the main difference being the choice of a prior distribution that depends on the design vectors $x_1, \ldots, x_T$. 
\cref{thm:gaillard2019} can be immediately recovered from the general bounds for the exponential weights algorithm and the mixability of the squared loss discussed in \citep{vovk2001competitive} by choosing the following multivariate Gaussian prior $\mu$ in $\mathbb{R}^d$:
\begin{equation}  
\label{eq:jeffreys}
\mu(\theta) = \left(\frac{\lambda}{\pi}\right)^{d/2}\sqrt{\det\left(\sum\nolimits_{t = 1}^T x_t x_t^{\top}\right)}\exp\left(-\lambda \theta^{\top} \left(\sum\nolimits_{t = 1}^T x_t x_t^{\top}\right) \theta\right).  
\end{equation}
We do not provide the derivation here but instead present several versions of the derivation throughout the paper. This direct approach is a natural first step when analyzing the logarithmic loss, hinge loss, and absolute loss. In \cref{sec:linearrevisit}, we discuss some general deterministic identities for online linear regression that are capable of recovering 
existing sequential linear regression bounds, including the regret bound of \cref{thm:gaillard2019}.

The prior \eqref{eq:jeffreys} has a close connection with the 
\emph{Jeffreys prior}, where $\sqrt{\det\left(\sum\nolimits_{t = 1}^T x_t x_t^{\top}\right)}$ plays the role of this prior and the term $\exp\left(-\lambda \theta^{\top} \left(\sum\nolimits_{t = 1}^T x_t x_t^{\top}\right) \theta\right)$ plays the role of the exponential tilting that makes this prior integrable. We refer to \cite[Section 3]{barron1999information} as well as \citep[Chapter 13.4]{polyanskiy2024information}. Connections with Jeffreys priors in the context of existing information theory results will be made throughout the paper.

\subsection{Sparse linear regression}
\label{sec:sparselinear}
In this section, as an initial demonstration of our approach utilizing transductive priors for exponential weights, we aim to extend \cref{thm:gaillard2019}. Specifically, we will focus on linear regression, with an emphasis on deriving \emph{sparsity} regret bounds within the transductive setup. We remark that the techniques used in this section can be readily applied 
in other contexts, including logistic regression and classification with hinge loss.

The regret bounds of this section compare favorably to the existing bounds with a fixed design, such as those presented in \citep{dalalyan2008aggregation,dalalyan2012mirror, dalalyan2012sparse}, as well as their online counterparts presented in \citep{gerchinovitz2011sparsity}. 
Our regret bounds do not depend on the magnitude of the optimal solution, which is the case in the aforementioned papers in the \emph{batch} setup, where the whole data sample is revealed to us in advance. This aspect has already been noticed in \citep{rigollet2012sparse}, where discrete sparse priors are used to achieve sharp oracle inequalities without any dependence on the norm of the optimal solution but in the context of fixed design batch linear regression with Gaussian noise, a denoising result less general than our sequential result.

First, we introduce the following \emph{sparsity inducing} data-dependent prior. Fix some $\tau > 0$ and assume without loss of generality that for all $j$, it holds that $\sum\nolimits_{t = 1}^{T} (x_t^{(j)})^2 \neq 0$, where $x_t^{(j)}$ denotes the $j$-th coordinate of $x_t$. Indeed, otherwise, we may ignore this coordinate. The prior $\mu$ over $\mathbb{R}^d$ is defined as follows:
\begin{equation}
\label{eq:sparseprior}
    \mu(\theta) = \prod\limits_{j = 1}^d \frac{3 \cdot \sqrt{\sum\nolimits_{t = 1}^{T} (x_t^{(j)})^2} / \tau}{2 \left(1 + |\theta^{(j)}| \cdot \sqrt{\sum\nolimits_{t = 1}^{T} (x_t^{(j)})^2} / \tau \right)^4}.
\end{equation}

The prior is similar to the sparse priors used in \citep{dalalyan2008aggregation, dalalyan2012mirror, dalalyan2012sparse} and \citep{gerchinovitz2011sparsity}, but, as will always be the case in this paper, it explicitly depends on $x_1, \ldots, x_T$. The estimator we use has an explicit integral form following Vovk's algorithm (\Cref{lem:mixabilityofthesqloss}). Given $m$, which is the upper bound on $\max_{t} |y_t|$, define for $x \in \mathbb{R}^d$,
\begin{equation}
\label{eq:predictionformulasparse}
\widehat{f}_t(x) = \frac{m}{2} \cdot \log\left(\frac{\int\limits_{\mathbb{R}^d} \exp\left(-\frac{1}{2m^2}(m - \langle x, \theta \rangle)^2 - \frac{1}{2m^2} \sum\limits_{i = 1}^{t - 1} (y_i - \langle x_i, \theta \rangle)^2\right) \mu(\theta) d\theta}{\int\limits_{\mathbb{R}^d} \exp\left(-\frac{1}{2m^2}(-m - \langle x, \theta \rangle)^2 - \frac{1}{2m^2} \sum\limits_{i = 1}^{t - 1} (y_i - \langle x_i, \theta \rangle)^2\right) \mu(\theta) d\theta}\right).
\end{equation}
We discuss the computational aspects behind this estimator in \cref{sec:computational}. We emphasize that the analysis of the related predictor in \citep{gerchinovitz2011sparsity} is based on exp-concavity and requires an additional clipping step, which prevents computational efficiency. As we mentioned, the computational aspects of sparse sequential linear regression were raised as an open question in \citep[Chapter 23]{lattimore2020bandit}.

We need an additional scale-invariant assumption on the design vectors $x_1, \ldots, x_T$. This assumption is quite standard in the sparse recovery literature and is discussed and used in \citep{castillo2015bayesian, ray2022variational} among other papers.

\begin{definition}[Smallest scaled singular value]
    Assume that there is a constant $\kappa_{s} > 0$ such that for any non-zero, $s$-sparse vector $v \in \mathbb{R}^d$ it holds that 
    \[
\frac{\sqrt{\sum\nolimits_{t=1}^T (\langle x_t, v \rangle)^2}}{\max_j \sqrt{\sum\nolimits_{t=1}^T (x_t^{(j)})^2} \cdot \|v\|_2} \ge \kappa_{s}.
\]

\end{definition}

The following result is the main bound of this section.

\begin{theorem}
    \label{thm:sparsity}
    Assume that the sequence $\{(x_t, y_t)\}_{t = 1}^T$, where $(x_t, y_t) \in \mathbb{R}^d \times \mathbb{R}$ for all $1 \leq t \leq T$, is such that $\max_{t} |y_t| \le m$ and that the smallest scaled singular value condition is satisfied with constant $\kappa_{s}$. Then, for any $s$-sparse $\theta^{\star} \in \mathbb{R}^d$, the following regret bound holds for the estimator \eqref{eq:predictionformulasparse} with $\tau = m\sqrt{\frac{s}{d}}$,
    \begin{equation}
    \label{eq:sparsitytotalloss}
            \sum\limits_{t = 1}^T (y_t - \widehat{f}_t(x_t))^2 \le \sum\limits_{t = 1}^T (y_t - \langle x_t, \theta^{\star} \rangle)^2 + sm^2\left(1 + 8\log\left(1 + \frac{2}{\kappa_{s}}\sqrt{\frac{dT}{s}}\right)\right).
    \end{equation}
\end{theorem}

There are multiple interesting properties of the above result. First, the main term in \eqref{eq:sparsitytotalloss} is \(O\left(sm^2\log\left(\frac{dT}{\kappa_{s}^2s}\right)\right)\) and does not depend on the magnitude of design vectors and the magnitude of \(\theta^{\star}\). Second, the dimensionality of \(\theta^{\star}\) enters the regret bound only through \(s\log d\). Finally, we note that the dependence on \(\frac{1}{\kappa_{s}}\) is logarithmic, which is not typical for sparse recovery problems, where dependencies of this type are usually polynomial.

It is time to introduce the exponential weights algorithm \citep{Littlestone1994Weighted, vovk1990aggregating, freund1997decision}. Fix $\eta > 0$ and let $\ell_{\theta}(\cdot): \mathcal{X} \times \mathcal{Y} \to \mathbb{R}_{+}$ be a set of loss functions parameterized by some $\Theta \subseteq \mathbb{R}^d$. For notational simplicity, we usually use the same symbols for distributions and their densities. Fix some prior $\mu$ over $\Theta$. Given the sequence $(x_1, y_1), \ldots, (x_T, y_T)$, define $\rho_1 = \mu$. For $t \ge 2$, we set
\begin{equation}
\label{eq:expweights}
\rho_{t}(\theta) = \frac{\exp\left(-\eta \sum\limits_{i = 1}^{t - 1} \ell_{\theta}(x_i, y_i)\right)\mu(\theta)}{\mathbb{E}_{\theta^{\prime} \sim \mu} \exp\left(-\eta \sum\limits_{i = 1}^{t - 1} \ell_{\theta^{\prime}}(x_i, y_i)\right)}.
\end{equation}
The following identity is standard and appears in, e.g., \citep[Lemma 1]{vovk2001competitive}.
\begin{lemma}
\label{lem:sumofmixlosses}
The exponential weights algorithm satisfies the following identity:
\begin{equation}
\label{eq:sumofmixlosses}
\sum\limits_{t = 1}^T -\frac{1}{\eta} \log\left(\mathbb{E}_{\theta \sim \rho_t} \exp(-\eta \ell_{\theta}(x_t, y_t))\right) = -\frac{1}{\eta} \log\left(\mathbb{E}_{\theta \sim \mu} \exp\left(-\sum\nolimits_{t = 1}^T \eta \ell_{\theta}(x_t, y_t)\right)\right).
\end{equation}
\end{lemma}

Each summand on the left-hand side of \eqref{eq:sumofmixlosses} is called the \emph{mix-loss}. Our next lemma is also standard in the online learning literature and is an immediate corollary of the Donsker-Varadhan variational formula.
\begin{lemma}
\label{lem:donskervaradhan}
In the setup of \cref{lem:sumofmixlosses}, the following identity holds:
\[
-\frac{1}{\eta} \log\left(\mathbb{E}_{\theta \sim \mu} \exp\left(-\sum\nolimits_{t = 1}^T \eta \ell_{\theta}(x_t, y_t)\right)\right) = \inf_{\varphi}\left(\mathbb{E}_{\theta \sim \varphi} \sum\limits_{t = 1}^T \ell_{\theta}(x_t, y_t) + \frac{\mathcal{KL}(\varphi \| \mu)}{\eta}\right).
\]
\end{lemma}

Finally, we introduce another key lemma. This form is due to \citep[Lemma 2 and computations after it]{vovk2001competitive}. The important aspect of it is that it does not require boundedness of the class of functions (in our case, it is the unbounded linear class), which appears to be important for our analysis. We additionally refer to \citep{van2015fast} for a detailed account of mixability. For completeness, we provide a proof also in \cref{app:mixability-squared-loss}.
\begin{lemma}[Mixability of the squared loss]
\label{lem:mixabilityofthesqloss}
Consider a class $\mathcal F$ of functions $f_{\theta}: \mathcal{X} \to \mathbb{R}$ (possibly unbounded) parameterized by $\Theta \subseteq \mathbb{R}^d$. Assume that $y$ is such that $|y| \le m$ and choose $\eta = \frac{1}{2m^2}$. Given any distribution $\rho$ over $\Theta$, define the predictor
\begin{equation}
\label{eq:predictiormix}
\widehat{f}(x) = \frac{m}{2}\log\left(\frac{\mathbb{E}_{\theta \sim \rho} \exp(-\eta (m - f_{\theta}(x))^2)}{\mathbb{E}_{\theta \sim \rho} \exp(-\eta (-m - f_{\theta}(x))^2)}\right).
\end{equation}
Then, the following holds:
\[
(y - \widehat{f}(x))^2 \le -\frac{1}{\eta} \log\left(\mathbb{E}_{\theta \sim \rho} \exp(-\eta (y - f_{\theta}(x))^2)\right).
\]
\end{lemma}

\begin{proof}[\pfref{thm:sparsity}]
In what follows, $\eta = \frac{1}{2m^2}$. Our first observation is that if we choose $\rho = \rho_t$ in \cref{lem:mixabilityofthesqloss}, where $\rho_t$ is the exponential weights algorithm with prior $\mu$ given by \eqref{eq:sparseprior} and $\ell_{\theta}(x, y) = (y -\langle x, \theta\rangle)^2$, then the predictor \eqref{eq:predictiormix} coincides exactly with the predictor $\widehat{f}_t$ given by \eqref{eq:predictionformulasparse}. This implies by \cref{lem:mixabilityofthesqloss}, \cref{lem:sumofmixlosses} and \cref{lem:donskervaradhan}:
\begin{align}
\sum\limits_{t = 1}^T(y_t - \widehat{f}_t(x_t))^2 &\le \sum\limits_{t = 1}^T-\frac{1}{\eta} \log\left(\mathbb{E}_{\theta \sim \rho_t} \exp\left(-\eta (y_t - \langle x_t, \theta\rangle)^2\right)\right) \nonumber
\\
&= -\frac{1}{\eta} \log\left(\mathbb{E}_{\theta \sim \mu} \exp\left(-\sum\nolimits_{t = 1}^T \eta (y_t - \langle x_t, \theta\rangle)^2\right)\right) \nonumber
\\
&\le \E_{\theta \sim \mu^{\star}}\sum\limits_{t = 1}^T(y_t - \langle x_t, \theta\rangle)^2 + \frac{\mathcal{KL}(\mu^{\star} \| \mu)}{\eta}, \label{eq:riskpluskl}
\end{align}
where $\mu^{\star}$ is any distribution over $\Theta$. It remains to bound both terms in the last line. First, given an $s$-sparse target vector $\theta^{\star}$, we choose 
\[
\mu^{\star}(\theta) = \prod\limits_{j = 1}^d \frac{3 \cdot \sqrt{\sum\limits_{t = 1}^{T} (x_t^{(j)})^2} / \tau}{2 \left(1 + |\theta^{(j)} - \theta^{\star(j)}| \cdot \sqrt{\sum\limits_{t = 1}^{T} (x_t^{(j)})^2} / \tau \right)^4}.
\]
By symmetry $\E_{\theta\sim\mu^{\star}}\theta = \theta^{\star}$. Moreover, a simple computation shows that $\E_{\theta \sim \mu}(\theta^{(j)})^2 = \tau^2/(\sum\nolimits_{t = 1}^{T}(x_t^{(j)})^2)$ (recall that $\mu$ is defined by \eqref{eq:sparseprior}). Therefore, we have
\begin{align*}
\E_{\theta \sim \mu^{\star}} \sum\limits_{t = 1}^T(y_t - \langle x_t, \theta\rangle)^2 &= \sum\limits_{t = 1}^T(y_t - \langle x_t, \theta^{\star}\rangle)^2 + \sum\limits_{t = 1}^T\E_{\theta \sim \mu^{\star}}(\langle x_t, \theta - \theta^{\star}\rangle)^2
\\
&= \sum\limits_{t = 1}^T(y_t - \langle x_t, \theta^{\star}\rangle)^2 + \sum\limits_{j = 1}^d\sum\limits_{t = 1}^T\E_{\theta \sim \mu}(x_t^{(j)}\theta^{(j)})^2
\\
&= \sum\limits_{t = 1}^T(y_t - \langle x_t, \theta^{\star}\rangle)^2 + \sum\limits_{j = 1}^d\sum\limits_{t = 1}^T\frac{(\tau x_t^{(j)})^2}{\sum\nolimits_{t = 1}^{T}(x_t^{(j)})^2}
\\
&= \sum\limits_{t = 1}^T(y_t - \langle x_t, \theta^{\star}\rangle)^2 + \tau^2 d.
\end{align*}
Now we proceed with the second term in \eqref{eq:riskpluskl} exactly the same way as in \citep[Theorem 5]{dalalyan2008aggregation} or, more specifically, as in \citep[
Lemma 23]{gerchinovitz2011sparsity}. We have
\begin{align}
\mathcal{KL}(\mu^{\star}\|\mu) &\le 4\sum\limits_{j = 1}^d\log\left(1 + \frac{|\theta^{\star(j)}|\sqrt{\sum\nolimits_{t = 1}^{T}(x_t^{(j)})^2}}{\tau}\right) \nonumber
\\
&\le 4\|\theta^{\star}\|_0\max\limits_j\left\{\log\left(1 + \frac{|\theta^{\star(j)}|\sqrt{\sum\nolimits_{t = 1}^{T}(x_t^{(j)})^2}}{\tau}\right)\right\}. \label{eq:klbound}
\end{align}
Let us now compare any $s$-sparse $\theta^{\star}$ with the solution $\theta_0 = 0$. We may assume without loss of generality that
\begin{equation}
\label{eq:controloverthetastar}
\sum\limits_{t = 1}^T (y_t - \langle x_t, \theta^{\star} \rangle)^2 \le \sum\limits_{t = 1}^T (y_t - \langle x_t, \theta_0 \rangle)^2 = \sum\limits_{t = 1}^T y_t^2.
\end{equation}
Otherwise, $\theta_0$, satisfying $\|\theta_0\|_0 = 0$, will make both terms on the right-hand side of \eqref{eq:sparsitytotalloss} smaller than for a given $\theta^{\star}$. Thus, we restrict our attention only to those solutions satisfying \eqref{eq:controloverthetastar}. Now, we use the smallest singular value condition. This assumption implies
\[
|\theta^{\star(j)}| \sqrt{\sum\nolimits_{t = 1}^{T} (x_t^{(j)})^2} \le \|\theta^{\star}\|_2 \sqrt{\sum\nolimits_{t = 1}^{T} (x_t^{(j)})^2}  \le \frac{1}{\kappa_{s}}\sqrt{\sum\nolimits_{t = 1}^{T} (\langle x_t, \theta^{\star}\rangle)^2}.
\]
Finally, using \eqref{eq:controloverthetastar} and $|y_t| \le m$, we have
\[
\frac{1}{\kappa_{s}}\sqrt{\sum\nolimits_{t = 1}^{T} (\langle x_t, \theta^{\star}\rangle)^2} \le \frac{1}{\kappa_{s}}\sqrt{2\sum\nolimits_{t = 1}^{T} \left((y_t - \langle x_t, \theta^{\star}\rangle)^2 +  y_t^2\right)} \le \frac{2m\sqrt{T}}{\kappa_{s}}.
\]
Plugging this into \eqref{eq:klbound} and choosing $\tau = m\sqrt{\frac{s}{d}}$ in \eqref{eq:riskpluskl}, we conclude the proof.
\end{proof}

We make several remarks regarding the result. The bound of \cref{thm:sparsity} depends logarithmically on $\frac{1}{\kappa_{s}}$. However, even if $\kappa_{s}$ is close to zero, our bound remains meaningful as it depends at most logarithmically on $\|\theta^{\star}\|$ and $\max_t \|x_t\|_2$, provided both are bounded. This follows directly from upper bounding \eqref{eq:klbound} when both $\theta^{\star}$ and all $x_t$-s are bounded. 

We can also derive a regret upper bound independent of the parameter $\kappa_s$ by combining the transductive prior in \eqref{eq:jeffreys} with the discrete sparse prior proposed by \cite{rigollet2012sparse}. 
We include this result for completeness in \cref{app:additional-sparse}. Note that for predictors corresponding to such priors no polynomial time implementation is known.

\subsection{Separation between inductive and transductive settings}
\label{sec:lower-bound-squared-loss}

In this section, we present a challenging case for Vovk's online linear regression predictor in the standard online setting and prove a separation result between transductive and standard online learning. \looseness=-1

\paragraph{Lower bounds for Vovk's aggregating algorithm in the standard online setup.}
It is open whether the transductive setup that we are considering here shares the same minimax regret bound (up to constants) with the classical inductive setup considered by \citet{vovk2001competitive} and \citet{azoury2001relative} (see the discussion in \citep[Section 4.1]{gaillard2019uniform}). Namely, the minimax regret bound for the transductive setup scales with $\Theta(dm^2\log T)$ \citep{gaillard2019uniform} whereas the best known upper bound for the inductive setup scales with $O(dm^2\log (1+rbT))$ 
, where $\norm{x_t}\leq r$ and $\norm{\thstar}\leq b$ are given by assumption. Moreover, the upper bound for the inductive setup is obtained through the classical Vovk-Azoury-Warmuth algorithm with appropriate regularization. In the following, we show that this algorithm 
will incur linear regret in the inductive case when the comparator $\thstar$ is not assumed to be bounded. Concretely, we have the following simple result.
\begin{proposition}
\label{prop:lower-bound-for-inductive-vovk}
    Let $d=1$. Consider the Vovk-Azoury-Warmuth algorithm, which computes, upon observing $x_t$ at time $t$, 
    \begin{align*}
        \thetahat_t = \arg\min_{\theta \in \mathbb{R}}\left\{\sum_{i = 1}^{t - 1} (y_i - x_i  \theta )^2 + ( x_t \theta )^2 + \lambda\norm{\theta}^2 \right\},
    \end{align*}
    where $\lambda$ is the regularization fixed in advance. There exists a sequence $\set{(x_t,y_t)}_{t=1}^T$, where $y_t\in \set{\pm 1}$, such that 
    \[
\sum\limits_{t = 1}^T (y_t - x_t \hat{\theta}_{t})^2 - \inf_{\theta \in \mathbb{R}}\left\{\sum_{t = 1}^{T} (y_t -  x_t \theta )^2 \right\} \geq \Omega(T) .
\]
\end{proposition}

\begin{proof}[\pfref{prop:lower-bound-for-inductive-vovk}]
    We consider two cases. In the first case, the regularization is given at a large scale $\lambda\geq T$. For this case, consider $x_t=1$ and $y_t = 1$ for all $t$. Then we have 
    \begin{align*}
        \inf_{\theta \in \mathbb{R}}\left\{\sum_{t = 1}^{T} (y_t -  x_t \theta )^2 \right\} =  T    \inf_{\theta \in \mathbb{R}}(1 -  \theta )^2 = 0.
    \end{align*}
    However, at each time step $t$, we have
    \begin{align*}
        \thetahat_t =  \frac{\sum\limits_{i=1}^{t-1} x_iy_i }{ \lambda + \sum\limits_{i=1}^t x_i^2 } \leq \frac{t-1}{t+T} \leq \frac{1}{2}.
    \end{align*}
    Then, the total loss is lower bounded by
    \begin{align*}
        \sum\limits_{t = 1}^T (y_t - x_t\widehat{\theta}_{t} )^2 \geq T(1-1/2)^2 = T/4.
    \end{align*}
    In the second case, the regularization is given at a relatively small scale with $\lambda<T$.  We let $x_t =  2^t T$ and $y_t = (-1)^t$ for $t\geq 1$. Then, for any $t\geq 2$, 
    \begin{align*}
         x_t\thetahat_t = 2^t T \cdot  \frac{\sum\limits_{i=1}^{t-1} x_iy_i }{ \lambda + \sum\limits_{i=1}^t x_i^2 } = 2^t T\cdot\frac{T\cdot \sum\limits_{i=1}^{t-1} (-2)^i }{ \lambda  +  T^2\cdot \sum\limits_{i=1}^t 4^i  } = 2^tT^2  \cdot\frac{-2\prn*{1-(-2)^{t-1}}/3}{ \lambda  + 4T^2(4^{t}-1)/3 }.
    \end{align*}
    Then, we have for any $t\geq 2$,
    \begin{align*}
        |y_t - x_t\thetahat_t| &=  \abs*{ (-1)^t - 2^t T^2  \cdot\frac{-2\prn*{1-(-2)^{t-1}}/3}{ \lambda  + 4T^2(4^{t}-1)/3 }}\\
        &= \abs*{  1 +  \frac{T^2 2^{t+1}\prn*{(-1)^t + 2^{t-1}}/3}{ \lambda + 4T^2(4^{t}-1)/3 }   } \\
        &\geq \abs*{1+ \frac{T^2 2^{2t-1}/3}{5T^24^{t}/3}  } \geq  11/10.
    \end{align*}
    This establishes the lower bound of 
    \begin{align*}
        \sum\limits_{t = 1}^T (y_t - x_t \widehat{\theta}_{t} )^2 - \inf_{\theta \in \mathbb{R}}\left\{\sum_{t = 1}^{T} (y_t - x_t \theta)^2 \right\} &\geq  \sum\limits_{t=2}^T (y_t - x_t \widehat{\theta}_{t} )^2 - \sum_{t=1}^T y_t^2  \\
        &\geq (11/10)^2(T-1) -  T \geq \Omega(T),
    \end{align*}
    where the first inequality is by $\inf_{\theta \in \mathbb{R}}\left\{\sum_{t = 1}^{T} (y_t - x_t \theta)^2 \right\} \leq \sum_{t=1}^T y_t^2$. The claim follows.
\end{proof}

Although our result does not resolve the question in \citep[Section 4.1]{gaillard2019uniform}, it hints, in view of \cref{thm:gaillard2019}, at a separation between the transductive setup and the inductive setup. 
In the following section, we show that the separation indeed exists for the clipped linear function class.

\paragraph{Separation for clipped linear class.}
We now investigate the clipped linear class and demonstrate a separation between transductive online learning and inductive setups when dealing with unbounded covariates and comparators. Recall that for the linear class, the learner can compare with the optimal loss of 
\begin{align*}
    \inf_{\theta \in \mathbb{R}^d}\left\{\sum_{t = 1}^{T} (y_t - \langle x_t, \theta \rangle)^2 \right\}
\end{align*}
with at most a regret of $O(dm^2\log T)$ as indicated in \eqref{eq:gaillardlinear}. In fact, in the transductive setup even more is possible, as the learner can compare with an even smaller loss obtained by the clipped linear class, that is,
\begin{align}
\label{eq:clipped-optimal}
    \inf_{\theta \in \mathbb{R}^d}\left\{\sum_{t = 1}^{T} (y_t - \clip{\langle x_t, \theta \rangle})^2 \right\}
\end{align}
with a regret upper bound of $O(dm^2\log T)$ following a covering argument, which we  detail in \cref{app:covering}. 
To show a separation between the transductive setup and the inductive setup, we prove that any algorithm must suffer linear regret of $\Omega(T)$ in the worst case in the inductive setting for the clipped linear class, i.e., comparing to the optimal loss achieved by the clipped linear class as in \eqref{eq:clipped-optimal}. Concretely, we have the following result.

\newcommand{\yhat}{\hat{y}}
\begin{proposition}
\label{prop:lower-bound-clipped}
    Let $d=1$ and $m=1$. For any (inductive) algorithm outputting $\yhat_1,\ldots,\yhat_T$ as predictions, there exists a sequence $\set{(x_t,y_t)}_{t=1}^T$, where $y_t\in \set{\frac{1}{8} ,\frac{7}{8}}$ such that 
    \begin{equation*}
        \sum_{t=1}^T (y_t - \yhat_t)^2 - \inf_{\theta\in \bR} \prn*{ \sum_{t=1}^T (y_t - \clip{x_t\theta})^2 } =  \Omega(T).
    \end{equation*}
\end{proposition}

\newcommand{\thetatil}{\wt{\theta}}

\begin{proof}[\pfref{prop:lower-bound-clipped}]
With a slight abuse of the notation,
denote $x^{(i)} = 10^{i}$ and $\theta^{(i)} = \frac{15}{16 x^{(i)}}$ for $i\in \set{0,1,\ldots,2^T-2}$. 
We have
\begin{align*}
\begin{cases}
    \theta^{(i)}x^{(j)}  >7/8 & \text{if~} j\geq i,\\
    \theta^{(i)}x^{(j)}  <1/8 & \text{if~} j< i.
\end{cases}
\end{align*}
This thresholding structure enables the construction of a sequence $(x_1,y_1),\ldots,(x_T,y_T)$ and a parameter $\thetatil$ such that for any predictions $\yhat_1,\ldots,\yhat_T$ and for all $t\in [T]$,
\begin{align}
\label{ineq:binary-tree}
\begin{split}
    \begin{cases}
 x_t\thetatil >  y_t = 7/8 & \text{if~} \yhat_t \leq  1/2,\\
 x_t\thetatil <  y_t = 1/8 & \text{if~} \yhat_t > 1/2
\end{cases}
\end{split}
\end{align}
with the following binary searching construction. To start with, we set the left and the right end to be respectively $a_1= 0$ and $b_1 = 2^T-2$. At each round $t$, the adversary presents the middle point of $a_t$ and $b_t$, that is, $i_t = \lfloor \frac{a_t+b_t}{2} \rfloor$ and $x_t = x^{(i_t)}$. 
Then if $\yhat_t \leq 1/2$, then we set $y_t = 7/8$, $a_{t+1} = i_t+1$,
and $b_{t+1}=b_t$, else we set $y_t=1/8$, $a_{t+1}=a_t$, and $b_{t+1} = i_t-1$.
At the beginning, we have $n_1 = 2^T-1$ covariates in between $a_1$ and $b_1$ and each time we are left with $(n_t-1)/2$ covariates in between $a_{t+1}$ and $b_{t+1}$. Then by induction, $n_t = 2^{T-t+1}-1$. Thus, $n_T=1$. Finally, it is clear that $\thetatil = \theta^{(i_T)}$ satisfies \eqref{ineq:binary-tree}.

Then, with the sequence of $x_t,y_t,\yhat_t$ and $\thetatil$ satisfying \eqref{ineq:binary-tree}, since $|y_t-\yhat_t|>3/8$ and $\left|y_t - \clip{x_t\thetatil}\right|<1/8$, this implies the following lower bound on the regret
\begin{align*}
    \sum_{t=1}^T (y_t - \yhat_t)^2 - \inf_{\theta\in \bR} \prn*{ \sum_{t=1}^T (y_t - \clip{x_t\theta})^2 }  &\geq 9T/64 - \sum_{t=1}^T  \prn*{y_t - \clip{x_t\thetatil}}^2 \\
    &\geq 9T/64 - T/64 = T/8 = \Omega(T).
\end{align*}
Thus, the claim follows.
\end{proof}

\section{Logistic Regression}
\label{sec:logistic}
\newcommand{\sig}{\sigma}
\newcommand{\prr}{\mu} 
\newcommand{\logprr}{\mu_{\mathrm{log}}}
\newcommand{\sgn}{\mathrm{sign}}
\newcommand{\halving}{\mathrm{HALVING}}
\newcommand{\ewa}{\mathrm{EWA}}
\renewcommand{\KL}{\mathcal{KL}}

In this section, we introduce our main result for logistic regression with exponential weights algorithm via transductive priors. Concretely, we establish an assumption-free tight guarantee for logistic regression in the transductive online setting. The prior is carefully tailored to the logistic case as a mixture of the multivariate Gaussian prior defined in \eqref{eq:jeffreys} for a subset of the covariates together with the Vapnik-Chervonenkis-style argument to select these covariates. We will discuss more the intuition behind the choice of the prior as we develop our result.

For the transductive setting of log loss, it is well-known that the exact minimax risk is characterized by the Shtarkov sum \citep{shtar1987universal,wu2022sequential}, and this exact minimax risk is achieved by the Normalized Maximum Likelihood (NML) algorithm. However, the Shtarkov sum is difficult to interpret, and the NML algorithm is computationally prohibitive, with the computational complexity scaling as $\Omega(2^T)$ in the case of logistic regression. 
In terms of estimating the Shtarkov sum, \citet{jacquet2022precise} notably proved that for the logistic loss, the Shtarkov sum scales as $\frac{d}{2}\log \frac{T}{2\pi} +  \log \int \sqrt{\det(J_F(\theta;x_1,\dots,x_T) )} \, d \theta + o(1)$ when the covariates $x_1,\dots,x_T$ are selected from a finite set of size $N=o(\sqrt{T})$ and dimension $d\leq N$, where $J_F(\theta;x_1,\dots,x_T)$ represents the Jeffreys prior given $x_1,\dots,x_T$. This result leaves the interpretation of the Jeffreys prior term open for general covariates $x_1,\dots,x_T$. 
For the special case where the covariates are uniformly sampled from the unit sphere, an asymptotic upper bound of $O(d\log T)$ for the Shtarkov sum can be inferred.
Recent work by \citet*{drmota2024unbounded} claims to achieve an upper bound of $O(d\log (T/d))$ through a covering argument, which we formally detail in \cref{app:covering} but with worse constants in the regret bound. However, the covering-based argument lacks implementable algorithms for constructing the covering. 

Meanwhile, the sequential learning of logistic loss in the inductive setting has been extensively studied, both statistically and computationally, with various boundedness assumptions on the covariates $x_1,\dots,x_T$ and the optimal comparator $\thstar$ (see \citet{kakade2004online,foster2018logistic,shamir2020logistic,jezequel2021mixability} and references therein). However, it is crucial to obtain a regret bound without such assumptions, as, for instance, in the simplest case of linearly separable data, the norm of $\thstar$ can be infinite. \looseness=-1

We provide the first non-asymptotic, tight upper bound on the regret for transductive logistic regression without any assumptions on the covariates $x_1,\dots,x_T$ or the optimal comparator $\thstar$. Importantly, in \cref{sec:computational}, we discuss practical implementable versions of our algorithms that are guaranteed to converge to the optimal solution, with the only caveat that the number of iterations required may grow exponentially with the dimension in the worst-case scenario. In the specific case of random design logistic regression, we achieve in  \cref{sec:complog} a truly polynomial time implementation.

\subsection{Background for exponential weights algorithm for logistic loss}
Similar to the mixability of the squared loss (\cref{lem:mixabilityofthesqloss}), the mixability of logistic loss facilitates the analysis of exponential weights algorithms \citep{vovk1990aggregating}, or equivalently, Bayesian learning algorithms \citep{kakade2004online}. 
Formally, the logistic loss defined as $\ell_\theta (x,y) = - \log (\sigma(y\inner{x,\theta}))$ is $\eta$-mixable with $\eta=1$, i.e., the improper prediction of $\mathbb{E}_{\theta \sim \rho} \brk*{\sigma(y\inner{x,\theta}) }$ satisfies
\begin{align}
    \label{eq:mixability-of-log-loss}
     - \log\left(\mathbb{E}_{\theta \sim \rho} \brk*{\sigma(y\inner{x,\theta}) } \right) = -\log\left(\mathbb{E}_{\theta \sim \rho} \exp(-\ell_{\theta}(x, y))\right).
\end{align}
The Bayesian learning algorithm combined with the Jeffreys prior has been used to establish asymptotic optimality for parametric cases under the assumption that maximum likelihood estimates satisfy the Central Limit Theorem \citep{rissanen1996fisher}. A simpler presentation can be found in \citet[Section 13.4]{polyanskiy2024information}. These rest on the celebrated asymptotic Laplace approximation 
which states the right-hand-side of \eqref{eq:sumofmixlosses} for any prior $\mu$ is asymptotically
\begin{align}
    \begin{split}
        \label{eq:laplace-approximation}
    &-\frac{1}{\eta} \log\left(\mathbb{E}_{\theta \sim \mu} \exp\left(-\sum\nolimits_{t = 1}^T \eta \ell_{\theta}(x_t, y_t)\right)\right) \\
    &\approx -\frac{1}{\eta} \log \int \mu(\thstar) \exp\prn*{ - \sum\limits_{t=1}^{T}\eta \ell_{\thstar}(x_t, y_t)  - \frac{1}{2} (\theta-\thstar)^\top \Hess_t(\thstar)(\theta-\thstar) } d\theta,
    \end{split}
\end{align}
where $\thstar \ldef \argmin_{\theta \in \mathbb{R}^d}\sum\limits_{t = 1}^T\ell_{\theta}(x_t, y_t)$ and $\Hess_t(\thstar) = \frac{x_tx_t^\top}{(1+\exp((\thstar)^\top x_t))(1+\exp(-(\thstar)^\top x_t))}$ is the Hessian matrix of $\ell_\theta(x_t,y_t)$ at $\thstar$.
Applying an exact Laplace formula for quadratic forms (\cref{lem:laplace} below) to \eqref{eq:laplace-approximation} together with \cref{lem:sumofmixlosses} and \cref{eq:mixability-of-log-loss}, the asymptotic logistic loss of an exponential weights algorithm with prior $\mu$ is approximately 
\begin{align}
    \label{eq:asymptotic-exp-weights}
-\sum\limits_{t=1}^{T}  \log\left(\mathbb{E}_{\theta \sim \rho_t} \brk*{\sigma(y_t\inner{x_t,\theta}) } \right) \approx \sum\limits_{t=1}^{T}\ell_{\thstar}(x_t, y_t) 
+ \frac{d}{2}\log \frac{T}{2\pi} + \log \left(\frac{\sqrt{\det \prn*{\frac{1}{T} \sum\nolimits_{t=1}^{T} \Hess_t(\thstar) }}}{\mu(\thstar)} \right).
\end{align}
The above approximation motivates the choice of a prior $\mu(\theta)\propto \sqrt{\det \prn*{\frac{1}{T} \sum\nolimits_{t=1}^{T} \Hess_t(\theta) }}$ which is the Jeffreys prior.
The right-hand side of \eqref{eq:asymptotic-exp-weights} is also obtained as an upper bound for the Shtarkov sum (thus for the normalized maximum likelihood algorithm) by \citet{jacquet2022precise}. However, for the purpose of non-asymptotic analysis, which is the focus of this paper, the Jeffreys prior for the logistic loss can be computed but is hard to work with, as it is proportional to
\begin{align*}
   \mu(\theta) \propto \sqrt{  \det \prn*{ \frac{1}{T} \sum_{t=1}^T\frac{x_tx_t^\top}{(1+\exp(\theta^\top x_t))(1+\exp(-\theta^\top x_t))} }}.
\end{align*}

Instead, \citet{kakade2004online} employs a Gaussian prior and obtains an upper bound through quadratic approximation of the logistic loss. Concretely, by choosing the prior distribution $\mu = \cN(0, \nu^2 I)$ to be gaussian for some $\nu>0$ to be specified, by \cref{lem:sumofmixlosses,lem:donskervaradhan,eq:mixability-of-log-loss}, the loss of the exponential weights algorithm satisfies
\[
    -\sum\limits_{t=1}^{T}  \log\left(\mathbb{E}_{\theta \sim \rho_t} \brk*{\sigma(y_t\inner{x_t,\theta}) } \right) = \inf_{\varphi}\left(-\sum\limits_{t = 1}^T  \mathbb{E}_{\theta \sim \varphi} \brk*{   \log (\sigma(y_t\inner{x_t,\theta})) }+ \mathcal{KL}(\varphi \| \mu)\right).
\]
By choosing $\varphi = \cN(\thstar,\varepsilon^2I)$ for some $\varepsilon>0$, together with $|\partial^2 (- \log \sigma(y_tz) )/\partial z^2| \leq 1/4 $, the quadratic approximation of $- \En_{\theta\sim \varphi}\sum\limits_{t = 1}^T \log \sigma(y_t\inner{x_t,\theta}) $ near $\thstar$ is upper bounded by
\begin{align*}
    - \En_{\theta\sim \varphi}\sum\limits_{t = 1}^T \log \sigma(y_t\inner{x_t,\theta}) \leq - \sum\limits_{t = 1}^T \log \sigma(y_t\inner{x_t,\thstar})  +  \varepsilon^2\sum\limits_{t=1}^{T}\norm{x_t}^2/8.
\end{align*}
This approximation is accurate up to an additive constant with the assumption that $\norm{x_t}\leq 1$ and $\varepsilon$ chosen to be $O(1/T)$. The difficulty of the choice of a Gaussian prior arises in the KL-divergence term, which is
\begin{align*}
    \mathcal{KL}(\varphi \| \mu) = d\log \nu + \frac{1}{2\nu^2} \prn*{\norm{\thstar}^2 + d\varepsilon^2} - \frac{d}{2} - d\log \varepsilon.
\end{align*}
In fact, $\thstar$ can be unbounded in the logistic regression case. Moreover, even $\min_{t}|\langle\thstar, x_t\rangle|$ can be unbounded, as can be seen from the realizable case where all $x_t$'s have the same label. 
Indeed, since a lower bound for the inductive case is established by \citet[Theorem 5]{foster2018logistic} in the following theorem, it is not possible to obtain our result with any prior that does not depend on $x_1,\ldots,x_T$. 

\begin{theorem}[\citet{foster2018logistic}]
\label{thm:logistic-inductive-lower-bound}
Consider the binary logistic regression problem over the class of linear predictors with parameter set $\Theta=\left\{\theta \in \mathbb{R}^d \mid\|\theta\|_2 \leq b\right\}$ with $b=\Omega(\sqrt{d} \log (T))$. Then for any algorithm for prediction with the binary logistic loss, there is a sequence of examples $\left(x_t, y_t\right) \in\mathbb{R}^d \times\{-1,1\}$ for $t\in [T]$ with $\left\|x_t\right\|_2 \leq 1$ such that the regret of the algorithm is
\[
\Omega\left(d \log \left(\frac{b}{\sqrt{d} \log (T)}\right)\wedge T\right).
\]
\end{theorem}

This theorem, in conjunction with our upper bound \cref{thm:main-logistic}, establishes a learnability separation in terms of assumptions between the transductive and inductive settings for logistic loss. Specifically, it demonstrates that the transductive setting is crucial for achieving assumption-free regret bounds. We note that \citet{dzhamtyrova2021lower} obtained an asymptotically tight lower bound for a related multiclass softmax problem under KL loss with the norm in the bound; while their setting differs, it would be interesting to see if similar ideas improve Theorem~\ref{thm:logistic-inductive-lower-bound}.

\subsection{Slab-experts}
\label{sec:slab-experts}
\newcommand{\SE}{B}

As discussed, in logistic regression, both $\thstar$ and $\min_{t}|\langle\thstar, x_t\rangle|$ can be unbounded.
To address this difficulty, we introduce the concept of \emph{slab-experts}, which operate by modifying the Gaussian prior to avoid dependence on $\norm{\thstar}$ or $\max_t \langle x_t, \thstar \rangle$. Specifically, a slab-expert, based on its own belief regarding a slab, partitions the set of covariates into two regions: the large margin region and the small margin region. On the small margin covariates, the slab-expert applies exponential weights using the transductive Gaussian prior restricted to those covariates. On the large margin covariates, the slab-expert deterministically predicts the label according to the slab.
Before detailing the slab-expert methodology, we first present a simple comparison lemma that controls the margin when the prediction is incorrect.
The lemma is stated as follows.

\begin{lemma}
\label{lem:margin-for-mistake-logistic}
For any sequence $\set{(x_t,y_t)}_{1\leq t\leq T}$ where $(x_t,y_t)\in \bR^d\times \set{\pm 1}$ for all $1\leq t\leq T$, suppose $\thstar$ is the optimal comparator in hindsight, i.e., 
\begin{align}
\thstar \ldef \argmin_{\theta \in \mathbb{R}^d}-\sum\limits_{t = 1}^T\log(\sig(y_t\langle x_t, \theta\rangle)) . \label{eq:theta-star}
\end{align}
For this comparator\footnote{The parameter $\thstar$ is not always well-defined in $\bR^d$. For instance, when the sequence is realizable, $\thstar$ will tend to be infinite in the separating direction. Since this is a minor technical issue, we avoid the complication in this paper.
For a more rigorous treatment, we refer to \citet{mourtada2022improper}.}  and any $1\leq t\leq T$, if $y_t\langle x_t, \thstar\rangle \leq 0$, then $|\langle\thstar, x_t\rangle| \le T\log(2)$.
\end{lemma}

\begin{proof}[\pfref{lem:margin-for-mistake-logistic}]
For any $t$,
\[
    -\log(\sig(y_t\langle x_t, \thstar\rangle)) \le \sum_{t = 1}^T-\log(\sig(y_t\langle x_t, \thstar{}\rangle)) \le  \sum_{t = 1}^T-\log(\sig(0)) = T\log 2,
\]
where the second inequality is by comparing $\thstar$ to $0\in \bR^d$.
Thus,
\[
|\langle x_t, \thstar\rangle| \le \log(1 + \exp(|\langle x_t, \thstar\rangle|)) = -\log(\sig(y_t\langle x_t, \thstar\rangle))\le T\log 2.
\]
Thus, the claim follows.
\end{proof}

\cref{lem:margin-for-mistake-logistic} states that for any $x_t$ wrongly predicted covariate by $\thstar$, the margin $|\langle\thstar, x_t\rangle|$ will not exceed $T$. This motivates the following separation of the covariates: either $|\langle\thstar, x_t\rangle|\leq T$ or $y_t\langle\thstar, x_t\rangle > T$. 

In the following, we first introduce a partial exponential weights algorithm that depends on any chosen subsets of the covariates that only predicts on the chosen subsets. 
Then, we use these partial exponential weights algorithms to build slab-experts. 
\begin{definition}[Partial exponential weights]
    \label{def:partial-ewa}
    For any subset $K\subset \set{x_t}_{t\in [T]}$, regularization parameter $\alpha>0$, loss function $\ell_\theta(x,y)$ for $\theta,x\in \bR^d$ and $y\in \set{\pm 1}$, and learning rate $\eta$, we use $\ewa_{K,\alpha,\ell,\eta}$ to denote the partial exponential weights algorithm with loss function $\ell$
      restricted to predictions on covariates in $K$ 
    with prior
    \begin{align*}
    \prr_{K,\alpha}(\theta) \ldef \left(\frac{\alpha}{\pi}\right)^{d/2}\sqrt{\det\left(\sum\nolimits_{x\in K}xx^{\top}\right)}\exp\left(-\alpha\theta^{\top}\left(\sum\nolimits_{x\in K}xx^{\top}\right)\theta\right).
    \end{align*}
    We omit the dependence on the loss function and the learning rate when they are clear from the context.
\end{definition}
Now, we introduce the construction of a slab-expert and present our main theorem. 
\begin{definition}[Slab-expert]
    \label{def:slab-expert}
A slab-expert $\SE_{\theta,\alpha}$ is an expert that is specified by a vector $\theta\in \bR^d$ and a regularization parameter $\alpha>0$.
Given $x_1,\ldots,x_T$, the expert constructs a slab $K_\theta = \set{ x_t:\abs*{\inner{x_t,\theta}} \leq T} $ and predicts labels in the following way: At time step $t$, if $x_t\in K_\theta$, the expert predicts on $x_t$ according to $\ewa_{K_\theta,\alpha}$ with logistic loss and learning rate $1$. If $x_t\notin K_\theta$, then the expert predicts $\yhat_t=\sign(\inner{x_t,\theta})$ with probability $1$. We denote by $\SE_{\theta,\alpha}(x_t,y_t)$ the probability of predicting $y_t$ upon observing $x_t$.
\end{definition}

\newcommand{\rhotil}{\tilde{\rho}}
\begin{algorithm}
\caption{Aggregation over slab-experts}
\label{alg:main}
\begin{algorithmic}[htb]
\Require The set of (unordered) covariates $\set{x_t}_{t\in [T]}$ and let $\alpha=1/T^3$.
\State Construct the set
of all unique slab-experts $\cA =\set{\SE_{\theta,\alpha}\mid{}\theta\in \bR^d}$ (removing copies) 
with logistic loss and learning rate $1$.
\For{$t=1,2,\ldots,T$}
\State The learner receives $x_t$.
\State \multiline{The learner predicts with aggregation on the slab-experts 
\begin{equation}
\label{eq:agg-slab-experts}
\rhotil_t(x_t,\cdot) =  \sum\limits_{\SE\in \cA} p_t(\SE) \SE(x_t,\cdot) \text{~~where~~} p_t(\SE) \propto \exp\prn*{ \sum\limits_{i=1}^{t-1}  \log \SE(x_i,y_i) } 
\end{equation}
and $\SE(x_i,y_i)$ is probability of $\SE$ predicting $y_i$ on $x_i$.}
\State Nature reveals $y_t$.
\State Update all $\SE$ in $\cA$ with $(x_t,y_t)$.
\EndFor 
\end{algorithmic}
\end{algorithm}

\begin{theorem}
    \label{thm:main-logistic}
    The estimators $\rhotil_1,..,\rhotil_T$ produced by \cref{alg:main} with logistic loss and learning rate $1$ satisfy the total loss upper bound of
    \begin{align*}
        -\sum\limits_{t=1}^{T}  \log \rhotil_t(x_t,y_t) \leq -\sum\limits_{t=1}^T\log(\sigma(y_t\langle x_t, \theta^{\star}\rangle)) 
         + 6d \log (eT).
    \end{align*}
\end{theorem}

There are several noteworthy aspects of this result. First and foremost, it provides the first assumption-free guarantee, meaning that no assumptions are required on the covariates $x_1,\dots,x_T$ or the comparator $\thstar$. 
In \cref{app:covering}, we will present an alternative covering approach that achieves the same upper bound, where the covering is constructed in a manner analogous to the slab division on the covariates.
However, our focus remains on the slab-experts approach, as the covering approach 
lacks implementable algorithms for constructing the covering.
Indeed, our algorithm can be implemented in time polynomial in $T$, though exponential in $d$ in the worst case. This efficiency stems from the slab-experts themselves, each of which consists of an exponential weight algorithm with a Gaussian prior and a deterministic linear separator. The exponential weight algorithm with a Gaussian prior can be computed in polynomial time through log-concave sampling, and the aggregation of the slab-experts can be performed in $\poly(T^d)$ time. Importantly, under mild additional assumptions on the distribution in the i.i.d. batch setup, we can reduce the number of experts to $2$ with pretraining, providing a polynomial time algorithm.
For further details, see \cref{sec:computational,sec:online-to-batch}.
Additionally, the slab-experts approach has two notable aspects:
\begin{enumerate}
    \item No clipping is needed for unbounded losses when computing exponential weights, unlike in \citet{gerchinovitz2011sparsity} or the covering argument in \cref{app:covering}, where clipping would disrupt the convexity and computational feasibility of sampling the slab-experts.
    
    \item Despite using two layers of exponential weights—one for slab-experts and another for covariates—the algorithm remains a Bayesian learning algorithm with a single prior, demonstrating that such an algorithm can achieve the tight assumption-free regret upper bound.
\end{enumerate}

The analysis of the algorithm relies on three key facts. First, the number of slab-experts is at most $\poly(T^d)$. Second, the optimal slab-expert $\SE_{\thstar,\alpha}$ with $\alpha=1/T^3$ incurs a regret of no more than $O(d\log T)$. Finally, the aggregation over the slab-experts results in an additional regret compared to the optimal slab-expert of at most $O(\log (1/\mu(\SE_{\thstar,\alpha})))$, where $\mu$ is a prior over any set of experts that includes the optimal slab-expert $\SE_{\thstar,\alpha}$. Consequently, using a uniform prior over the slab-experts, as in \cref{alg:main}, guarantees a regret upper bound of $O(d\log T)$.
We first derive an upper bound on the number of slab-experts.

\begin{lemma}
    \label{lem:num-slab-expert}
If $T \ge d$, then there are at most $(eT/d)^{4d}$ distinct slab-experts.
\end{lemma}

\begin{proof}[\pfref{lem:num-slab-expert}]
    By the Vapnik-Chervonenkis-Sauer-Shelah lemma (\citet[Lemma 6.10]{shalev2014understanding}), the number of all possible linear separations on the set $K$ is upper bounded by $(eT/(d+1))^{(d+1)}\leq (eT/d)^{2d}$, 
    whenever $T \ge d$. Therefore, since slab-experts can be formulated as intersections of two linear separations on the set $K$, the number of slab-experts is upper bounded by $(eT/d)^{4d}$.
\end{proof}

We will upper-bound the loss for the slab-expert $\SE_{\thstar,\alpha}$ and compare the performance of \cref{alg:main} to that of the slab-expert $\SE_{\thstar,\alpha}$. To do this, we first prove the following performance guarantee for the partial exponential weights algorithm $\ewa_{K,\alpha}$, where $K$ is a subset of the covariate set and $\alpha$ is a regularization parameter.
For simplicity, let $y_x$ be the labeling for $x\in \set{x_t}_{t\in [T]}$, that is $y_{x_t} = y_t$.

\begin{proposition}
\label{prop:exponential-weight-for-small-margin}
Let $K\subset \set{x_t}_{t\in [T]}$ and $\thstar\in \bR^d$.
Then the partial exponential weights algorithm $\ewa_{K,\alpha}$  as defined in \cref{def:partial-ewa}
with 
prediction restricted on covariates in $K$
satisfies a loss upper bound of 
\begin{align*}
    \revindent-\sum_{x\in K}  \log\left(\mathbb{E}_{\theta \sim \rho_x} \brk*{\sigma(y_x\inner{x,\theta}) } \right)  \\
    &\leq -\sum_{x\in K}\log(\sigma(y_x\langle x, \theta^{\star}\rangle)) 
    + \alpha(\thstar)^{\top}\left(\sum\limits_{x\in K} xx^{\top}\right)\thstar +  \frac{d}{2}\log\left(1 + \frac{1}{8\alpha}\right) 
\end{align*}
on the set $K$ where $\rho_x$ is the prediction of $\ewa_{K,\alpha}$ on $x$.
Specifically, suppose that $\thstar$ as defined in \cref{eq:theta-star} satisfies $|\inner{x,\thstar}| \leq T$ for all $x \in K$ and let $\alpha=1/T^3$. We have
\begin{align*}
    -\sum_{x\in K}  \log\left(\mathbb{E}_{\theta \sim \rho_x} \brk*{\sigma(y_x\inner{x,\theta}) } \right)  \leq -\sum_{x\in K}\log(\sigma(y_x\langle x, \theta^{\star}\rangle)) 
     + 2d \log T.
\end{align*}
\end{proposition}

Before providing the proof, we present the following standard result that computes exponential quadratic integrals based on the multivariate Gaussian integral.
\begin{lemma}[Exact Laplace formula for quadratic forms]
\label{lem:laplace}
    Let $Q(\theta) = \theta^{\top}A\theta + b^{\top}\theta + c$, where $A \in \mathbb{R}^{d \times d}$ is a positive definite matrix, $b \in \mathbb{R}^{d}$ and $c \in \mathbb{R}$. Then 
    \[
    \int\limits_{\mathbb{R}^d}\exp(-Q(\theta))d\theta = \exp\left(-\inf\limits_{\theta \in \mathbb{R}^d}Q(\theta)\right)\frac{\pi^{d/2}}{\sqrt{\det(A)}}.
    \]
\end{lemma}

\begin{proof}[\pfref{prop:exponential-weight-for-small-margin}]
By \cref{lem:sumofmixlosses}, we have
\begin{align*}
    -\sum_{x\in K}  \log\left(\mathbb{E}_{\theta \sim \rho_x} \brk*{\sigma(y_x\inner{x,\theta}) } \right)   = -\log\left(\int\limits_{\mathbb{R}^d}\exp\left(\sum\limits_{x\in K}\log(\sig(y_x\langle x, \theta\rangle))\right)\prr_{K,\alpha}(d\theta)\right).
\end{align*}

Let $f(z) = -\log(\sigma(z))$.
Since $|f^{\prime\prime}(z)| \le 1/4$,
using the Taylor expansion for $f$ we have for any $x\in K$,
\[
-\log(\sig(y_x\langle x, \theta\rangle)) \le -\log(\sig(y_x\langle x, \thstar\rangle)) - \frac{1}{1 + \exp(y_x\langle x, \thstar\rangle)}\cdot(y_x\langle x, \theta - \thstar\rangle) + \frac{1}{8}(\langle x, \theta - \thstar\rangle)^2.
\]
Summing these inequalities together, we have
\begin{align}
    -\sum\limits_{x\in K}\log(\sig(y_x\langle x, \theta\rangle)) &\le -\sum\limits_{x\in K}\log(\sig(y_x\langle x, \thstar\rangle)) -  \sum\limits_{x\in K}\frac{1}{1 + \exp(y_x\langle x, \thstar\rangle)}\cdot(y_x\langle x, \theta - \thstar\rangle) \notag\\
    &\quad\quad+ \frac{1}{8}(\theta - \thstar)^{\top}\left(\sum\limits_{x\in K}xx^{\top}\right)(\theta - \thstar).\label{ineq:taylor-for-logistic}
\end{align}
Thus, we have
\begin{align*}
    &-\log\left(\int\limits_{\mathbb{R}^d}\exp\left(\sum\limits_{x\in K}\log(\sig(y_x\langle x, \theta\rangle))\right)\prr_K(d\theta)\right) \\
    &\leq 
    -\sum\limits_{x\in K}\log(\sig(y_x\langle x, \thstar\rangle)) 
    - \log\left(\int\limits_{\mathbb{R}^d}\exp\left( \sum\limits_{x\in K}\frac{1}{1 + \exp(y_x\langle x, \thstar\rangle)}\cdot(y_x\langle x, \theta - \thstar\rangle) \right.\right.\\
    & \quad\quad\quad\quad\quad\quad\quad\quad\quad\quad\quad\quad\quad\quad\quad\quad\quad\quad\left.\left. -\frac{1}{8}\cdot (\theta - \thstar)^{\top}\left(\sum\limits_{x\in K} xx^{\top}\right)(\theta - \thstar)\right)\prr_K(d\theta)\right) .
\end{align*}
Using \cref{lem:laplace} to compute the integral of the quadratic form we have
\begin{align*}
&-\log\left(\int\limits_{\mathbb{R}^d}\exp\left(\sum\limits_{x\in K}\frac{1}{1 + \exp(y_x\langle x, \thstar\rangle)}\cdot(y_x\langle x, \theta - \thstar\rangle) \right.\right.\\
&\quad\quad\quad\quad\quad\quad\quad\quad\quad \left.\left. -\frac{1}{8}\cdot (\theta - \thstar)^{\top}\left(\sum\limits_{x\in K} xx^{\top}\right)(\theta - \thstar)\right)\prr_K(d\theta)\right) \\
&=-\log\left(\left(\frac{\alpha}{\pi}\right)^{\frac{d}{2}}\cdot\left(\frac{\pi}{\alpha + \frac{1}{8}}\right)^{\frac{d}{2}}
\exp\left(-\min\limits_{\theta\in \mathbb{R}^d}\left(\alpha\theta^{\top}\left(\sum\limits_{x\in K} xx^{\top}\right)\theta \right.\right.\right.\\
& \quad\quad\quad\quad\quad\quad\quad\quad\quad\left.\left.\left.+ \frac{1}{8}(\theta - \thstar)^{\top}\left(\sum\limits_{x\in K} xx^{\top}\right)(\theta - \thstar)\right.\right.\right.\\
& \quad\quad\quad\quad\quad\quad\quad\quad\quad\left.\left.\left. -\sum\limits_{x\in K}\frac{1}{1 + \exp(y_x\langle x, \thstar\rangle)}\cdot(y_x\langle x, \theta - \thstar\rangle) \right)\right)\right)\\
&\leq \alpha(\thstar)^{\top}\left(\sum\limits_{x\in K} xx^{\top}\right)\thstar +  \frac{d}{2}\log\left(1 + \frac{1}{8\alpha}\right). 
\end{align*}
This concludes our proof.
\end{proof}

\begin{proof}[\pfref{thm:main-logistic}]
We first upper bound the total loss by the loss of slab-expert $\SE_{\thstar,\alpha}$ by the definition of the exponential weights algorithm,
\begin{align}
    \begin{split}
        \label{eq:total}
    -\sum\limits_{t=1}^{T}  \log \rhotil_t(x_t,y_t)&= -\sum\limits_{t=1}^T \log \prn*{ \frac{\sum\limits_{\SE\in \cA} \exp\prn*{  \sum\limits_{i=1}^{t} \log \SE(x_i,y_i) }}{\sum\limits_{\SE\in \cA} \exp\prn*{ \sum\limits_{i=1}^{t-1} \log \SE(x_i,y_i)  }}} \\
    &= -\log \prn*{ \frac{1}{|\cA|} \sum\limits_{\SE\in \cA} \exp\prn*{ \sum\limits_{i=1}^{T} \log \SE(x_i,y_i) }}\\
    &\leq -\log \prn*{ \exp\prn*{  \sum\limits_{i=1}^{T} \log \SE_{\thstar,\alpha}(x_i,y_i) }}+ \log |\cA|\\
    &\leq  -\sum\limits_{i=1}^{T} \log \SE_{\thstar,\alpha}(x_i,y_i)+ 4 d\log(eT/d),
    \end{split}
\end{align}
where the last inequality is due to \cref{lem:num-slab-expert}.
For the slab-expert $\SE_{\thstar,\alpha}$ with $\alpha=1/T^3$, it is clear that inside the slab $K_\thstar$ as defined in \cref{def:slab-expert}, the assumption of \cref{prop:exponential-weight-for-small-margin} is satisfied, thus 
\begin{align}
    \label{eq:small-margin}
    -\sum_{x\in K_\thstar}  \log \SE_{\thstar,\alpha}(x,y_x)  \leq -\sum_{x\in K_\thstar}\log(\sigma(y_x\langle x, \theta^{\star}\rangle)) 
     + 2d \log T.
\end{align}
Meanwhile, for all the covariates outside the slab, the prediction of $\SE_{\thstar,\alpha}$ coincides with the linear separator with vector $\thstar$. Moreover, $y_x\inner{x,\thstar} > 0$ for all $x\in \set{x_t}_{t\in [T]}\setminus K_\thstar$. Thus we have as a result
\begin{align}
    \label{eq:large-margin}
    -\sum_{x\in \set{x_t}_{t\in [T]}\setminus K_\thstar}  \log \SE_{\thstar,\alpha}(x,y_x) = 0.
\end{align}
Plug \eqref{eq:small-margin} and \eqref{eq:large-margin} back into \eqref{eq:total}, we obtain
\begin{align*}
    -\sum\limits_{t=1}^{T}  \log \rhotil_t(x_t,y_t) \leq - \sum\limits_{t=1}^{T} \log(\sigma(y_t\langle x_t, \theta^{\star}\rangle))  + 6d\log (eT).
\end{align*}
The claim follows.
\end{proof}

\section{Non-curved losses}
\label{sec:non-curve}

In the previous sections, we focused on mixable losses, such as the logistic loss and the squared loss. These losses are generally known to yield $O(\log T)$-type regret bounds. We now turn our attention to another unbounded loss function, namely the hinge loss, which lacks mixability properties and thus compels us to focus on the so-called \emph{first-order} regret bounds. These bounds scale as $O(\sqrt{T})$ in the worst case but can improve if the total loss of the optimal solution is close to zero.

\subsection{Regret bounds for hinge loss}
\label{sec:hinge}
In this section, we first obtain new margin-dependent bounds for the hinge loss in the sequential learning setup using an exponential weight algorithm; then, we proceed to use a transductive prior to achieve an assumption-free regret bound. In the classification setup with classes $1, -1$, given $\gamma > 0$, we define the normalized hinge loss of a predictor $\widehat{f}$ as follows 
\[
\frac{(\gamma - y\widehat{f}(x))_{+}}{\gamma}.
\]
There is a standard connection with the binary loss. Indeed, 
    \begin{equation}
    \label{eq:binarytohinge}
    \ind{\sign(\widehat{f}(x))\neq y} \le \frac{(\gamma - y\widehat{f}(x))_{+}}{\gamma}.
    \end{equation}
In the agnostic case where the design vectors are not known in advance, the state-of-the-art regret bound is given by the \emph{second order perceptron} algorithm of \citet*{cesa2005second}. Denoting the output of this predictor with regularization parameter $\lambda > 0$ at round $t$ by $\widehat{y}_t \in \{1, -1\}$ (see \citep[Section 12.2]{cesa2006prediction} for the exact definition of this predictor), we have for any $\theta^{\star} \in \mathbb{R}^d$,
\begin{equation}
\label{eq:secorderdperceptron}
  \sum\limits_{t = 1}^T\ind{\widehat{y}_t\neq y_t} \le \sum\limits_{t = 1}^T\frac{(\gamma - y_t\langle x_t, \theta^{\star}\rangle)_{+}}{\gamma} + \frac{1}{\gamma}\sqrt{\left(\lambda\|\theta^\star\|^2 + \|A_T^{1/2}\theta^{\star}\|^2\right)\sum\limits_{i = 1}^d\log\left(1 + \frac{\lambda_i}{\lambda}\right)},  
\end{equation}
where $\lambda_1, \ldots, \lambda_d$ are the eigenvalues of the matrix $A_T = \sum\nolimits_{t = 1}^Tx_tx_t^{\top}\ind{\widehat{y}_t \neq y_t}$.

This regret bound \eqref{eq:secorderdperceptron}, although corresponding to a computationally efficient algorithm, has several features we are aiming to avoid. One of the key problems is that the bound depends on the term $\|A_T^{1/2}\theta^{\star}\|^2$ which scales badly with the magnitude of $\max_t|\langle x_t, \theta^\star\rangle|$. 

Our first result in this section corresponds to the analysis of the exponential weights algorithm, whose prior is given by the general covariance matrix. Given a positive-definite matrix $A \in \mathbb{R}^{d \times d}$ and $\lambda > 0$, set 
\begin{equation}
    \label{eq:prior}
    \mu(\theta) = \sqrt{ \frac{\det(\lambda A)}{\pi^d}}\cdot\exp\left(-\pi\lambda\theta^\top A\theta\right).
\end{equation}
Fix $\eta > 0$. Using the general formula for the exponential weights algorithm \eqref{eq:expweights}, we set $\rho_1 = \mu$ and for $t \ge 2$,
\[
\rho_t(\theta) = \frac{\exp\left(-\eta \sum\limits_{i = 1}^{t - 1} (\gamma - y_{i}\langle x_i, \theta\rangle)_+\right)\mu(\theta)}{\mathbb{E}_{\theta^{\prime} \sim \mu} \exp\left(-\eta \sum\limits_{i = 1}^{t - 1} (\gamma - y_{i}\langle x_i, \theta^{\prime}\rangle)_+\right)}.
\]
Given the positive-definite matrix $A$ defining the Gaussian prior \eqref{eq:prior}, the prediction of our algorithm at round $t$ is given by 
\begin{equation}
\label{eq:hingepredictor}
\widehat{f}_t(x_t) = \E_{\theta\sim\rho_t}\clipg{\langle x_t, \theta\rangle}.
\end{equation}
The computation of such prediction depends on the ability of sampling from $\rho_t$, which is a log-concave measure. 
Indeed, once we can approximately sample from $\rho_t$, we can use a simple estimate based on Hoeffding's inequality to identify the number of samples required from $\rho_t$ in order to approximate $\widehat{f}_t(x_t)$ with the desired accuracy.

\begin{proposition}
\label{prop:hingeexpweights}
For any $\eta \in [0, 3/(10\gamma)]$, $\beta > 0$, and any $\theta^{\star} \in \mathbb{R}^{d}$ the following bound holds for the predictor \eqref{eq:hingepredictor}:
\begin{align*}
\sum\limits_{t = 1}^T\frac{(\gamma - y_t\widehat{f}_t(x_t))_{+}}{\gamma}&\le (1 + 2\eta\gamma)\cdot\Bigg(\sum\limits_{t = 1}^T\frac{(\gamma - y_t\langle x_t, \theta^{\star}\rangle)_+}{\gamma} + \frac{1}{\pi\gamma\sqrt{\beta}}\sum\limits_{t = 1}^T\|A^{-1/2}x_t\|_2 \\
&\qquad + \frac{\pi\lambda(\theta^{\star})^{\top}A\theta^{\star}}{\eta\gamma} + \frac{d}{2\eta\gamma}\left(\log\left(\frac{\beta}{\lambda}\right) - \frac{1}{2} + \frac{\lambda}{2\beta}\right) \Bigg).
\end{align*}

\end{proposition}

\begin{proof}
    The proof follows the standard lines. First, observe that Jensen's inequality implies
    \begin{equation}
    \label{eq:convexhinge}
       \sum\limits_{t = 1}^T(\gamma - y_t\widehat{f}_t(x_t))_{+} \le \sum\limits_{t = 1}^T\E_{\theta \sim \rho_{t}}(\gamma - y_t\clipg{\langle x_t, \theta\rangle})_{+}.
    \end{equation}
    Using Bernstein's inequality for the moment generating function and \[\E_{\theta\sim\rho}(\gamma - \clipg{y\langle x, \theta\rangle})_{+}^2 \le 2\gamma\E_{\theta\sim\rho}(\gamma - \clipg{y\langle x, \theta\rangle})_{+},
    \]we have for any distribution $\rho$ over $\mathbb{R}^d$,
\begin{align*}
    &\frac{1}{\eta}\log\left(\mathbb{E}_{\theta \sim \rho}\exp\left(-\eta(\gamma - \clipg{y\langle x, \theta\rangle})_{+}\right)\right) \\
    &\qquad\le -\mathbb{E}_{\theta \sim \rho}(\gamma - \clipg{y\langle x, \theta\rangle})_{+} + \frac{\eta\gamma \mathbb{E}_{\theta \sim \rho}(\gamma - \clipg{y\langle x, \theta\rangle})_{+}}{1 - 2\eta\gamma/3}.
\end{align*}
Combining this inequality with \eqref{eq:convexhinge} and using \cref{lem:sumofmixlosses}, we obtain
\begin{align*}
\sum\limits_{t = 1}^T\left(1 - \frac{\eta \gamma}{1 - 2\eta \gamma/3}\right)(\gamma - y_t\widehat{f}_t(x_t))_{+} &\le -\sum\limits_{t = 1}^{T}\frac{1}{\eta}\log\left(\mathbb{E}_{\theta \sim \rho_t}\exp\left(-\eta(\gamma - \clipg{y_t\langle x_t, \theta\rangle})_{+}\right)\right)
\\
&\le -\sum\limits_{t = 1}^{T}\frac{1}{\eta}\log\left(\mathbb{E}_{\theta \sim \rho_t}\exp\left(-\eta(\gamma - y_t\langle x_t, \theta\rangle)_{+}\right)\right)
\\
&= -\frac{1}{\eta}\log\left(\mathbb{E}_{\theta \sim \mu}\exp\left(-\eta\sum\limits_{t = 1}^{T}(\gamma - y_t\langle x_t, \theta\rangle)_{+}\right)\right).
\end{align*}
We upper bound the last quantity using the Donsker-Varadhan formula. Fix $\beta > 0$ and let $\mu^{\star}$ the multivariate Gaussian distribution with density
\[
\mu^{\star}(\theta) = \sqrt{\det(\beta A)}\cdot\exp\left(-\pi(\theta - \theta^{\star})^\top \beta A(\theta - \theta^{\star})\right).
\]
By \cref{lem:donskervaradhan}, we have
\[
-\frac{1}{\eta}\log\left(\mathbb{E}_{\theta \sim \mu}\exp\left(-\eta\sum\limits_{t = 1}^{T}(\gamma - y_t\langle x_t, \theta\rangle)_{+}\right)\right) \le \mathbb{E}_{\theta \sim \mu^{\star}}\sum\limits_{t = 1}^{T}(\gamma - y_t\langle x_t, \theta\rangle)_{+} + \frac{\mathcal{KL}(\mu^{\star}||\mu)}{\eta}.
\]
We analyze the two terms of the right-hand side of the last inequality separately. Define the density $\mu_0(\theta) = \sqrt{\det(\beta A)}\cdot\exp\left(-\pi\theta^\top \beta A\theta\right).$ Using the standard properties of the Gaussian distribution we have
\begin{align*}
\E_{\theta \sim \mu^{\star}}\sum\limits_{t = 1}^T(1 - y_t\langle x_t, \theta\rangle)_+ &\le \sum\limits_{t = 1}^T(1 - y_t\langle x_t, \theta^{\star}\rangle)_+ + \sum\limits_{t = 1}^T\E_{\theta \sim \mu_0}|\langle \theta, x_t\rangle|
\\
&= \sum\limits_{t = 1}^T(1 - y_t\langle x_t, \theta^{\star}\rangle)_+ + \frac{1}{\pi\sqrt{\beta}}\sum\limits_{t = 1}^T\|A^{-1/2}x_t\|_2.
\end{align*}
Another standard formula implies
\[
\mathcal{KL}(\mu^{\star}||\mu) = \pi\lambda(\theta^{\star})^{\top}A\theta^\star + \frac{d}{2}\left(\log\left(\frac{\beta}{\gamma}\right) - \frac{1}{2} + \frac{\lambda}{2\beta}\right).
\]
Combining all the bounds together, we get
\begin{align}
\label{eq:hingefinal}
\revindent\sum\limits_{t = 1}^T\left(1 - \frac{\eta \gamma}{1 - 2\eta \gamma/3}\right)\cdot\frac{(\gamma - y_t\widehat{f}_t(x_t))_{+}}{\gamma} \nonumber\\
&\le \sum\limits_{t = 1}^T\frac{(1 - y_t\langle x_t, \theta^{\star}\rangle)_+}{\gamma} + \frac{1}{\pi\gamma\sqrt{\beta}}\sum\limits_{t = 1}^T\|A^{-1/2}x_t\|_2
+ \frac{\pi\lambda(\theta^{\star})^{\top}A\theta^\star}{\eta\gamma} \nonumber\\
&\quad+ \frac{d}{2\eta\gamma}\left(\log\left(\frac{\beta}{\lambda}\right) - \frac{1}{2} + \frac{\lambda}{2\beta}\right).
\end{align}
Finally, we observe that for any $\eta \le \frac{3}{10\gamma}$, we have
\[
\left(1 - \frac{\eta \gamma}{1 - 2\eta \gamma/3}\right)^{-1} \le 1 + 2\eta\gamma.
\]
Plugging this into \eqref{eq:hingefinal}, we prove the claim.
\end{proof}

There are some immediate corollaries of the result of \cref{prop:hingeexpweights}. First, we consider the standard bounded online setup where the vectors $x_1, \ldots, x_T$ are not known in advance.  
\begin{corollary}
\label{cor:inductivehinge}
Assume that $\|x_t\| \le r$ and let $\theta^{\star} \in \mathbb{R}^d$ be such that $\|\theta^{\star}\| \le b$. Then, for any $\eta \in [0, 3/(10\gamma)]$, the predictor \eqref{eq:hingepredictor} with $A = I_d$ and $\lambda = \frac{1}{b^2}$ satisfies
\[
    \sum\limits_{t = 1}^T\frac{(\gamma - y_t\widehat{f}_t(x_t))_{+}}{\gamma} \le (1 + 2\eta\gamma)\left(\sum\limits_{t = 1}^T\frac{(\gamma - y_t\langle x_t, \theta^{\star}\rangle)_+}{\gamma} + \frac{\frac{1}{\pi} + \pi + \frac{d}{2}\log\left(1 + \eta^2T^2r^2b^2\right)}{\eta\gamma} \right).
\]
    In particular, when the data is linearly separable with margin $\gamma$, that is, for the same $\theta^{\star}$ we have $y_{t}\langle x_t, \theta^{\star}\rangle \ge \gamma$ for $t = 1, \ldots, T$, by fixing $\eta = 3/(10\gamma)$ we obtain
    \begin{equation}
    \label{eq:perceptron}
    \sum\limits_{t = 1}^T\frac{(\gamma - y_t\widehat{f}_t(x_t))_{+}}{\gamma} \le \frac{7}{3}\left(\frac{1}{\pi} + \pi + \frac{d}{2}\log\left(1 + \frac{9T^2r^2b^2}{100\gamma^2}\right)\right).
    \end{equation}
\end{corollary}

\begin{proof}
The first claim follows from \cref{prop:hingeexpweights} if we plug in $A = I_d$, $\lambda = \frac{1}{b^2}$, $\beta = \frac{1}{b^2} + \eta^2T^2r^2$. The second part of the claim follows from the fact that in the linearly separable case with margin $\gamma$ we have $\sum\nolimits_{t = 1}^T(\gamma- y_t\langle x_t, \theta^{\star}\rangle)_+ = 0$. The claim follows.
\end{proof}

Finally, we present our main result for the transductive setup, where we employ the transductive prior by selecting a data-dependent matrix $A$ in \eqref{eq:prior}. Additionally, we note that a similar sparsity regret bound can be provided, analogous to the one in \cref{thm:sparsity}.

\subsection{Slab-experts for hinge loss}

In this section, we show that slab experts can also be leveraged to obtain assumption-free bounds for hinge loss. To start, we present the following comparison lemma for hinge loss.

\begin{lemma}
    \label{lem:comparison-for-hinge-loss}
For any sequence $\set{(x_t,y_t)}_{1\leq t\leq T}$ where $(x_t,y_t)\in \bR^d\times \set{\pm 1}$ for all $1\leq t\leq T$, suppose $\thstar$ is the optimal comparator in hindsight, i.e., 
\begin{align}
\thstar \ldef \argmin_{\theta \in \mathbb{R}^d}\sum\limits_{t = 1}^T
\frac{(\gamma - y_t\langle x_t, \theta\rangle)_+}{\gamma}
\label{eq:theta-star-hinge}
\end{align}
For this comparator
and any $1\leq t\leq T$, if $y_t\langle x_t, \thstar\rangle \leq 0$, then $|\langle x_t,\thstar\rangle| \le T\gamma$.
\end{lemma}

\begin{proof}[\pfref{lem:comparison-for-hinge-loss}]
For any $t\in [T]$, suppose $y_t\langle x_t, \thstar\rangle \leq 0$. Since $\thstar$ minimizes the loss, we have
\begin{align*}
    \frac{\gamma + \abs*{\langle x_t, \thstar\rangle}}{\gamma} =\frac{(\gamma - y_t\langle x_t, \thstar\rangle)_+}{\gamma} \leq  \sum\limits_{t = 1}^T
    \frac{(\gamma - y_t\langle x_t, \thstar\rangle)_+}{\gamma}  \leq T.
\end{align*}
Reorganizing terms, we obtain the desired result.
\end{proof}

\begin{definition}[Slab-expert for hinge loss]
    \label{def:slab-expert-hinge}
A slab-expert for the hinge loss $\SE_{\theta,\alpha,\gamma}$ is an expert that is specified by a vector $\theta\in \bR^d$, a regularization parameter $\alpha$, and a margin $\gamma>0$.
The expert constructs 
a subset of covariates that belong to a single slab, namely
$K_{\theta,\gamma} = \set{ x_t:\abs*{\inner{x_t,\theta}} \leq \gamma T} $ and predicts labels in the following way: At time step $t$, if $x_t\in K_{\theta,\gamma}$, the expert predicts according to $\clipg{\ewa_{K_{\theta,\gamma}, \alpha}}$ with hinge loss and learning rate $\eta$ on $x_t$, where $\ewa_{K_{\theta,\gamma}, \alpha}$ is the partial exponential weight expert defined as in \cref{def:partial-ewa} with subset $K_{\theta,\gamma}$ and regularization parameter $\alpha$. If $x_t\notin K_{\theta,\gamma}$, then the expert predicts $\yhat_t=\gamma\sign(\inner{x_t,\theta})$. We denote by $\SE_{\theta,\alpha, \gamma}(x_t)$ the value predicted on receiving the covariate $x_t$.
\end{definition}

\newcommand{\ytil}{\tilde{y}}
\begin{algorithm}
\caption{Aggregation over slab-experts for the hinge loss}
\label{alg:main-hinge}
\begin{algorithmic}[htb]
\Require The set of (unordered) covariates $\set{x_t}_{t\in [T]}$, margin $\gamma$, learning rate $\eta$, regularization parameter $\alpha=1/(\gamma^2 T^3)$.
\State Construct the set of unique slab-experts for the hinge loss $\cA =\set{\SE_{\theta,\alpha,\gamma}\mid{}\theta\in \bR^d}$ with the hinge loss function $\ell_\theta(x,y)= (\gamma - y_t\langle x_t, \theta\rangle)_+$ and learning rate $\eta$.
\For{$t=1,2,\ldots,T$}
\State The learner receives $x_t$.
\State \multiline{The learner predicts with aggregation on the slab-experts 
\[
    \fhat_t(x_t) =  \sum\limits_{\SE\in \cA} p_t(\SE) \SE(x_t) \text{~~where~~} p_t(\SE) \propto \exp\prn*{ \sum\limits_{i=1}^{t-1}  - \eta(\gamma - y_i\SE(x_i))_{+} }.
\]
}
\State Nature reveals $y_t$.
\State Update all $\SE$ in $\cA$ with $(x_t,y_t)$.
\EndFor 
\end{algorithmic}
\end{algorithm}

\begin{theorem}
    \label{thm:slab-experts-for-hinge}
    For any margin $\gamma>0$, $\eta \in [0, 3/(10\gamma)]$,
    and $\theta^{\star}$ as defined in \eqref{eq:theta-star-hinge},   
    the estimators $\fhat_1(x_1),..,\fhat_T(x_T)$ produced by \cref{alg:main-hinge} satisfy the total loss upper bound of
    \begin{align*}
        \sum\limits_{t = 1}^T\frac{(\gamma - y_t \fhat_t(x_t))_{+}}{\gamma}&\le (1 + 2\eta\gamma)^2 \prn*{ \sum\limits_{t = 1}^T
            \frac{(\gamma - y_t\langle x_t, \thstar\rangle)_+}{\gamma} } +    \frac{7d}{\eta \gamma} (1 + 2\eta\gamma)^2   \log(\pi T)  \\
            &\qquad+  \frac{1}{\eta\gamma} (1 + 2\eta\gamma)^2.
    \end{align*}
    In particular, when the data is linearly separable with margin $\gamma$, that is, for the same $\theta^{\star}$ we have $y_{t}\langle x_t, \theta^{\star}\rangle \ge \gamma$ for $t = 1, \ldots, T$, by fixing $\eta = 3/(10\gamma)$ we obtain
    \begin{equation}
    \label{eq:sumofhingelosses}
        \sum\limits_{t = 1}^T\frac{(\gamma - y_t \fhat_t(x_t))_{+}}{\gamma} \le 9 +  60d \log(\pi T).
    \end{equation}
\end{theorem}

Observe that due to \eqref{eq:binarytohinge}, the bound of \eqref{eq:sumofhingelosses} upper-bounds the number of mistakes made by our predictor.

\begin{proof}[\pfref{thm:slab-experts-for-hinge}]
Similar to the logistic case, we first upper bound the total loss by the loss of slab-expert $\SE_{\thstar,\alpha,\gamma}$ as defined in \cref{def:slab-expert-hinge}.
First, observe that Jensen's inequality implies
    \begin{equation}
    \label{eq:convexhinge-slab}
       \sum\limits_{t = 1}^T(\gamma - y_t\widehat{f}_t(x_t))_{+} \le \sum\limits_{t = 1}^T\E_{\SE \sim p_{t}}(\gamma - y_t\SE(x_t))_{+}.
    \end{equation}
    Using Bernstein's inequality for the moment generating function together with $\E_{\SE \sim p_{t}}(\gamma - y_t\SE(x_t))_{+}^2 \le 2\gamma\E_{\SE \sim p_{t}}(\gamma - y_t\SE(x_t))_{+}$, we have for any distribution $p$ over $\mathbb{R}$,
\begin{align*}
    \frac{1}{\eta}\log\left(\mathbb{E}_{y' \sim p}\exp\left(-\eta(\gamma - y y')_{+}\right)\right) \le -\mathbb{E}_{y'\sim p}(\gamma - y y')_{+} + \frac{\eta\gamma \mathbb{E}_{y'\sim p}(\gamma - y y')_{+}}{1 - 2\eta\gamma/3}.
\end{align*}
Combining this inequality with \eqref{eq:convexhinge-slab} and using Lemma \cref{lem:sumofmixlosses}, we obtain
\begin{align*}
\revindent\sum\limits_{t = 1}^T\left(1 - \frac{\eta \gamma}{1 - 2\eta \gamma/3}\right)\cdot(\gamma - y_t\widehat{f}_t(x_t))_{+} \\
&\le -\sum\limits_{t = 1}^{T}\frac{1}{\eta}\log\left(\mathbb{E}_{A\sim p_t}\exp\left(-\eta(\gamma - y_t\SE(x_t) )_{+}\right)\right)
\\
&= -\sum\limits_{t=1}^T\frac{1}{\eta}  \log \prn*{ \frac{\sum\limits_{\SE\in \cA} \exp\prn*{ -\eta \sum\limits_{i=1}^{t} (\gamma - y_i\SE(x_i))_{+} }}{\sum\limits_{\SE\in \cA} \exp\prn*{ -\eta\sum\limits_{i=1}^{t-1} (\gamma - y_i\SE(x_i))_{+}  }}} \\
&= -\frac{1}{\eta} \log \prn*{ \frac{1}{|\cA|} \sum\limits_{\SE\in \cA} \exp\prn*{-\eta \sum\limits_{i=1}^{T} (\gamma - y_i\SE(x_i))_{+} }}\\
&\leq - \frac{1}{\eta}\log \prn*{  \exp\prn*{  -\eta\sum\limits_{i=1}^{T} (\gamma - y_i\SE_{\thstar,\alpha,\gamma}(x_i))_{+} }}+ \log |\cA|\\
&\leq  \sum\limits_{i=1}^{T} (\gamma - y_i\SE_{\thstar,\alpha,\gamma}(x_i))_{+}+ \frac{4d}{\eta} \log(eT/d),
\end{align*}
where the last inequality is due to \cref{lem:num-slab-expert}.
We apply \cref{prop:hingeexpweights} to the slab-expert $\SE_{\thstar,\alpha,\gamma}$ with $\alpha = \pi\lambda=1/(\gamma^2T^3)$ and $A= \sum\limits_{x\in K_{\thstar,\gamma}} xx^\top$ to obtain
\begin{align}
    \label{eq:hinge-logT-1}
    \begin{split}
        \revindent\sum\limits_{x\in K_{\thstar,\gamma}}\frac{(\gamma - y_x\SE_{\thstar,\alpha,\gamma}(x))_{+}}{\gamma}
        \\&\le (1 + 2\eta\gamma)\cdot\Bigg(\sum\limits_{x\in K_{\thstar,\gamma}}\frac{(\gamma - y_x\langle x, \theta^{\star}\rangle)_+}{\gamma} + \frac{1}{\pi\gamma\sqrt{\beta}}\sum\limits_{x\in K_{\thstar,\gamma}}\|A^{-1/2}x\|_2\\
        &\qquad + \frac{\pi\lambda(\theta^{\star})^{\top}A\thstar}{\eta\gamma} + \frac{d}{2\eta\gamma}\left(\log\left(\frac{\beta}{\lambda}\right) - \frac{1}{2} + \frac{\lambda}{2\beta}\right) \Bigg).
    \end{split}
\end{align} 
Set $\beta = T/(\pi^2\gamma^2)$. 
We can bound the additive terms separately by
\begin{align*}
    \frac{1}{\pi\gamma\sqrt{\beta}}\sum\limits_{x\in K_{\thstar,\gamma}}\|A^{-1/2}x\|_2 &=  \frac{1}{\sqrt{T}} \sum\limits_{x\in K_{\thstar,\gamma}}\sqrt{ x^\top A^{-1}x } \leq \frac{1}{\sqrt{Td}} \sqrt{ T \sum\limits_{x\in K_{\thstar,\gamma}}x^\top A^{-1}x } \leq \sqrt{d}
\end{align*}
and
\begin{align*}
    \frac{\pi\lambda(\theta^{\star})^{\top}A\thstar}{\eta\gamma} \leq \frac{1}{\eta\gamma^3T^3}  \sum\limits_{x\in K_{\thstar,\gamma}} \inner{x,\thstar}^2 \leq \frac{1}{\eta\gamma^3T^3}  \gamma^2 T^3 = \frac{1}{\eta\gamma}.
\end{align*}
Plug the above inequalities to \eqref{eq:hinge-logT-1}, we can obtain
\begin{align*}
    \revindent\sum\limits_{x\in K_{\thstar,\gamma}}\frac{(\gamma - y_x\SE_{\thstar,\alpha,\gamma}(x))_{+}}{\gamma}\\
    &    \leq (1 + 2\eta\gamma)\cdot\Bigg(\sum\limits_{x\in K_{\thstar,\gamma}}\frac{(\gamma - y_x\langle x, \theta^{\star}\rangle)_+}{\gamma} + \sqrt{d} + \frac{1}{\eta\gamma}+ \frac{2d}{\eta\gamma}\log\left(\pi T\right)  \Bigg)\\
    &\leq  (1 + 2\eta\gamma)\cdot\Bigg(\sum\limits_{x\in K_{\thstar,\gamma}}\frac{(\gamma - y_x\langle x, \theta^{\star}\rangle)_+}{\gamma} +  \frac{1}{\eta\gamma} +\frac{3d}{\eta\gamma}\log\left(\pi T\right)  \Bigg).
\end{align*}
Meanwhile, for all the covariates outside the slab, the prediction of $\SE_{\thstar,\alpha}$ coincides with the linear separator with vector $\thstar$. Moreover, $y_x\inner{x,\thstar} > 0$ for all $x\in \set{x_t}_{t\in [T]}\setminus K_{\thstar,\gamma}$.
Thus, we have as a result
\begin{equation}
    \label{ineq:large-margin}
    \sum_{x\in \set{x_t}_{t\in [T]}\setminus K_{\thstar,\gamma}} \frac{(\gamma - y_x\SE_{\thstar,\alpha,\gamma}(x))_{+}}{\gamma} = 0.
\end{equation}
Altogether, we obtain
\begin{align*}
    &\sum\limits_{t = 1}^T\frac{(\gamma - y_t \fhat_t(x_t))_{+}}{\gamma}\\
    &\leq  \prn*{1+2\eta\gamma} \prn*{ \sum\limits_{i=1}^{T}  \frac{(\gamma - y_i\SE_{\thstar,\alpha,\gamma}(x_i))_{+}}{\gamma} + \frac{4d}{\eta\gamma} \log(eT/d)  }\\
    &\leq  (1 + 2\eta\gamma) \prn*{(1 + 2\eta\gamma)\cdot\Bigg(\sum\limits_{x\in K_{\thstar,\gamma}}\frac{(\gamma - y_x\langle x, \theta^{\star}\rangle)_+}{\gamma} +  \frac{1}{\eta\gamma} +\frac{3d}{\eta\gamma}\log\left(\pi T\right)  \Bigg) + \frac{4d}{\eta\gamma} \log(eT)     }  \\
    &\le (1 + 2\eta\gamma)^2 \prn*{ \sum\limits_{t = 1}^T
        \frac{(\gamma - y_t\langle x_t, \thstar\rangle)_+}{\gamma} } +   \frac{7d}{\eta \gamma}  (1 + 2\eta\gamma)^2   \log(\pi T) + \frac{1}{\eta\gamma} (1 + 2\eta\gamma)^2.
\end{align*}
The claim follows.
\end{proof}

\section{Computational Aspects}
\label{sec:computational}
\newcommand{\halfspace}{\mathrm{HalfSpace}}
\newcommand{\svm}{\mathrm{SVM}}
\newcommand{\hhat}{\widehat{h}}

In this section, we briefly discuss the computational properties of our algorithms in the problems considered so far, namely, sparse linear regression, logistic regression, and hinge loss regression. 

\subsection{Computation for sparse linear regression}
In this section, we discuss the computational properties of our methods in \cref{thm:sparsity} and \cref{prop:expscreening}.  
Concretely, we would like to compute the quantity $\widehat{f}_t(x_t)$ as defined in \eqref{eq:predictionformulasparse}.
We provide three possible approaches for this computation. First, set
\begin{align*}
{L}_{m}(\theta) &= \frac{1}{2m^2}(m - \langle x_t, \theta \rangle)^2 + \frac{1}{2m^2}\sum\limits_{i = 1}^{t - 1} (y_i - \langle x_i, \theta \rangle)^2, \\
{L}_{-m}(\theta) &= \frac{1}{2m^2}(- m - \langle x_t, \theta \rangle)^2 + \frac{1}{2m^2}\sum\limits_{i = 1}^{t - 1} (y_i - \langle x_i, \theta \rangle)^2.
\end{align*}
The first approach is to notice that computing $\widehat{f}_t(x_t)$ is equivalent to computing
\[
\frac{\int\limits_{\mathbb{R}^d}^{} \exp\left(-{L}_{m}(\theta)\right) \mu(\theta) d\theta}{\int\limits_{\mathbb{R}^d}^{} \exp\left(-{L}_{-m}(\theta)\right) \mu(\theta) d\theta} =  \int_{\bR^d}  \frac{\exp\prn*{-V(\theta)}}{\int_{\bR^d} \exp \prn*{-V(\theta')} d\theta'} \cdot \exp\prn*{ \frac{2}{m} \inner{x_t,\theta} } d\theta,
\]
where 
\begin{align*}
    V(\theta) = {L}_{-m}(\theta)
    + 4 \sum_{j=1}^d\log \prn*{1 + |\theta^{(j)}| \cdot \sqrt{\sum\nolimits_{t = 1}^{T} (x_t^{(j)})^2} / \tau}.
\end{align*}
Thus, we can sample from the distribution with density proportional to $\exp\prn*{-V(\theta)}$, following almost verbatim the approach in \citet[Section 4]{dalalyan2012sparse}, who utilized the Langevin Monte Carlo method\footnote{Note that, in our setting, it remains unclear whether the Langevin Monte Carlo method admits a worst-case polynomial-time implementation. This limitation also applies to the results of \citet{dalalyan2012sparse}.}. Using these samples, we can evaluate the expectation of $\exp\prn*{\frac{2}{m} \inner{x_t, \theta}}$ under that distribution.

The second approach is to notice that 
\[
    \frac{\int\limits_{\mathbb{R}^d}^{} \exp\left(-{L}_{m}(\theta)\right) \mu(\theta) d\theta}{\int\limits_{\mathbb{R}^d}^{} \exp\left(-{L}_{-m}(\theta)\right) \mu(\theta) d\theta} = \frac{\En_{\theta\sim \xi_1} \brk*{ \prod_{j=1}^d\prn*{{1 + |\theta^{(j)}| \cdot \sqrt{\sum\nolimits_{t = 1}^{T} (x_t^{(j)})^2} / \tau}}^{-4} } }{\En_{\theta\sim \xi_2} \brk*{ \prod_{j=1}^d\prn*{{1 + |\theta^{(j)}| \cdot \sqrt{\sum\nolimits_{t = 1}^{T} (x_t^{(j)})^2} / \tau}}^{-4} }},
\]
where 
\begin{align*}
\xi_1(\theta) &\propto \exp\left(-\frac{1}{2m^2}(m - \langle x_t, \theta \rangle)^2 - \frac{1}{2m^2} \sum\limits_{i = 1}^{t - 1} (y_i - \langle x_i, \theta \rangle)^2\right), 
\\
\xi_2(\theta) &\propto \exp\left(-\frac{1}{2m^2}(-m - \langle x_t, \theta \rangle)^2 - \frac{1}{2m^2} \sum\limits_{i = 1}^{t - 1} (y_i - \langle x_i, \theta \rangle)^2\right)
\end{align*}
are both Gaussian distributions. If sufficiently many Gaussian samples from $\xi_1$ and $\xi_2$ are obtained to approximate both expectations up to multiplicative constants close to one, the value of $\widehat{f}_t(x_t)$ can then be approximated within the desired precision level.

Finally, we can analyze Laplace's approximation, which is computationally efficient.
Specifically, we can compute two least squares solutions, $\widehat{\theta}_{m} = \arg\min_{\theta \in \mathbb{R}^d} {L}_m(\theta)$ and $\widehat{\theta}_{-m} = \arg\min_{\theta \in \mathbb{R}^d} {L}_{-m}(\theta)$. When the asymptotic error of the Laplace's method is small, i.e., 
\begin{align*}
\begin{cases}
    \int\limits_{\mathbb{R}^d}^{} \exp\left(-{L}_{m}(\theta)\right) \mu(\theta) \, d\theta \approx \exp\left(-{L}_{m}(\widehat{\theta}_{m})\right) \mu(\widehat{\theta}_{m}), \\
    \int\limits_{\mathbb{R}^d}^{} \exp\left(-{L}_{-m}(\theta)\right) \mu(\theta) \, d\theta \approx \exp\left(-{L}_{-m}(\widehat{\theta}_{-m})\right) \mu(\widehat{\theta}_{-m}).
\end{cases}
\end{align*}
Using these solutions, we can express the following approximate identity:
\[
\frac{\int\limits_{\mathbb{R}^d}^{} \exp\left(-{L}_{m}(\theta)\right) \mu(\theta) \, d\theta}{\int\limits_{\mathbb{R}^d}^{} \exp\left(-{L}_{-m}(\theta)\right) \mu(\theta) \, d\theta} 
\approx 
\frac{\exp\left(-{L}_{m}(\widehat{\theta}_{m})\right) \mu(\widehat{\theta}_{m})}{\exp\left(-{L}_{-m}(\widehat{\theta}_{-m})\right) \mu(\widehat{\theta}_{-m})}.
\]
We leave a rigorous analysis of the last two approaches as an interesting direction for future research.

\subsection{Computation for logistic regression}
Motivated by the high computational complexity of the algorithm proposed by \citet{foster2018logistic}, the computational problem of online logistic regression has been further studied in \citep{jezequel2020efficient, jezequel2021mixability}. However, since their algorithms are designed for the online setup, the lower bound in \cref{thm:logistic-inductive-lower-bound} shows that they cannot achieve the regret upper bound achieved by our \cref{alg:main}. On the other hand, the Normalized Maximum Likelihood (NML) algorithm is known to achieve exact minimax regret \citep{wu2022sequential}, but its computation requires exhausting all possible label sequences $y_1, \ldots, y_T \in \set{\pm 1}^T$, which scales as $2^T$. Consequently, without additional algorithmic insights, the NML algorithm cannot be computed in polynomial time. Methods based on $\veps$-nets typically work for the problems we consider but suffer from the lack of known implementable algorithms constructing the $\veps$-nets (see \cref{app:covering} for detailed discussion).

In this section, we discuss a computable variant of \cref{alg:main}, where, however, the computational cost can scale in the worst case as $O(\poly(T^{d}))$ 
The computation of \cref{alg:main} presents two main challenges:
\begin{enumerate}
    \item How do we compute a prediction from a slab-expert within the slab? This is equivalent to determining how to compute the probability assigned when the covariate has a small margin with respect to a slab-expert, i.e., how to compute the Gaussian integral in this case.
    \item How do we aggregate the predictions from the slab-experts?
\end{enumerate}
We address both in what follows.

\subsubsection{Computing slab-expert prediction} 
\label{sec:compute-slab-expert-prediction}

For any $\theta\in \bR^d$, to produce our prediction for the slab-expert $\SE_{\theta,\alpha}$ as defined in \cref{def:partial-ewa} at time step $t$ on a  point $x_t$ inside a slab $K_\theta\subset \set{x_i}_{i\in [T]}$ , we need to estimate probability assignments defined as $\SE_{\theta,\alpha}(x_t,1) = \int \sigma(\inner{x_t,\theta}) \pi_t(d\theta)$ 
where the density is
\begin{align*}
    \pi_t(\theta)\propto\exp\left(\sum\limits_{x\in K_{t-1}}\log(\sig(y_x\langle x, \theta \rangle)) -\alpha\theta^{\top}\left(\sum\nolimits_{x\in K_\theta}xx^{\top}\right)\theta \right),
\end{align*}
where $K_{t-1} = \set{x_i}_{i\in [t-1]}\cap K_\theta$. We can estimate $\SE_{\theta,\alpha}(x_t,1)$ by sampling $\theta$ according to $\pi_t$ and averaging over $\sigma(-\inner{x_t,\theta})$. Sampling from the distribution $\pi_t$ has been studied in the literature \citep[Example 2]{dalalyan2017theoretical}. 
The LMC algorithm \citep{dalalyan2017theoretical} can sample from a distribution $\veps$-close to the distribution $\pi$, in TV distance
with a number of steps scaling with 
\begin{align*}
O \prn*{ \prn*{ \frac{\alpha+1/4}{\alpha} }^2 d \veps^{-2} \prn*{  2\log (1/\veps) + (d/2)\log \prn*{ \frac{\alpha+1/4}{\alpha} }  }^2  } = O(\poly(T,d))
\end{align*}
when $\veps = 1/T$ and recall $\alpha=1/T^3$ in \cref{alg:main}.
 Thus, since the function $\sigma(-\inner{x_t,\theta})$ is bounded in $[0,1]$, by sampling $T^2$ times, and using Hoeffding's inequality for concentration,
we can estimate the assignment $\SE_{\theta,\alpha}(x_t,1)$ with accuracy to the level of $O(1/T)$.

\subsubsection{Aggregating the slab-experts}

Suppose $\mu$ is a prior on any set of experts that contains the optimal slab-expert $\SE_{\thstar,\alpha}$. Recall from \cref{sec:slab-experts}, the mixability of the logistic loss implies that the regret of the exponential weights algorithm over this set will suffer at most a regret scaling with $O(\log (1/\mu(\SE_{\thstar,\alpha})) + d\log T)$. 
While it is challenging to identify all the slab-experts due to computational constraints, we can consider a broader class of generalized slab-experts. These generalized slab-experts extend the concept of slab-experts by allowing the slab to be defined as the intersection of two halfspaces.

Consider all the half-space separations $\cS = \set{(K_+,K_-)}$ where $K_+ \cup K_- = \set{x_t}_{t\in [T]}$, and there exists a vector $\theta$ and an intercept $b$ such that either for all $x \in K_+$, $\inner{x,\theta} - b \geq 0$ and for all $x \in K_-$, $\inner{x,\theta} - b < 0$.
\begin{definition}[Generalized slab-expert]
    Let $S_1 = (K_{1,+}, K_{1,-})$ and $S_2 = (K_{2,+},K_{2,-})$ be any two elements in the set of the half-space separations $\cS$. If $K_{1,-}\cap K_{2,+} = \emptyset$, then denote by $A_{S_1,S_2}$ the generalized slab-expert that runs exponential weights on set $K = K_{1,+}\cap K_{2,-} = \set{x_t}_{t\in [T]}\setminus (K_{1,-}\cup K_{2,+})$, predicts $-1$ on  $K_{1,-}$, and predicts $1$ on $K_{2,+}$. 
    Denote the set of all such generalized slab-experts by $\cA^+$. 
\end{definition}
It is clear that the slab experts are such experts as well. Thus, $\cA\subseteq \cA^+$.  Since the number of half-spaces is upper bounded by $|\cS|\leq (eT/d)^{2d}$, 
we have $|\cA^+| \leq |\cS|^2 \leq (eT/d)^{4d}$.
The generalized slab-experts are constructed through all possible halfspaces. This reduces the task of aggregating slab-experts to the simpler problem of enumerating/sampling halfspaces. 
Indeed, in the following, we show that it is possible to enumerate all the halfspaces with the help of the SVM
algorithm.
Moreover, it is possible to use rejection sampling where each halfspace will be chosen with probability at least $1/(2T)^{d+1}$ and with acceptance at least $1/2$. This, together with the Metropolis-Hastings (MH) algorithm for evaluating exponential weighted averaging, constitutes an implementable algorithm.

\paragraph{Enumerating the generalized slab-experts.} In order to enumerate the generalized slab-experts, we only need to enumerate all the possible halfspaces. Fortunately, the hard-margin SVM for halfspaces in $\bR^d$ with $n$ input labeled data points can output at most $d+1$ labeled data points that capture all the labeling information (specified in the proof of \cref{lem:sample-compression-for-enumeration}) with computational cost $O(\poly(d,n))$ \citep{boser1992training}.
Utilizing this fact, we can enumerate all the halfspaces for any given covariate set in polynomial time, as shown by the following lemma.

\newcommand{\covSet}{{K}}
\newcommand{\covSetSet}{{\cK}}
\begin{lemma}
    \label{lem:sample-compression-for-enumeration}
    For any set $\covSet$ of $T>0$ covariates, the halfspaces restricted to $\covSet$ defined as $\halfspace_\covSet = \{h_{\theta,b}|_\covSet : h_{\theta,b} = \sign(\inner{x,\theta}- b), \theta\in \bR^d, b\in \bR\}$ can be enumerated using at most $O((2T)^{d})$ calls to the $\svm$ algorithm.
    \looseness=-1
\end{lemma}

\newcommand{\Dhat}{\wh{\cD}}
\newcommand{\Dbar}{\wb{\cD}}
\renewcommand{\hbar}{\wb{h}}
\begin{proof}[\pfref{lem:sample-compression-for-enumeration}]
The $\svm$ algorithm for $\bR^d$ takes as input labeled data points $\cD = \set{(x_i,y_i)}_{i\in [n]}$ with $n>0$
and outputs $(\Dhat,\hhat) = \svm(\cD)$, where $\Dhat\subseteq \cD$ is a subset with at most $d+1$ labeled data points and $\hhat$ is a linear separator such that $\hhat(x_i) = y_i$ for all $i\in [n]$.  Moreover, the composition of the $\svm$ algorithm is stationary in the sense that if $(\Dbar,\hbar) = \svm(\Dhat)$, then $\Dbar = \Dhat$ and $\hbar=\hhat$ (For a proof of this fact, we refer to \citet{long2020complexity}, originally established by \citet{vapnik1974theory}).

We first enumerate all the subsets of $\covSet$ with at most $d+1$ elements, i.e., $\covSetSet_{d+1} = \set{\covSet'\subseteq \covSet:|\covSet'|\leq d+1}$. For each of the subset $\covSet'$ in $\covSetSet_{d+1}$, we consider the set of all possible labelings $\cD(\covSet') = \set{\set{\pm 1}^{\covSet'} :\covSet'\in \covSetSet_{d+1}}$. For any of the labeling $\cD = \set{(x,y_x)}_{x\in \covSet'}\in \cD(\covSet')$, we can apply the $\svm$ algorithm to obtain $ (\Dhat, \hhat) = \svm(\cD)$. With this process, we can define the following set of linear separators
\begin{align*}
\cH_{\covSet,\svm} = \set*{\hhat|_{\covSet}\;\middle|\;
\begin{aligned}
(\Dhat,\hhat) = \svm(\cD)&,~~\cD\in \cD(\covSet'),\\
\cD(\covSet') = \set{\set{\pm 1}^{\covSet'} :\covSet'\in \covSetSet_{d+1}}&,~~\covSetSet_{d+1} = \set{\covSet'\subseteq \covSet:|\covSet'|\leq d+1}
\end{aligned}
}.
\end{align*}
We claim that $\halfspace_\covSet = \cH_{\covSet,\svm}$. For any hypothesis $h|_\covSet \in \halfspace_\covSet$, let $ (\Dhat,\hhat) = \svm (\set{x,h(x)}_{x\in \covSet}) $. By the property of $\svm$, we have $|\Dhat|\leq d+1$, $h|_\covSet = \hhat|_\covSet$ and $ (\Dhat,\hhat) = \svm (\Dhat) $. 

Let $\Dhat= \set{(x_i,y_i)}_{i\in [k]}$, then since $\Dhat\in \cD(\set{x_i}_{i\in [k]})$ and $\set{x_i}_{i\in [k]} \subseteq \covSetSet_{d+1}$, thus $\hhat|_\covSet\in \cH_{\covSet,\svm}$ as desired. Moreover, to construct the hypothesis class $\cH_{\covSet,\svm}$, we need to enumerate all possible labelings of covariates subset of size no larger than $d+1$, which results in at most $O((2T)^{d+1})$ labelings and calls to the $\svm$ algorithm. Thus, the claim follows.
\end{proof}

With the enumeration of halfspaces using the $\svm$ algorithm, we state the following result.

\begin{proposition}
    \label{prop:computation-logsitic}
\cref{alg:main}, when using generalized slab-experts $\cA^+$ instead of the slab-experts $\cA$, can be implemented with a computational cost $O(\poly(T^d))$ and achieves a regret upper bound of $O(d\log T)$.
\end{proposition}

\begin{proof}[\pfref{prop:computation-logsitic}]
 Applying \cref{lem:sample-compression-for-enumeration}, we know that we can enumerate the halfspaces on given covariate set with computational cost $O(\poly(T^d))$. Once the halfspaces are enumerated, the generalized slab-experts can be constructed by the combinations of all halfspaces. Then \cref{alg:main} with generalized slab-experts $\cA^+$ aggregates each round with sampling cost at most $O(\poly(T,d))$ as shown in \cref{sec:compute-slab-expert-prediction} and computational cost at most $O(\poly(T^d))$ (since there are $\poly(T^d)$ generalized slab-experts). In all, \cref{alg:main}, when using generalized slab-experts $\cA^+$ instead of the slab-experts $\cA$, can be implemented with a computational cost that scales as at most $O(\poly(T^d))$. Moreover, recall that the number of generalized slab-experts is upper bounded by $(eT/d)^{4d}$. Thus, the uniform prior on the generalized slab-experts puts at least $\Omega(\poly(1/T^d))$ probability on
the optimal slab-expert $\SE_{\thstar,\alpha}$. Therefore, the regret can be upper bounded by $O(d\log T)$. The claim follows. \looseness=-1
\end{proof}

\paragraph{Sample the generalized slab-expert.} To further improve the implementability of \cref{alg:main} with generalized slab-experts $\cA^+$ instead of the slab-experts $\cA$, we can replace the enumeration of the generalized slab-experts with sampling. Concretely, we first sample a random subset of at most $d+1$ points of the covariate set $\covSet$. Next, we sample the labels randomly. Finally, we apply SVM to obtain a halfspace. 
Since any halfspace is sampled with probability at least $\Omega(\poly(1/T^d))$, the optimal slab-expert $\SE_{\thstar,\alpha}$ will also be sampled with probability at least $\Omega(\poly(1/T^d))$. Thus, the regret can be upper bounded by $O(d\log T)$. 
The advantage of such a sampling scheme is that, for a set of covariates in general position (no $d$ points are linearly dependent), random labeling of any $d+1$ subset will typically lead to a halfspace. This is validated by \citet[Theorem 1]{cover1965geometrical}. Specifically, for any $d+1$ points in general position, the number of distinct halfspaces is at least $2 \sum\limits_{k=0}^{d-1} \binom{d}{k} = 2^{d+1} - 2 \geq 2^d$ according to  \citet[Theorem 1]{cover1965geometrical}. However, there are at most $2^{d+1}$ different labeling available. Thus, random labeling allows for a halfspace separation with probability at least $1/2$.
Moreover, since the aggregation follows the exponentially weighted averaging
\begin{align*}
    \rhotil_t(x_t,\cdot) =  \sum\limits_{\SE\in \cA^+} p_t(\SE) \SE(x_t,\cdot), \text{~~where~~} p_t(\SE) \propto \exp\prn*{ \sum\limits_{i=1}^{t-1}  \log \SE(x_i,y_i) }.
\end{align*}
The distribution $p_t$ on the set of generalized slab-experts $\cA^+$ is a Gibbs-type distribution that can be represented as the stationary distribution of a Markov chain generated by the Metropolis-Hastings (MH) algorithm \cite[Section~7.3]{robert2004monte},
in a manner similar to that in \citep[Section 7]{rigollet2012sparse}. Note that this approach provides an implementable sampling procedure rather than offering any computational advantages.

\section{Online-to-Batch Conversion: Fixed Design vs Random Design}
\label{sec:online-to-batch}
\newcommand{\Rrnd}{R_{\mathrm{rnd}}}
\newcommand{\Rtrans}{R_{\mathrm{trans}}}
In the problem of statistical learning, the learner is given a sample $S_T  =\{(X_1,Y_1),\ldots,\\(X_T,Y_T)\}$ consisting of i.i.d. pairs sampled according to some unknown distribution $\cD$ on the space $\cX\times \cY$, where $\cX$ is the covariate space, and $\cY$ is the observation space. The learner is further given a function class $\cF:\cX\to\cY$ and a loss function $\ell(\cdot,\cdot):\cY\times\cY\to \bR$. The learner's goal is to output a (possibly randomized) predictor $\fhat_T$ to minimize the (expected) excess risk defined as %
\begin{align*}
     \E_{S_{T}}\left[\En_{(X,Y)\sim \cD} [\ell(\fhat_T(X),Y)] - \inf_{f\in \cF}\En_{(X,Y)\sim \cD} [\ell(f(X),Y)] \right] .
\end{align*}

We use transductive learning with the online-to-batch technique to obtain the following random design excess risk bound.

\begin{proposition}[Transductive online-to-batch conversion]
    \label{prop:online-to-batch}
    For any function class $\cF$, number of rounds $T$, and a loss function $\ell(\cdot, \cdot)$, convex with respect to the first argument, consider a deterministic sequence $x_1, \dots, x_T$. Suppose there exists an algorithm that produces a sequence of predictors $\widehat{f}_1, \dots, \widehat{f}_T$ in the sequential transductive setting, satisfying
    \begin{equation}
    \label{eq:transdregret}
    \sum_{t=1}^{T} \ell(\widehat{f}_t(x_t), y_t) - \inf_{f \in \cF} \sum_{t=1}^{T} \ell(f(x_t), y_t) \leq \operatorname{Reg}_T^{\operatorname{trd}},
    \end{equation}
    where $\operatorname{Reg}_T^{\operatorname{trd}}$ is an upper bound on the regret, independent of $x_t$ and $y_t$.
    
    Now, given a sample $S_T$ of size $T$ and an additional test point $(X, Y)$, both drawn i.i.d. from an unknown distribution $\mathcal{D}$, construct the sequence $\widehat{f}_{1},\ldots, \fhat_{T+1}$
    , where $\widehat{f}_{t+1}$ is based on the labeled sample $(X_1, Y_1), \dots, (X_t, Y_t)$ and the unlabeled set $\{X_1, \ldots, X_T, X\}$\footnote{Accordingly, we assume that at the level of the entire unlabeled dataset, the order of the observations does not influence the output of the algorithm, a condition consistently met in our applications. Whether this assumption is indeed required for transductive online-to-batch conversions in general remains open.} and satisfies the deterministic regret bound \eqref{eq:transdregret} with an upper bound on the regret $\operatorname{Reg}_{T+1}^{\operatorname{trd}}$. The final predictor, defined as $\widetilde{f} = \frac{1}{T+1} \sum_{t=1}^{T+1} \widehat{f}_t$, satisfies the following excess risk bound:
    \begin{align*}
    \E_{S_T}\left[\E_{(X,Y)} [\ell(\widetilde{f}(X), Y)] - \inf_{f \in \cF} \E_{(X,Y)} [\ell(f(X), Y)]\right] \leq \frac{\operatorname{Reg}_{T + 1}^{\operatorname{trd}}}{T+1}.
    \end{align*}
\end{proposition}

\begin{proof}[\pfref{prop:online-to-batch}]
Denote $(X_{T+1}, Y_{T+1}) = (X, Y)$. To highlight the dependence of any predictor $f_t$ on both the unlabeled set $\{X_s\}_{s=1}^{T+1}$ and the labeled sequence of observations $X_{1:t-1}, Y_{1:t-1}$, we use the following notation for the loss $\ell(\fhat_{t}(X_t \mid X_{1:t-1}, Y_{1:t-1}, \{X_s\}_{s=1}^{T+1}), Y_t)$ at the observation $X_t$, where $X_{1:t-1}$ denotes the sequence $X_1, \ldots\\, X_{t-1}$ and $Y_{1:t-1}$ stands for the sequence $Y_1, \ldots, Y_{t-1}$, respectively.
We first break the loss into summands of conditional expectations with Jensen's inequality: 
\begin{align*}
&(T+1)\E_{S_T}\left[\E_{(X,Y)} [\ell(\widetilde{f}(X), Y)] \right]\\
&\leq \En_{S_T}\En_{(X_{T+1},Y_{T+1})}\brk*{\sum\limits_{t=1}^{T+1}  [\ell(\fhat_t(X_{T+1}|X_{1:t-1},Y_{1:t-1}, \set{X_s}_{s=1}^{T+1}),Y_{T+1})]} \\
&=  \En_{X_1,\ldots,X_{T+1}} \brk*{\sum\limits_{t=1}^{T+1} \En_{(Y_s|X_s)_{s=1}^{t-1},Y_{T+1}|X_{T+1}}\brk*{ \ell(\fhat_t(X_{T+1}|X_{1:t-1},Y_{1:t-1}, \set{X_s}_{s=1}^{T+1}),Y_{T+1}) }}\\
&= \sum\limits_{t=1}^{T+1}\En_{X_1,\ldots,X_{T+1}} \brk*{ \En_{(Y_s|X_s)_{s=1}^{t-1},Y_{T+1}|X_{T+1}}\brk*{ \ell(\fhat_t(X_{T+1}|X_{1:t-1},Y_{1:t-1}, \set{X_s}_{s=1}^{T+1}),Y_{T+1}) }}
\end{align*}
Now we notice that for each summand, there is exchangeability between $(X_t,Y_t)$ and $(X_{T+1},Y_{T+1})$. Concretely, let $(X_t',Y_t') = (X_{T+1},Y_{T+1})$ and $(X_{T+1}',Y_{T+1}') = (X_t,Y_t)$, together with the fact that the transductive algorithm is indifferent to the ordering of $X_{1}, \ldots, X_{T+1}$, we have
\begin{align*}
&\En_{X_1,\ldots,X_{T+1}} \brk*{ \En_{(Y_s|X_s)_{s=1}^{t-1},Y_{T+1}|X_{T+1}}\brk*{ \ell(\fhat_t(X_{T+1}|X_{1:t-1},Y_{1:t-1}, \set{X_s}_{s=1}^{T+1}),Y_{T+1}) }}\\
&=\En_{X_1,\ldots,X_{t}',\ldots,X_T,X_{T+1}'} \brk*{ \En_{(Y_s|X_s)_{s=1}^{t-1},Y_{t}'|X_{t}'}\brk*{ \ell(\fhat_t(X_{t}'|X_{1:t-1},Y_{1:t-1}, \set{X_1,\ldots,X_{t}',\ldots,X_T,X_{T+1}'}),Y_{t}') }}\\
&\stackrel{(*)}{=} \En_{X_1,...,X_T,X_{T+1}} \brk*{ \En_{(Y_s|X_s)_{s=1}^{t}}\brk*{ \ell(\fhat_t(X_t| X_{1:t-1},Y_{1:t-1},\set{X_s}_{s=1}^{T+1}),Y_t) }}\\
&= \E_{S_T}\E_{(X_{T+1},Y_{T + 1})}[\ell(\widehat{f}_{t}(X_{t}), Y_{t})] .
\end{align*}
where for $(*)$, we use the fact that the distribution is not changed with the exchange. 
Then,
\begin{align*}
&(T + 1)\E_{S_T}\left[\E_{(X_{T+1},Y_{T+1})} [\ell(\widetilde{f}(X_{T +1}), Y_{T + 1})] - \inf_{f \in \cF} \E_{(X_{T+1},Y_{T+1})} [\ell(f(X_{T+1}), Y_{T + 1})]\right] \\
&\leq \sum\limits_{t = 1}^{T+1}\E_{S_T}\E_{(X_{T+1},Y_{T + 1})}[\ell(\widehat{f}_{t}(X_{t}), Y_{t})] - \inf_{f \in \cF} \E_{S_T}\E_{(X_{T+1},Y_{T + 1})} \sum\limits_{t = 1}^{T+1}\ell(f(X_t), Y_t) \\
&\le\E_{S_T}\E_{(X_{T+1},Y_{T + 1})}\left[\sum\limits_{t = 1}^{T+1}\ell(\widehat{f}_{t}(X_{t}), Y_{t}) - \inf_{f \in \cF}\sum\limits_{t = 1}^{T+1} \ell(f(X_t), Y_t)\right]
\\
&\le \operatorname{Reg}_{T+1}^{\operatorname{trd}}.
\end{align*}

The claim follows.
\end{proof}

\subsection{Applications of online-to-batch conversion}
In this section, we demonstrate two applications of the online-to-batch conversion result presented in \cref{prop:online-to-batch}.

\paragraph{Regression with squared loss.} Consider the regression problem with squared loss. Let the covariates $X_1, \dots, X_T$ and response variables $Y_1, \dots, Y_T$ be i.i.d. copies of the random pair $(X, Y)$ drawn from an unknown distribution. Denote the dataset as $S_T = \{(X_1, Y_1), \dots, (X_T, Y_T)\}$. We make no assumptions on the distribution of $X \in \mathbb{R}^d$, but assume that $|Y| \leq m$ almost surely. By combining the regret bound from \cref{thm:gaillard2019} (with $\lambda = 1/T$) and the online-to-batch conversion argument from \cref{prop:online-to-batch}, we immediately construct a computationally efficient batch predictor $\widetilde{f}$, based on $S_T$, that satisfies the following excess risk bound:
\[
\E_{S_T}\left[\E (Y - \widetilde{f}(X))^2 - \inf_{\theta \in \mathbb{R}^d} \E (Y - \langle X, \theta\rangle)^2\right] \leq \frac{m^2 (1 + d \log(T + 2))}{T + 1}.
\]
This bound, up to a logarithmic factor, matches the batch-specific bounds for the estimator proposed by \citet{forster2002relative} and the clipped linear least squares estimator analyzed by \citet{mourtada2022distribution}.

\paragraph{Excess risk bounds for logistic regression.} Now consider the case of logistic regression, where the covariates $X_1, \dots, X_T$ and the labels $Y_1, \dots, Y_T$ are i.i.d. copies of $(X, Y)$. Again, we make no assumptions about the distribution of $X \in \mathbb{R}^d$, and assume that $Y \in \{-1, 1\}$. Define the dataset $S_T = \{(X_1, Y_1), \dots, (X_T, Y_T)\}$. The following result holds.

\begin{corollary}
\label{cor:conditional-bernoulli}
In the logistic regression setting described above, without assumptions on the distribution of $(X, Y)$, there exists a probability assignment $\widetilde{p}$, based on $S_T$, that satisfies
\[
\E_{S_T}\left[-\log(\widetilde{p}(X, Y)) - \inf_{\theta \in \mathbb{R}^d} \E\left(-\log(\sigma(Y \langle X, \theta \rangle))\right)\right] \leq \frac{6d \log(e(T + 1))}{T + 1}.
\]
In particular, if we assume that $\mathbb{P}(Y = 1|X)$ follows a Bernoulli distribution with parameter $\sigma(\langle X, \theta^\star\rangle)$ (denoted as $\Ber(\sigma(\langle X, \theta^\star\rangle))$), then
\[
\E_{S_T} \E_X \mathcal{KL}\left(\Ber(\sigma(\langle X, \theta^\star \rangle)) \| \Ber(\widetilde{p}(X, 1))\right) \leq \frac{6d \log(e(T + 1))}{T + 1}.
\]
\end{corollary}

This result differs from previous excess risk bounds for logistic regression found in the literature (see, e.g., \citep{bach2010self, ostrovskii2021finite, foster2018logistic, mourtada2022improper, vijaykumar2021localization, bilodeau2023minimax, van2023high}) in that it imposes no assumptions on the marginal distribution of the covariates, and the regret bound does not depend on the norm of the minimizer $\theta^\star$.

\begin{proof}[\pfref{cor:conditional-bernoulli}]
The proof follows directly from combining \cref{thm:main-logistic} with \cref{prop:online-to-batch}. The second statement follows from the standard computation, showing that in the well-specified case, the logarithmic loss excess risk corresponds to the Kullback-Leibler divergence between the corresponding distributions.
\end{proof}

\subsection{Polynomial time algorithm for batch logistic regression}
\label{sec:complog}

In this section, we focus on minimizing the logistic loss in the batch setting for logistic distributions on the scaled sphere $\sqrt{d}\cS^{d-1}$ using polynomial-time algorithms. 
Specifically, we assume $(X, Y) \sim \mathcal{D}_{\theta^\star}$, where $X$ is sampled uniformly from the scaled sphere $\sqrt{d}\cS^{d-1}$. The label $Y \in \{\pm 1\}$ is assigned according to the probability $\sigma(Y \langle X, \theta^\star \rangle)$, conditioned on $X$, for some unknown parameter $\theta^\star$. 

We will present the proof for the distribution $\mathcal{D}_{\theta^\star}$ where the covariate $X$ follows a uniform distribution on the scaled sphere $\sqrt{d}\cS^{d-1}$. Nevertheless, our proof only requires mild assumptions (\cref{ass:bounded-eigen} and \cref{ass:thin-slab-proba}) on the marginal distribution of the covariates $X$.  
In effect, our proof extends to and well beyond 
the standard Gaussian distribution.

Our algorithm attains an excess risk upper bound of $O\left(\frac{(d + \log T) \log T}{T}\right)$ using the transductive online-to-batch method with two experts, both of which can be implemented in polynomial time. Besides being implementable in polynomial time, our excess risk upper bound, combined with the lower bound established by \citet{hsu2024sample} for proper predictors, reveals a separation between proper and improper approaches in terms of the dependence of the minimax excess risk on the norm of $\theta^\star$.

As a starting point, to address potential issues with unbounded loss, we introduce the probability assignment induced by the clipped linear separator.

\begin{definition}[Clipped linear separator]
    Let $T > 1$ be an integer.
    For any $\theta \in \mathbb{R}^d$, the clipped linear separator $k_{\theta}: \mathbb{R}^d \times \{\pm 1\} \to [0,1]$ is defined as 
    \begin{align*}
        k_{\theta}(x, y) = 
        \begin{cases}
            1- \frac{1}{2T}, & y \langle x, \theta \rangle \geq 0,\\
            \frac{1}{2T},& \textit{otherwise.}
        \end{cases}
    \end{align*}
    Namely, the expert assigns probability $1 - 1/(2T)$ to the label predicted by the linear separator and $1/(2T)$ to the alternative one.
\end{definition}

We state our main theorem of this section as follows.
\begin{algorithm}
\caption{Batch algorithm for logistic regression with pretraining}
\label{alg:batch-logistic}
\begin{algorithmic}[htb]
\Require The set of the data $\set{(X_t,Y_t)}_{t\in [T]}\cup \set{(X_t',Y_t')}_{t\in [T]}$ and the point of prediction $X$. Let $\alpha=1/(T+1)^{9}$ and $K=\set{X_t}_{t\in [T]}\cup \set{X}$.
\State Solve the hard margin $\svm$ for the pretraining dataset $\set{(X_t',Y_t')}_{t\in [T]}$.
\If{$\set{(X_t',Y_t')}_{t\in [T]}$ is separable}
\State Obtain $\thetahat$ that separates $\set{(X_t',Y_t')}_{t\in [T]}$.
\Else 
\State Let $\thetahat = e_1$ the unit vector in the first dimension.
\EndIf
\State Let $\cA = \set*{\ewa_{K,\alpha}, k_{\thetahat}}$ be the set of the two predictors. 
\For{$t=1,2,\ldots,T+1$}
\State \multiline{Let the prediction be the aggregation on the slab-experts 
\begin{equation}
\rhotil_t(\cdot,\cdot) =  \sum\limits_{\SE\in \cA} p_t(\SE) \SE(\cdot,\cdot) \text{~~where~~} p_t(\SE) \propto \exp\prn*{ \sum\limits_{i=1}^{t-1}   \log \SE(X_i,Y_i) }
\end{equation}
and $\SE(X_i,Y_i)$ is probability of $\SE$ predicting $Y_i$ on $X_i$.}
\EndFor 
\State Return the prediction $\widetilde{p}(X,\cdot) = \frac{1}{T+1} \sum_{t=1}^{T+1} \rhotil_t(X,\cdot) $.
\end{algorithmic}
\end{algorithm}

\begin{theorem}
    \label{thm:batch-logistic}
    Let $T > 1$ be an integer. Suppose $S_T = \{(X_t,Y_t)\}_{t \in [T]}$, $S_T' = \{(X_t',Y_t')\}_{t \in [T]}$, and $(X,Y)$ are i.i.d. drawn from $\mathcal{D}_{\theta^*}$. Then the $\widetilde{p}$ returned by \cref{alg:batch-logistic} satisfies 
    \begin{align*}
        \mathbb{E}_{S_T,S_T',(X,Y)}[ -\log \widetilde{p}(X,Y) ] \leq \mathbb{E}_{(X,Y)}[ -\log \sigma(Y \langle X, \theta^* \rangle) ] + \frac{C(d + \log T) \log T}{T},
    \end{align*}
    where $C > 0$ is an absolute constant.
    \end{theorem}
    
    The proof proceeds as follows: 
    We use the transductive online-to-batch method as described in \cref{sec:online-to-batch}. There are two experts, one based on the transductive prior, which incorporates all covariates, and the other obtained by solving the hard-margin SVM on a set of pretraining data. If the hard-margin SVM admits a solution $\widehat{\theta}$, we include the clipped linear separator with this solution $\widehat{\theta}$; otherwise, we use the unit vector along the first coordinate.
    
    There are two cases to analyze. The first case assumes that the norm of $\theta^*$ is bounded by $T^4$. In this scenario, by setting $\alpha < 1/T^9$, the expert using the transductive prior achieves a rate of $O(d \log T / T)$ (\cref{prop:exponential-weight-for-small-margin}).
    
    In the second case, where the norm of $\theta^*$ exceeds $T^4$, we further divide the analysis into three subcases:
    1. The first subcase is when the data is not separable by $\theta^*$, which occurs with small probability due to the large norm of $\theta^*$ (\cref{lem:separable-proba}).
    2. The second subcase is when $\widehat{\theta}$ deviates from $\theta^*$ by more than $(d + \log T) / T$, which also happens with small probability (\cref{lem:large-norm-error-region}).
    3. The final subcase occurs when $\widehat{\theta}$ is close to $\theta^*$, in which case the clipped linear separator with $\widehat{\theta}$ achieves a rate of $O((d + \log T) \log T / T)$ (\cref{lem:large-norm}).
    This completes the proof.
    
    We provide a proof for covariates distributed uniformly on the scaled sphere $\sqrt{d}\cS^{d-1}$, relying on two key properties of the covariate distribution, as described in \cref{ass:bounded-eigen} and \cref{ass:thin-slab-proba}. These properties also hold when the covariates follow a standard Gaussian distribution, thus extending our proof to that case as well.

\begin{assumption}[Bounded eigenvalue]
    \label{ass:bounded-eigen}
    A random vector $X$ in $\mathbb{R}^d$ has a bounded eigenvalue if the largest eigenvalue of the covariance matrix $\mathbb{E}[XX^\top]$ is upper bounded by $1$, i.e., $\lambda_{\max}\left( \mathbb{E}[XX^\top] \right) \leq 1$. Here, for any positive semi-definite matrix $A \in \mathbb{R}^{d \times d}$, we denote by $\lambda_{\max}(A)$ the largest eigenvalue of the matrix $A$.
\end{assumption}

The choice of the constant one in the upper bound $\lambda_{\max}\left( \mathbb{E}[XX^\top] \right) \leq 1$ can be relaxed to any constant and will appear as a log factor in the final risk bound. 
Both the uniform distribution on the sphere and the standard Gaussian distribution satisfy the bounded eigenvalue assumption (\cref{ass:bounded-eigen}). For the former, we have for any $\theta$ on the scaled sphere $\sqrt{d}\cS^{d-1}$:
\[
\theta^\top \mathbb{E}_{X\sim \Unif(\sqrt{d}\cS^{d-1})}[XX^\top] \theta  = d\mathbb{E}_{X\sim \Unif(\cS^{d-1})} [\langle X,e_1 \rangle^2 ] = d\cdot \En_{X\sim \Unif(\cS^{d-1}) } \brk*{\frac{1}{d} \sum\limits_{i=1}^{d} X_i^2 } =1.
\]
For the standard Gaussian distribution, the covariance matrix is the identity, thus it also satisfies the assumption.

\begin{assumption}[Small probability in thin slabs]
    \label{ass:thin-slab-proba}
    A random vector $X$ in $\mathbb{R}^d$ has small probability in thin slabs if for any $\varepsilon \in (0,1)$ and the $\thstar$ used to define the conditional distribution of $Y$, we have 
    \begin{align*}
        \mathbb{P}_{X} \left( \left|\langle X, \thstar/\norm{\thstar} \rangle \right| \leq \varepsilon \right) \leq \frac{\sqrt{d} \varepsilon}{\sqrt{\pi}}.
    \end{align*}
\end{assumption}

The choice of the constant $\sqrt{d}/\sqrt{\pi}$ can be relaxed to any polynomial dependence on $d$ and will appear as a log factor in the final risk bound. Both the uniform distribution on the scaled sphere $\sqrt{d}\cS^{d-1}$ and the standard Gaussian distribution satisfy \cref{ass:thin-slab-proba}. For the former, if $d = 1$, $\langle X, \theta \rangle = \pm 1$, thus 
\[
\mathbb{P}_{X} \left( \left|\langle X, \theta \rangle \right| \leq \varepsilon \right) = 0.
\]
Moreover, for $d \geq 2$,
\begin{align*}
    \mathbb{P}_X \left( \left|\langle X, \theta \rangle \right| \leq \varepsilon \right)
    = \int_{-\varepsilon/\sqrt{d}}^{\varepsilon/\sqrt{d}} \frac{\Gamma(d/2)}{\sqrt{\pi} \Gamma((d-1)/2)} (1 - x^2)^{\frac{d-3}{2}} \, dx \leq \frac{\sqrt{d} \varepsilon}{\sqrt{\pi}},
\end{align*}
where the first equality is due to the formula for the projection of a uniform distribution on the interval $[-1,1]$, and the inequality follows from the fact that the density is upper bounded by $\frac{d}{2\sqrt{\pi}}$. For covariates $X$ distributed according to the standard Gaussian, \cref{ass:thin-slab-proba} also holds up to a constant factor since the projection in any direction is standard normal with density upper bounded by $1/\sqrt{2\pi}$.

With \cref{ass:thin-slab-proba}, when $\|\theta^\star\|$ is large, then with high probability, the data generated will be separable as shown by the following \cref{lem:separable-proba}.

\begin{lemma}
\label{lem:separable-proba}
Let $T > 1$ be an integer. 
For any $\theta^\star \in \mathbb{R}^d$, with $\|\theta^\star\| > T^4$, with probability at least $1 - d/T$, the set $\{(X_t', Y_t')\}_{t=1}^T$ drawn i.i.d. from $\mathcal{D}_{\theta^\star}$ satisfies $Y_t' = \operatorname{sign}(\langle X_t', \theta^\star \rangle)$ for all $t$.
\end{lemma}

\begin{proof}[\pfref{lem:separable-proba}]
If $d \geq T^2$, we have nothing to prove. Thus, we assume $d < T^2$ for the following.

By the independence of the data points, we have 
\begin{align*}
    \mathbb{P} \left( \forall t, Y_t' \langle X_t', \theta^\star \rangle > 0 \right) = \left( \mathbb{P} \left( Y_1' \langle X_1', \theta^\star \rangle > 0 \right) \right)^T.
\end{align*}

Furthermore, by the conditional dependence of $Y_1'$ on $X_1'$, we have
\begin{align*}
    \mathbb{P} \left( Y_1' \langle X_1', \theta^\star \rangle > 0 \right) 
    &\geq \mathbb{P} \left( Y_1' \langle X_1', \theta^\star \rangle > 0, \left| \langle X_1', \theta^\star \rangle \right| > T \right) \\
    &\geq \mathbb{P} \left( \left| \langle X_1', \theta^\star \rangle \right| > T \right) \cdot \frac{1}{1 + e^{-T}}.
\end{align*}

Then, by \cref{ass:thin-slab-proba}, we have
\begin{align*}
    \mathbb{P} \left( \left| \langle X_1', \theta^\star \rangle \right| > T \right) 
    &= 1 - \mathbb{P} \left( \left| \langle X_1', \theta^\star \rangle \right| \leq T \right) \\
    &\geq 1 - \mathbb{P} \left( \left| \langle X_1', \theta^\star / \|\theta^\star\| \rangle \right| \leq 1/T^3 \right) \\
    &\geq 1 - \frac{\sqrt{d}}{\sqrt{\pi} T^3}.
\end{align*}

Altogether, we have shown that
\begin{align*}
    \mathbb{P} \left( \forall t, Y_t' \langle X_t', \theta^\star \rangle > 0 \right) 
    &\geq \left( \left( 1 - \frac{\sqrt{d}}{\sqrt{\pi} T^3} \right) \cdot \frac{1}{1 + e^{-T}} \right)^T \\
    &\geq 1 - \frac{d}{T}.
\end{align*}

This concludes our proof.
\end{proof}

\newcommand{\err}{\mathrm{err}}
\newcommand{\errtil}{\wt{\err}}

Whenever the data $\{(X_t', Y_t')\}_{t \in [T]}$ is separable, the hard margin SVM will, with high probability, output a classifier $\hat{\theta}$ where the disagreement region for the linear separators induced by $\hat{\theta}$ and $\theta^\star$ is only $O\left(\frac{d + \log T}{T}\right)$, according to the recent result of \citet{bousquet2020proper}. This is concretized by the following \cref{lem:large-norm-error-region}.

\begin{lemma}
\label{lem:large-norm-error-region}
Let $T>1$ be an integer. For any $\thstar\in \bR^d$, with $\norm{\thstar}>T^2$, suppose the $\set{X_t'}_{t=1}^T$ are drawn i.i.d. from the distribution $\cD_{\thstar}$ conditioning on the event $\set{Y' = \sign(\inner{X',\thstar})}$. 
Then, conditioning on the event $\set{\forall t,Y_t' = \sign(\inner{X_t',\thstar})}$, with probability at least $1-d/T$,
the $\svm$ solution $\thetahat$ on $\set{(X_t',Y_t')}_{t\in [T]}$ achieves
\begin{align*}
\thetahat \in \argmin_{\theta} \sum\limits_{t=1}^{T}\mathbbm{1}(Y_t'\inner{X_t',\theta}<0)
\end{align*}
satisfies
\begin{align*}
\bP_{X,Y\sim \cD_{\thstar}}\prn*{\inner{X,\thstar}\inner{X,\thetahat}<0} \leq \frac{C_1 (d + \log T)}{T},
\end{align*}
for some constant $C_1>0$.
\end{lemma}

\begin{proof}[\pfref{lem:large-norm-error-region}]
If $d>T$, then the argument automatically holds. Thus, we assume hereafter that $d\leq T$.
This proof is a revised version of \citet[Lemma 14]{hsu2024sample}.
For any $\theta$, the error of the linear separator defined by $\theta$ in $0/1$ loss is defined as 
\begin{align*}
\err_{\thstar}(\theta) &\ldef \En\brk*{ \mathbbm{1}(Y\inner{X,\theta} <0 ) } \\
& = \En\brk*{\sigma(- \sign(\inner{X,\theta})\inner{X,\thstar})    
}\\
&= \En \brk*{ \prn*{1-  \sigma(-\abs*{\inner{X,\theta}})} \mathbbm{1}\prn*{\inner{X,\thstar}\inner{X,\theta} <0}  } + \En \brk*{  \sigma(-\abs*{\inner{X,\theta}}) \mathbbm{1}\prn*{\inner{X,\thstar}\inner{X,\theta} \geq 0}  }\\
&= \bP\prn*{\inner{X,\thstar}\inner{X,\theta}<0}  - 2 \En \brk*{  \sigma(-\abs*{\inner{X,\theta}})\mathbbm{1}\prn*{\inner{X,\thstar}\inner{X,\theta} <0}  } + \En \brk*{  \sigma(-\abs*{\inner{X,\theta}})  }
\end{align*}
This implies
\begin{align}
\label{ineq:goal-error-region}
\begin{split}
    &\bP\prn*{\inner{X,\thstar}\inner{X,\theta}<0} \\
    &\leq  \err_{\thstar}(\theta) - \En \brk*{  \sigma(-\abs*{\inner{X,\theta}})  } +  2 \En \brk*{  \sigma(-\abs*{\inner{X,\theta}})\mathbbm{1}\prn*{\inner{X,\thstar}\inner{X,\theta} <0}  } \\
    &\leq \err_{\thstar}(\theta)  + \err_{\thstar}(\thstar). 
\end{split}
\end{align}
Now we note that by \cref{ass:thin-slab-proba},
\begin{align}
\label{ineq:control-error-region-1}
\begin{split}
     \err_{\thstar}(\thstar) &= \En \brk*{ \sigma(-\abs*{\inner{X,\thstar}}) } \\
    &\leq \bP\prn*{\abs*{\inner{X,\thstar}} \leq T} + \sigma(-T) \\
    &\leq \bP\prn*{\abs*{\inner{x,\thstar/\norm*{\thstar}}} \leq 1/T} + \frac{1}{1+e^{T}}\\
    &\leq \frac{d}{T}.   
\end{split}
\end{align}
Now, for any $\theta$, we let the error rate under conditioning be
\begin{align*}
\errtil_{\thstar}(\theta) \ldef  \En\brk*{ \mathbbm{1}(\inner{X,\theta} <0 ) \mid{} Y =\sign(\inner{X,\thstar}) }.
\end{align*}
By the risk bound for $\svm$ \citep[Theorem 12]{bousquet2020proper}, we have, with probability at least $1-d/T$, that
\begin{align}
\label{ineq:control-error-region-2}
    \errtil_{\thstar}(\thetahat) \leq  \frac{C'(d+\log T)}{T} 
\end{align}
for some constant $C'>0$.
Finally, we have 
\begin{align}
\label{ineq:control-error-region-3}
\begin{split}
    \err_{\thstar}(\thetahat) &= \En\brk*{ \mathbbm{1}(\inner{X,\theta} <0 ) } \\
    &\leq 2   \En\brk*{ \mathbbm{1}(\inner{X,\theta} <0 ) \mid{} Y =\sign(\inner{X,\thstar}) } \\
    &= 2\errtil_{\thstar}(\thetahat).  
\end{split}
\end{align}
Combining \Cref{ineq:goal-error-region,ineq:control-error-region-1,ineq:control-error-region-2,ineq:control-error-region-3} concludes the proof.
\end{proof}

Finally, whenever $\norm{\thstar}$ is large enough, and the $\thetahat$ is close enough to $\thstar$, the clipped linear predictor induced by $\thetahat$ admits a good performance as shown by the following \cref{lem:large-norm}.

\begin{lemma}
    \label{lem:large-norm}
    Let $T>1$ be an integer.
    For any $\thstar\in \bR^d$, with $\norm{\thstar}>T^2$,
    and $\thetahat\in S^{d-1}$ such that 
    $\bP\prn*{\inner{X,\thstar}\inner{X,\thetahat} \leq 0} \leq \veps$
    , then we have the logistic loss of the clipped linear separator $k_{\thetahat}$ under distribution $\cD_{\thstar}$ is upper bounded by
    \begin{align*}
    \En_{(X,Y)\sim \cD_{\thstar}} \brk*{-\log k_{\thetahat}(X,Y)} \leq   \frac{2d\log(2T)}{T} + \log(2T)\veps.
    \end{align*}
\end{lemma}

\begin{proof}[\pfref{lem:large-norm}]
We have 
\begin{align*}
    \revindent\En_{(X,Y)\sim \cD_{\thstar}} \brk*{-\log k_{\thetahat}(X,Y)} \\
    &\leq  \En_{(X,Y)\sim \cD_{\thstar}} \brk*{ -\log k_{\thetahat}(X,Y) \mathbbm{1}\prn*{\abs{\inner{X,\thstar}} \leq T}} \\
    &\quad + \En_{(X,Y)\sim \cD_{\thstar}} \brk*{ -\log k_{\thetahat}(X,Y) \mathbbm{1}\prn*{\inner{X,\thstar}\inner{X,\thetahat} \leq 0} }\\
    &\quad + \En_{(X,Y)\sim \cD_{\thstar}} \brk*{ -\log k_{\thetahat}(X,Y) \mathbbm{1}\prn*{\abs{\inner{X,\thstar}} > T,\inner{X,\thstar}\inner{X,\thetahat} > 0} }\\
    &\leq \log(2T) \prn*{\bP\prn*{\abs*{\inner{X,\thstar}} \leq T}  + \bP\prn*{\inner{X,\thstar}\inner{X,\thetahat} \leq 0}}\\
    &\quad + \En_{(X,Y)\sim \cD_{\thstar}} \brk*{ -\log k_{\thetahat}(X,Y) \mathbbm{1}\prn*{\abs{\inner{X,\thstar}} > T,\inner{X,\thstar}\inner{X,\thetahat} > 0} },
\end{align*}
where the second inequality is by $- \log k_{\thetahat}(X,Y) \leq \log (2T)$.
We first estimate the probability in the slab by \cref{ass:thin-slab-proba},
\begin{align*}
    \bP\prn*{\abs*{\inner{X,\thstar}} \leq T} &= \bP\prn*{\abs*{\inner{X,\thstar/\norm*{\thstar}}} \leq T/\norm{\thstar}}\\
    &\leq \bP\prn*{\abs*{\inner{X,\thstar/\norm*{\thstar}}} \leq 1/T}\\
    &\leq \frac{\sqrt{d}}{\sqrt{\pi} T},
\end{align*}
Finally, we estimate the risk in the correctly labeled region as
\begin{align*}
    &\En_{(X,Y)\sim \cD_{\thstar}} \brk*{ -\log k_{\thetahat}(X,Y) \mathbbm{1}\prn*{\abs{\inner{X,\thstar}} > T,\inner{X,\thstar}\inner{X,\thetahat} > 0} } \\
    &\leq -\frac{1}{1+e^{-T}}\log \prn*{1-  \frac{1}{2T}} + \frac{e^{-T}}{1+e^{-T}}\log (2T) \\
    &\leq  \frac{1}{T}. 
\end{align*}
Altogether, we conclude our proof.
\end{proof}

\begin{proof}[\pfref{thm:batch-logistic}]
Denote $(X_{T+1},X_{T+1}) = (X,Y)$. We consider two cases. For the first case, suppose $\norm*{\thstar} \leq T^4$. We have by the exponential weights
\begin{align*}
    -\sum\limits_{t=1}^{T+1}  \log \rhotil_t(X_t,Y_t)&= -\sum\limits_{t=1}^{T+1} \log \prn*{ \frac{\sum\limits_{\SE\in \cA} \exp\prn*{  \sum\limits_{i=1}^{t} \log \SE(X_i,Y_i) }}{\sum\limits_{\SE\in \cA} \exp\prn*{ \sum\limits_{i=1}^{t-1} \log \SE(X_i,Y_i)  }}} \\
    &= -\log \prn*{ \frac{1}{|\cA|} \sum\limits_{\SE\in \cA} \exp\prn*{ \sum\limits_{i=1}^{T} \log \SE(X_i,Y_i) }}\\
    &\leq -\log \prn*{ \exp\prn*{  \sum\limits_{i=1}^{T+1} \log \SE_{\thstar,\alpha}(X_i,Y_i) }}+ \log |\cA|\\
    &\leq \min\set*{ -\sum\limits_{i=1}^{T+1} \log \prn*{\ewa_{K,\alpha}(X_i,Y_i)},  -\sum\limits_{i=1}^{T+1} \log \prn*{k_{\thetahat}(X_i,Y_i)} }+ \log 2.
\end{align*}
Furthermore, we have by \cref{prop:exponential-weight-for-small-margin}
\begin{align*}
    \revindent-\sum\limits_{i=1}^{T+1} \log \prn*{\ewa_{K,\alpha}(X_i,Y_i)}  \\
    &\leq  - \sum\limits_{t=1}^{T+1} \log(\sigma(Y_t\langle X_t, \theta^{\star}\rangle))
    + \alpha(\thstar)^{\top}\left(\sum\limits_{t=1}^{T+1} X_tX_t^{\top}\right)\thstar +  \frac{d}{2}\log\left(1 + \frac{1}{8\alpha}\right) \\
    &\leq - \sum\limits_{t=1}^{T+1} \log(\sigma(Y_t\langle X_t, \theta^{\star}\rangle)) + \alpha \sum\limits_{t=1}^{T+1}\inner{X_t,\thstar}^2 + 3d\log(2T).
\end{align*}
Take expectation, we have, by the choice of $\alpha=1/(T+1)^{9}$ and the case assumption $\norm{\thstar}\leq T^4$, that
\begin{align*}
    \En_{S_T,S_T',(X,Y)}[ -\log \widetilde{p}(X,Y) ] &\le \frac{1}{T+1} \En\brk*{ - \sum\limits_{t=1}^{T+1}\log \rhotil_t(X_t,Y_t)}\\
    &\leq \En_{(X,Y)}[ -\log \sigma(Y\inner{X,\thstar}) ] + \alpha \lambda_{\max}\prn*{ \En [XX^\top] } \norm{\thstar}^2  + \frac{4d\log (2T)}{T+1}\\
    &\leq \En_{(X,Y)}[ -\log \sigma(Y\inner{X,\thstar}) ] + \frac{5d\log (2T)}{T},
\end{align*}
where the last inequality uses \cref{ass:bounded-eigen}.
For the second case, suppose $\norm{\thstar}>T^4$. For this, we have
\begin{align*}
    \En_{S_T,S_T',(X,Y)}[ -\log \widetilde{p}(X,Y) ] &= \frac{1}{T+1} \En\brk*{ - \sum\limits_{t=1}^{T+1}\log \rhotil_t(X_t,Y_t)}\\
    &\leq \En\brk*{- \log \prn*{k_{\thetahat}(X,Y)}} + \frac{\log 2}{T+1}.
\end{align*}
Then we divide into cases as
\begin{align*}
    &\En\brk*{- \log \prn*{k_{\thetahat}(X,Y)}} \\
    &\leq \En\left[- \log \prn*{k_{\thetahat}(X,Y)} \left(  \mathbbm{1}\prn*{\exists t,Y_t\neq \sign(\inner{X_t,\thstar})}
    \right.\right. \\
    & \quad\quad\quad\quad\quad+ \mathbbm{1}\prn*{ \forall t,Y_t = \sign(\inner{X_t,\thstar}),
     \bP_{X,Y\sim \cD_{\thstar}}\prn*{\inner{X,\thstar}\inner{X,\thetahat}<0} > \frac{C_1(d+\log T)}{T}} \\
    & \left.\left.  \quad\quad\quad\quad\quad+\mathbbm{1}\prn*{\bP_{X,Y\sim \cD_{\thstar}}\prn*{\inner{X,\thstar}\inner{X,\thetahat}<0} \leq \frac{C_1(d+\log T)}{T}} \right) \right] \\
    &\leq \log(2T) \bP\prn*{ \exists t,Y_t\neq \sign(\inner{X_t,\thstar})}\\
    &\quad  +\log(2T) \bP\prn*{ \forall t,Y_t = \sign(\inner{X_t,\thstar}),
     \bP_{X,Y\sim \cD_{\thstar}}\prn*{\inner{X,\thstar}\inner{X,\thetahat}<0} > \frac{C_1(d+\log T)}{T} }\\
     &\quad + \En\left[- \log \prn*{k_{\thetahat}(X,Y)}\mathbbm{1}\prn*{\bP_{X,Y\sim \cD_{\thstar}}\prn*{\inner{X,\thstar}\inner{X,\thetahat}<0} \leq \frac{C_1(d+\log T)}{T}} \right].
\end{align*}
Then, by \cref{lem:large-norm,lem:separable-proba,lem:large-norm-error-region}, we have
\begin{align*}
    \En\brk*{- \log \prn*{k_{\thetahat}(X,Y)}} \leq \frac{C_2(d+\log T)\log T}{T}.
\end{align*}
The claim follows.
\end{proof}

\section{Directions of Future Research}
\label{sec:final}
Finally, we outline several promising directions for future research.
\looseness=-1
\begin{itemize}
    \item Our online-to-batch conversion argument in \cref{sec:online-to-batch} only provides in-expectation results in important applications, such as logistic regression. While recent martingale-based arguments allow for high-probability excess risk bounds (see, e.g., \citep*{van2023high}), these existing arguments do not extend straightforwardly to the transductive setup, where the natural martingale structure is not present. An even more challenging question would be to provide computationally efficient methods for logistic regression with no assumptions on the design, for which the excess risk bound of \cref{cor:conditional-bernoulli} holds with high probability.
    \item For computational considerations, it is also worthwhile to explore algorithms such as Follow-the-Regularized-Leader (FTRL), where regularization is incorporated. To align with the context of our current work, we may examine transductive regularization where the regularizer depends on the set of covariates, as studied by \cite{gaillard2019uniform} in the case of squared loss. For logistic loss and hinge loss, however, it remains unclear whether FTRL with transductive regularization can achieve similar assumption-free excess risk bounds.
\item  Our results currently apply only to linear classes; a natural next step is to extend them to more general classes. At present, however, we lack the technical tools to construct natural, nontrivial priors for such classes, beyond, for example, the uniform distribution on $\varepsilon$-covers.
\item While some of our techniques, such as exponential weights and mixability, extend to broader loss families, our analysis also exploits specific structural properties of the considered models, including the form of the sigmoid link and the resulting Hessian in the parameter space. For instance, our logistic regression results rely on the slab construction and need not hold for the same logarithmic loss in more general GLM settings. As in the case of general classes, it would be interesting to study other losses and different link functions in the context of our setup.

\end{itemize}

\acks{JQ and AR acknowledge support from ARO through award W911NF-21-1-0328, as well as the Simons Foundation and the NSF through awards DMS-2031883 and PHY-2019786. The authors thank Jaouad Mourtada for identifying an issue with the polynomial-time implementation of exponential weights with heavy-tailed priors in the initial version of the paper and for contributing to the discussion on the related sampling strategy described in \citep{dalalyan2012sparse}. The authors also thank Jingfeng Wu for pointing out the reference \citep{hsu2024sample}. Finally, the authors are thankful to the referees whose precise comments improved the presentation of the paper.}

\newpage

\appendix

\bibliography{ref}\textbf{}

\section{Additional Results for Sparse Regression}
\label{app:additional-sparse}

It is worth making the connection with sparsity priors used in \citep{rigollet2011exponential, rigollet2012sparse}. These authors discuss that certain \emph{discrete priors} in the batch fixed design linear regression problem with Gaussian noise allow bypassing any dependence on compatibility constants such as the constant $\kappa_s$ used above. \citet{rigollet2011exponential, rigollet2012sparse} also make comparisons with classical Lasso and BIC-type oracle inequalities. While Lasso has computational advantages, its bounds involve compatibility constants such as the restricted eigenvalue assumption. In contrast, discrete sparsity priors allow bypassing any compatibility constants akin to what is observed for $\|\cdot\|_{0}$ penalization.

Before discussing these discrete priors, we note that despite the minor penalty of the logarithm of the inverse smallest scaled singular value constant $\kappa_s$, there is a significant advantage: the predictor in \cref{thm:sparsity} has an explicit Gaussian-type integral form. On the other hand, the predictor derived from sparsity priors can be quite challenging to exploit, as it might require too many samples from this prior (or equivalently, sampling from the posterior distribution $\rho_t$ is hard). However, from a practical implementation standpoint, it can still be effective, as claimed by simulations in \citep{rigollet2011exponential, rigollet2012sparse}.

We use the sparsity priors as in \citep{rigollet2011exponential}. Consider the set \(\mathcal{P} = \{1, \ldots, d\}\). For each \(p \subseteq \mathcal{P}\) assign the probability mass 
\begin{equation}
\label{eq:discreteprior}
\mu(p) = \left({d \choose |p|}\exp(|p|)H_d\right)^{-1},
\end{equation}
where \(H_d = \sum\nolimits_{i = 1}^{d} \exp(-i)\). Although this prior is not data-dependent, it can be naturally combined with the data-dependent prior \eqref{eq:jeffreys} by introducing a product prior that first samples the sparsity pattern vector according to \eqref{eq:discreteprior}, then uses \eqref{eq:jeffreys} restricted to the coordinates corresponding to this sparsity pattern. Finally, the exponential weights algorithm is run with respect to this data-dependent prior. To slightly simplify the analysis, we will use a version of \cref{thm:gaillard2019}, where the additional clipping step is performed, and then we use aggregation using the exp-concavity of the squared loss in the bounded case.

For any \(p \subseteq \mathcal{P}\) let \(\Theta_{p} = \{\theta \in \mathbb{R}^d: \theta^{(j)} = 0\ \textrm{if}\ j \notin p \}\). That is, \(\Theta_{p}\) contains all vectors in \(\mathbb{R}^d\), which are equal to zero for all coordinates outside of the set \(p\). For each \(p \in \mathcal{P}\), define the sequence of weights 
\[
\widetilde{\theta}_{t, p} = \arg\min_{\theta \in \Theta_p}\left\{\sum_{i = 1}^{t - 1} (y_i - \langle x_i, \theta \rangle)^2 + \lambda \sum\limits_{t = 1}^T (\langle x_t, \theta \rangle)^2 \right\},
\]
where in the case of a non-unique minimizer, we choose the one with the smallest Euclidean norm, which corresponds to the pseudo-inverse matrix in the least squares solution formula. Define the sequence $\rho_t$ of distributions over subsets of $\mathcal P$, by $\rho_1 = \mu$ and for any $t \ge 2$, set
\[
\rho_t(p) = \frac{\exp\left(-\frac{1}{8m^2}\sum\limits_{i = 1}^{t - 1}\left(y_t - \clip{\langle\widetilde{\theta}_{t, p}, x_t\rangle}\right)^2\right)\mu(p)}{\sum\limits_{p^{\prime} \in \mathcal P}\exp\left(-\frac{1}{8m^2}\sum\limits_{i = 1}^{t - 1}\left(y_t - \clip{\langle\widetilde{\theta}_{t, p^{\prime}}, x_t\rangle}\right)^2\right)\mu(p^{\prime})}.
\]
Finally, we introduce the sequence of predictors
\begin{equation}
\label{eq:discretepriors}    
\widetilde{f}_{t}(x) = \sum\limits_{p \in \mathcal P}\rho_{t}(p)\clip{\langle x_t, \widetilde{\theta}_{t, p}\rangle}.
\end{equation}
The work of \citet{rigollet2012sparse} suggests a Metropolis–Hastings type algorithm for such sampling, though in the worst case its convergence rate can be quite slow.

The following result provides the desired regret bound for this predictor. Given the established regret bounds in the non-sparse case, the proof follows a standard approach and is included for reference. 
\begin{proposition}
    \label{prop:expscreening}
    Assume that the sequence $\{(x_t, y_t)\}_{t = 1}^T$, where $(x_t, y_t) \in \mathbb{R}^d \times \mathbb{R}$ for all $1 \leq t \leq T$, satisfies $\max_{t} |y_t| \le m$. Then, for any $s$-sparse $\theta^{\star} \in \mathbb{R}^d$, the following regret bound holds for the predictor \eqref{eq:discretepriors} with $\lambda = \frac{s}{T}$,
    \[
        \sum_{t = 1}^T (y_t - \widetilde{f}_t(x_t))^2 \le \sum_{t = 1}^T (y_t - \langle x_t, \theta^{\star} \rangle)^2 + s m^2\left(\log\left(1 + \frac{T}{s}\right) + 16\log\left(\frac{ed}{s}\right) + 5\right).
    \]
\end{proposition}

\begin{proof}
    Fix any $p \in \mathcal{P}$ such that $|p| = s$. When restricted to the coordinates corresponding to the set $p$, the bound of \cref{cor:fixedclipped} in \cref{sec:linearrevisit} gives the regret bound
    \[
        \sum_{t = 1}^T \left(y_t - \clip{\langle x_t, \widetilde{\theta}_{t, p} \rangle}\right)^2 \le \inf_{\theta \in \Theta_p} \left\{ \sum_{t = 1}^T (y_t - \langle x_t, \theta \rangle)^2 \right\} + \lambda T m^2 + 4s m^2 \log\left(1 + \frac{1}{\lambda}\right).
    \]
    It remains to notice that for any $p \in \mathcal{P}$ such that $|p| = s$, the prediction $\clip{\langle x_t, \widetilde{\theta}_{t, p} \rangle}$ is absolutely bounded by $m$. Finally, combining the exp-concavity of the squared loss in the bounded case \citep[Remark 3]{vovk2001competitive} (thus, we choose $\eta = \frac{1}{8m^2}$ in the exponential weights algorithm) and the fact that the predictor \eqref{eq:discretepriors} is a posterior average with respect to the exponential weights algorithm with prior $\mu$ over $\mathcal{P}$, and by \cref{lem:donskervaradhan} with $\log\left(\frac{1}{\mu(p)}\right) = 2s \log\left(\frac{ed}{s}\right) + \frac{1}{2}$ by \cite[Inequality (5.4)]{rigollet2012sparse}, we obtain that for any $s$-sparse $\theta^{\star} \in \mathbb{R}^d$,
    \[
        \sum_{t = 1}^T (y_t - \widetilde{f}_t(x_t))^2 \le \sum_{t = 1}^T (y_t - \langle x_t, \theta^{\star} \rangle)^2 + s m^2 + 4s m^2 \log\left(1 + \frac{T}{s}\right) + 8m^2 \left(2s \log\left(\frac{ed}{s}\right) + \frac{1}{2}\right).
    \]
    The claim follows.
\end{proof}

\section{Mixability of the Squared Loss for Unbounded Sets of Predictors}
\label{app:mixability-squared-loss}
For completeness, we present the following mixability result, which is implicitly used in the analysis of sequential linear regression by \citet{vovk2001competitive}. Our proof closely follows the arguments in \citep{vovk2001competitive}, with the key clarification that the boundedness of the class of predictors is not required, a point that may not be clear from the original paper.

\begin{lemma}[Restatement of \cref{lem:mixabilityofthesqloss}]
    Consider a class $\mathcal F$ of functions $f_{\theta}: \mathcal{X} \to \mathbb{R}$ (possibly unbounded) parameterized by $\Theta \subseteq \mathbb{R}^d$. Assume that $y$ is such that $|y| \le m$ and choose $\eta = \frac{1}{2m^2}$. Given any distribution $\rho$ over $\Theta$, define the predictor
\begin{equation*}
\widehat{f}(x) = \frac{m}{2}\log\left(\frac{\mathbb{E}_{\theta \sim \rho} \exp(-\eta (m - f_{\theta}(x))^2)}{\mathbb{E}_{\theta \sim \rho} \exp(-\eta (-m - f_{\theta}(x))^2)}\right).
\end{equation*}
Then, the following holds:
\[
(y - \widehat{f}(x))^2 \le -\frac{1}{\eta} \log\left(\mathbb{E}_{\theta \sim \rho} \exp(-\eta (y - f_{\theta}(x))^2)\right).
\]
\end{lemma}

\begin{proof}
Without loss of generality, assume $m = 1$. The classical mixability result from \citep[Lemmas 2 and 3]{vovk2001competitive} shows that for the set $\mathcal{F}$ of $1$-absolutely bounded prediction functions, there exists a predictor $\widetilde{f}$ such that for all $y \in [-1, 1]$:
\[
(y - \widetilde{f}(x))^2 \le -2\log\left(\E_{\theta \sim \rho}\exp\left(-(y - \operatorname{clip}_1(f_{\theta}(x)))^2/2\right)\right).
\]
By the definition of the clipping function, this implies that for the same predictor and all $y \in [-1, 1]$,
\begin{equation}
\label{eq:mixability}
(y - \widetilde{f}(x))^2 \le -2\log\left(\E_{\theta \sim \rho}\exp\left(-(y - f_{\theta}(x))^2/2\right)\right).
\end{equation}

Now, define a predictor $\widehat{f}$ as 
\[
\widehat{f}(x) = \mathop{\operatorname{arg\ inf}}_{\widehat{y}} \sup\limits_{y \in [-1, 1]}\left((y - \widehat{y})^2 + 2\log\left(\E_{\theta \sim \rho}\exp\left(-(y - f_{\theta}(x))^2/2\right)\right)\right),
\]
which satisfies the same inequality \eqref{eq:mixability} as $\widetilde{f}$ by construction. The proof of Lemma 3 in \citep{vovk2001competitive} shows that the function $h$, defined as
\[
h(y) = (y - \widehat{f}(x))^2 + 2\log\left(\E_{\theta \sim \rho}\exp\left(-(y - f_{\theta}(x))^2/2\right)\right),
\]
satisfies $\frac{\partial^2 h(y)}{\partial y^2} \geq 0$, making $h$ convex. Importantly, the derivation of $\frac{\partial^2 h(y)}{\partial y^2}$ does not impose any boundedness condition on $f_{\theta}(x)$. Since $h$ is convex, its maximum is attained at either $-1$ or $1$. Therefore, by symmetry, $\widehat{f}$ satisfies
\begin{align*}
&(1 - \widehat{f}(x))^2 + 2\log\left(\E_{\theta \sim \rho}\exp\left(-(1 - f_{\theta}(x))^2/2\right)\right) \\
&\qquad= (-1 - \widehat{f}(x))^2 + 2\log\left(\E_{\theta \sim \rho}\exp\left(-(-1 - f_{\theta}(x))^2/2\right)\right).
\end{align*}
Solving this equation for $\widehat{f}(x)$ completes the proof.
\end{proof}

\section{Revisiting Existing Bounds for Sequential Linear Regression}
\label{sec:linearrevisit}
This section contains several results related to linear regression that might be of independent interest. The proof strategy goes beyond the transductive setting and follows the derivations of \citet{vovk2001competitive} with the following twist: instead of exploiting the \emph{mixability} of the squared loss, we consider a similar problem of conditional density estimation for the class of conditional densities
\begin{equation}
\label{eq:conddensity}
\set*{ p(y|x,\theta) = \exp\prn*{- \pi(y-x^\top\theta)^2} : \ \theta \in \mathbb{R}^d}.
\end{equation}
Our loss function is the logarithmic function, and we simply compute the total error for the \emph{mix-loss} as suggested by \cref{lem:sumofmixlosses}. This straightforward computation is sufficient to recover all existing sequential online bounds, including the one in \citep{vovk2001competitive}, the bound for clipped ridge regression in \citep{kivinen1999averaging}, and the bound of \citet{gaillard2019uniform}, which has already been discussed in \cref{sec:linearregression}. In the following, we discuss how essentially the same strategy, specific to sequential ridge regression, was previously employed by \citet{zhdanov2009competing}. Here, we mainly provide a detailed exposition of this approach and extend it to our transductive setup.

Let $A \in \mathbb{R}^{d \times d}$ be any positive definite matrix and given the sequence $x_{1}, \ldots, x_T \in \mathbb{R}^d$ define the \emph{sequential leverage scores} of $x_i$ as follows:
\[
h_{x_i} = x_i^{\top}\left(\sum\limits_{t = 1}^ix_tx_t^{\top} + A\right)^{-1}x_i.
\]
We remark that the results hold if the matrix $A$ is positive semi-definite and thus not necessarily invertible. In this case, we should replace the inverse with the pseudo-inverse. The central result of this section is the following identity.
\begin{proposition}[Generalization of \citep{azoury2001relative, zhdanov2009competing}]
\label{prop:identityforlinearloss}
Let \(A \in \mathbb{R}^{d \times d}\) be a positive definite matrix. Given any sequence \(\{(x_t, y_t)\}_{t = 1}^T \in (\mathbb{R}^d \times \mathbb{R})^T\), define the ridge regression-type predictor
\[
\widehat{\theta}_t = \argmin_{\theta \in \mathbb{R}^d} \sum_{i = 1}^{t - 1} (y_i - \langle x_i, \theta \rangle)^2 + \theta^\top A \theta.
\]
The following identity holds:
\[
\sum_{t = 1}^T (1 - h_{x_t})(y_t - \langle x_t, \widehat{\theta}_t \rangle)^2 = \inf_{\theta \in \mathbb{R}^d} \left\{ \sum_{t = 1}^T (y_t - \langle x_t, \theta \rangle)^2 + \theta^\top A \theta \right\}.
\]
\end{proposition}

Although the bound of \cref{prop:identityforlinearloss} can be derived directly using linear algebraic arguments (which is essentially done in the proof of Theorem 4.6 in \citep{azoury2001relative}), we believe that the following proof strategy is more insightful. It provides this identity without prior knowledge of its possible existence. We note that the proof strategy and explicit identity in \cref{prop:identityforlinearloss} closely align with previous work, specifically with \citet[Theorem 1]{zhdanov2009competing}, which establishes this identity for \( A = I_d \) and acknowledges that a similar identity appears in \citet{azoury2001relative}. Additionally, the work by \citet{zhdanov2013identity} explores sequential formulations of kernel ridge regression with the focus on similar identities.
Although this formulation is not novel, it is included here for completeness and to support the transductive prior framework used in this paper. The extension to a general positive definite matrix \( A \) as the regularizer corresponds directly to the transductive priors considered, simplifying derivations of bounds such as that in \citet*{gaillard2019uniform}, as well as additional technical results used throughout the paper.

We will first introduce some auxiliary lemmas. The first result is standard.
\begin{lemma}
\label{lem:detlemma}
     Let \( B \in \mathbb{R}^{n \times n} \) be a full-rank matrix, let \( x \) be a vector in $\mathbb{R}^n$, and let \( A = B + xx^\top \). Assuming that both $A$ and $B$ are invertible, we have
\[
\frac{\det(B)}{\det(A)}= 1 - x^\top A^{-1} x, \quad \textrm{and} \quad \frac{\det(A)}{\det(B)} = 1 + x^{\top}B^{-1}x.
\]
\end{lemma}
We also need the following derivation.
\begin{lemma}
\label{lem:gauss}
Assume that $Z \sim \mathcal{N}(\nu, \sigma^2)$. We have for any $y \in \mathbb{R}$,
\[
-\log\left(\mathbb{E}\exp(-(y - Z)^2))\right) = \frac{(\nu-y)^2}{2\sigma^2 + 1} + \frac{1}{2}\log( 2\sigma^2 + 1).
\]
\end{lemma}
\begin{proof}
Using \cref{lem:laplace}, we have
     \begin{align*}-\log\left(\mathbb{E}\exp(-(y - Z)^2))\right) &= -\log\left(\frac{1}{\sqrt{2\pi}\sigma}\int\exp(-(y - z)^2 - (z- \nu)^2/2\sigma^2)dz\right)
    \\
    &=-\log\left(\frac{\sqrt{\pi}}{\sqrt{2\pi}\sigma\sqrt{1 + 1/(2\sigma^2)}}\exp\left(-\frac{(\nu^2 - 2\nu y + y^2)}{2\sigma^2 + 1}\right)\right)
     \\
    &=-\log\left(\frac{1}{\sqrt{2\sigma^2 + 1}}\exp\left(-\frac{(\nu-y)^2}{2\sigma^2 + 1}\right)\right)
    \\
    &= \frac{(\nu-y)^2}{2\sigma^2 + 1} + \frac{1}{2}\log( 2\sigma^2 + 1).
    \end{align*}
\end{proof}
We are ready to provide a proof of \cref{prop:identityforlinearloss}.
\begin{proof}
Our prior distribution $\mu$ is multivariate Gaussian in $\mathbb{R}^d$ with its density given by
\[
\mu(\theta) = \sqrt{\det(A)}\cdot\exp\left(-\pi\theta^\top A\theta\right).
\]
We want to compute the right-hand side of \eqref{eq:sumofmixlosses} for the class \eqref{eq:conddensity}.
The first computation uses \cref{lem:laplace}. We have
\begin{align}
\label{eq:freeenergy}
    &-\log\left(\E_{\theta \sim \mu}\exp\left(-\sum\limits_{t = 1}^T\pi(y_t-\langle x_t, \theta\rangle)^2\right)\right) 
    \\
    &= -\log\left(\sqrt{\det(A)}\int\limits_{\theta \in \mathbb{R}^d}\exp\left(-\sum\limits_{t = 1}^T\pi(y_t-\langle x_t, \theta\rangle)^2 - \pi\theta^{\top}A\theta\right)\right) \nonumber
    \\
    &= \inf\limits_{\theta \in \mathbb{R}^d}\left\{\sum\limits_{t = 1}^T\pi(y_t - \langle x_t, \theta\rangle)^2 + \pi\theta^{\top}A\theta\right\} - \frac{1}{2}\log(\det(A)) + \frac{1}{2}\log\left(\det\left(\sum\limits_{t = 1}^Tx_tx_t^{\top} + A\right)\right). \label{eq:finalfreeenergy}
\end{align}
Another perspective on \eqref{eq:freeenergy} is the following. We assume that at step $t$ we predict with the distribution $\rho_t$ over $\mathbb{R}^d$ using the exponential weights over the class \eqref{eq:conddensity} with respect to the logarithmic loss. By \cref{lem:laplace}, we have
\begin{equation}
-\log\left(\E_{\theta \sim \mu}\exp\left(-\sum\limits_{t = 1}^T\pi(y_t-\langle x_t, \theta\rangle)^2\right)\right) = -\sum\limits_{t = 1}^T\log\left(\E_{\theta \sim \rho_t}\exp\left(-\pi(y_t-\langle x_t, \theta\rangle)^2\right)\right).
\end{equation}
It is only left to compute individual terms (mix-losses) in the right-hand side of the above identity. By the definition of the exponential weights algorithm, we have 
\[
\rho_t \propto	\exp\left(-\sum\limits_{i = 1}^{t-1}\pi(y_i-\langle x_i, \theta\rangle)^2 -\pi \theta^{\top}A\theta\right).
\]
Clearly, the distribution $\rho_t$ is multivariate Gaussian. Using \cref{lem:laplace}, we can immediately obtain 
\[
\rho_t \sim \mathcal{N}\left(\widehat{\theta}_{t}, \frac{1}{2\pi}\left(\sum\limits_{i = 1}^{t - 1}x_ix_i^{\top} + A\right)^{-1}\right).
\]
Observe that if $\theta \sim \rho_t$, then $\langle x_t, \theta\rangle \sim \mathcal{N}\left(\langle \widehat{\theta}_t, x_t\rangle, \frac{1}{2\pi}x_t^{\top}\left(\sum\limits_{i = 1}^{t - 1}x_ix_i^{\top} + A\right)^{-1}x_t\right)$, which implies by \cref{lem:detlemma} that the following holds
\[
\langle x_t, \theta\rangle \sim \mathcal{N}\left(\langle \widehat{\theta}_i, x_t\rangle, \frac{1}{2\pi}\frac{h_{x_t}}{1 - h_{x_t}}\right).
\]
Finally, we have by \cref{lem:gauss}
\begin{align*}
-\log\left(\E_{\theta \sim \rho_t}\exp\left(-\pi(y_t-\langle x_t, \theta\rangle)^2\right)\right) &=-\log\left(\E_{\theta \sim \rho_t}\exp\left(-(\sqrt{\pi} y_t-\sqrt{\pi}\langle x_t, \theta\rangle)^2\right)\right)
\\
&= \frac{(\sqrt{\pi}\langle \widehat{\theta}_i, x_t\rangle-\sqrt{\pi} y_t)^2}{\frac{h_{x_t}}{1 - h_{x_t}} + 1} + \frac{1}{2}\log\left(\frac{h_{x_t}}{1 - h_{x_t}} + 1\right)
\\
&= \pi(1 - h_{x_t})(\langle \widehat{\theta}_i, x_t\rangle-y_t)^2 + \frac{1}{2}\log\left(\frac{1}{1 - h_{x_t}}\right).
\end{align*}
Summing the above identity with respect $t$ and using \eqref{eq:sumofmixlosses} we obtain 
\begin{align}
\label{eq:sumofmix}
    &-\log\left(\E_{\theta \sim \mu}\exp\left(-\sum\limits_{t = 1}^T\pi(y_t-\langle x_t, \theta\rangle)^2\right)\right) \nonumber
    \\
    &= \sum\limits_{t = 1}^T\left(\pi(1 - h_{x_t})(\langle \widehat{\theta}_i, x_t\rangle-y_t)^2 + \frac{1}{2}\log\left(\frac{1}{1 - h_{x_t}}\right)\right).
\end{align}
Denote $A_t = x_t^{\top}\left(\sum\limits_{i = 1}^{t - 1}x_ix_i^{\top} + A\right)^{-1}x_t$.
Finally, by \cref{lem:gauss} we have
\begin{align}
\label{eq:logdetsum}
    &\sum\limits_{t = 1}^T\frac{1}{2}\log\left(\frac{1}{1 - h_{x_t}}\right) = \frac{1}{2}\sum\limits_{t = 1}^T\log\left(\frac{\det(A_{t + 1})}{\det(A_t)}\right) \nonumber
    \\
    &= -\frac{1}{2}\log(\det(A)) + \frac{1}{2}\log\left(\det\left(\sum\limits_{t = 1}^Tx_tx_t^{\top} + A\right)\right).
\end{align}
where in the last computation, we used the telescopic sum. Combining \eqref{eq:finalfreeenergy}, \eqref{eq:sumofmix}, and \eqref{eq:logdetsum}, we prove the claim.
\end{proof}

The result of \cref{prop:identityforlinearloss} immediately implies the following classical result.
\begin{corollary}[The regret of the non-linear predictor of \citet{vovk2001competitive}]
\label{cor:vaw}
Fix $\lambda > 0$ and for $x \in \mathbb{R}^d$ consider a sequence of estimators
\[
\widehat{\theta}_{t, x} = \argmin_{\theta \in \mathbb{R}^d} \sum_{i = 1}^{t - 1} (y_i - \langle x_i, \theta \rangle)^2 + \lambda\|\theta\|^2 + (\langle x, \theta\rangle)^2.
\]
Assume that $|y_t| \le m$ for $t = 1, \ldots, T$. Then, the following regret bound holds
\[
\sum\limits_{t = 1}^T(y_t - \langle x_t, \widehat{\theta}_{t, x_t}\rangle)^2 \le \inf \limits_{\theta \in \mathbb{R}^d}\left\{\sum_{t = 1}^{T} (y_t - \langle x_t, \theta \rangle)^2 + \lambda \|\theta\|^2\right\} + m^2\log\left(\det\left(I_d + \frac{1}{\lambda}\sum\limits_{t = 1}^Tx_tx_t^{\top}\right)\right).
\]
\end{corollary}
\begin{proof}
To see how the implication, one should first verify, using \cref{lem:detlemma}, that 
\[
\langle x_t, \widehat{\theta}_{t, x_t}\rangle = (1 - h_{x_t})\langle x_t, \widehat{\theta}_t\rangle,
\]
where
\begin{equation}
\label{eq:ridge}
\widehat{\theta}_{t} = \argmin_{\theta \in \mathbb{R}^d} \sum_{i = 1}^{t - 1} (y_i - \langle x_i, \theta \rangle)^2 + \lambda\|\theta\|^2, \quad\textrm{and} \quad h_{x_i} = x_i^{\top}\left(\sum\limits_{t = 1}^ix_tx_t^{\top} + \lambda I_d\right)^{-1}x_i.
\end{equation}

Combining this with \cref{prop:identityforlinearloss}, we have
\begin{align*}
\sum\limits_{t = 1}^T(y_t - \langle x_t, \widehat{\theta}_{t, x_t}\rangle)^2 &=\sum\limits_{t = 1}^T(1 - h_{x_t})(y_t - \langle x_t, \widehat{\theta}_t\rangle)^2 + \sum\limits_{t = 1}^T(h_{x_t}^2 - h_{x_t})(\langle x_t, \widehat{\theta}_t\rangle)^2 + \sum\limits_{t = 1}^Th_{x_t}y_t^2
\\
&=\inf \limits_{\theta \in \mathbb{R}^d}\left(\sum_{t = 1}^{T} (y_t - \langle x_t, \theta \rangle)^2 + \lambda \|\theta\|^2\right) +\sum\limits_{t = 1}^T(h_{x_t}^2 - h_{x_t})(\langle x_t, \widehat{\theta}_t\rangle)^2
\\
&\qquad+ \sum\limits_{t = 1}^Th_{x_t}y_t^2.
\end{align*}
This is the exact expression of regret. We can further simplify it. First, we have 
\[
\sum\nolimits_{t = 1}^T(h_{x_t}^2 - h_{x_t})(\langle x_t, \widehat{\theta}_t\rangle)^2 \le 0,
\]
since $h_{x_t} \in [0, 1]$. It is only left to bound $\sum\nolimits_{t = 1}^Th_{x_t}y_t^2$, which, using \eqref{eq:logdetsum} and $x \le \log(1/(1 - x))$ for $x \in (0, 1)$, can be done as follows:
\begin{equation}
\label{eq:sumoflev}
  \sum\limits_{t = 1}^Th_{x_t}y_t^2 \le m^2\sum\limits_{t = 1}^T\log\left(\frac{\det\left(\sum\nolimits_{i = 1}^{t}x_ix_i^{\top} + \lambda I_d\right)}{\det\left(\sum\nolimits_{i = 1}^{t - 1}x_ix_i^{\top} + \lambda I_d\right)}\right) \le m^2\log\left(\det\left(I_d + \frac{1}{\lambda}\sum\nolimits_{t = 1}^Tx_tx_t^{\top}\right)\right).   
\end{equation}
The claim follows.
\end{proof}

In the following example, we demonstrate that \cref{prop:identityforlinearloss} implies the bound of \cite{kivinen1999averaging} (see also \citep{forster1999relative}) for clipped ridge regression. This bound is also revisited in \citep[Theorem 4]{vovk2001competitive} using exp-concavity arguments applied to the squared loss. However, there is an issue with that proof. The exp-concavity argument in \citep[Theorem 4]{vovk2001competitive} necessitates that both $|y_t| \le m$ and $|\langle x_t, \theta\rangle| \le m$ for $\theta$-s within the considered range. The second assumption is not mentioned in the statement of Theorem 4 in \citep{vovk2001competitive}, and indeed, it can be bypassed. Unlike the previous analysis, our approach does not rely on exp-concavity and, therefore, successfully recovers the claimed result without additional boundedness assumptions. The proof below was also presented in \citep[Corollary 1]{zhdanov2009competing}, and we include it here primarily for the reader's convenience.

\begin{corollary}[The regret of the clipped ridge predictor of \citet{kivinen1999averaging}]
\label{cor:clipped}
    In the setup of \cref{cor:vaw}, we have
    \begin{align*}
&\sum\limits_{t = 1}^T\left(y_t - \clip{\langle x_t, \widehat{\theta}_{t}\rangle}\right)^2 
\\
&\qquad\le \inf\limits_{\theta \in \mathbb{R}^d}\left\{\sum_{t = 1}^{T} (y_t - \langle x_t, \theta \rangle)^2 + \lambda \|\theta\|^2\right\} + 4m^2\log\left(\det\left(I_d + \frac{1}{\lambda}\sum\limits_{t = 1}^Tx_tx_t^{\top}\right)\right).
    \end{align*}
    where $\widehat{\theta}_t$ is the ridge regression estimator given by \eqref{eq:ridge}.
\end{corollary}
\begin{proof}
Using \cref{prop:identityforlinearloss}, the definition of the clipping operator and \eqref{eq:sumoflev}, we have
    \begin{align*}
&\sum\limits_{t = 1}^T\left(y_t - \clip{\langle x_t, \widehat{\theta}_t\rangle}\right)^2 - \inf\limits_{\theta \in \mathbb{R}^d}\left\{\sum_{t = 1}^{T} (y_t - \langle x_t, \theta \rangle)^2 + \lambda \|\theta\|^2\right\}
\\
&= \sum\limits_{t = 1}^T\left(y_t - \clip{\langle x_t, \widehat{\theta}_t\rangle}\right)^2 - \sum\limits_{t = 1}^T(1 - h_{x_t})\left(y_t - {\langle x_t, \widehat{\theta}_t\rangle}\right)^2
\\
&\le\sum\limits_{t = 1}^T\left(y_t - \clip{\langle x_t, \widehat{\theta}_t\rangle}\right)^2 - \sum\limits_{t = 1}^T(1 - h_{x_t})\left(y_t - \clip{\langle x_t, \widehat{\theta}_t\rangle}\right)^2
\\
&=\sum\limits_{t = 1}^Th_{x_t}\left(y_t - \clip{\langle x_t, \widehat{\theta}_t\rangle}\right)^2 \le 4m^2\sum\limits_{t = 1}^Th_{x_t} \le 4m^2\log\left(\det\left(I_d + \frac{1}{\lambda}\sum\nolimits_{t = 1}^Tx_tx_t^{\top}\right)\right).
\end{align*}
The claim follows.
\end{proof}

Finally, to recover \cref{thm:gaillard2019}, originally due to \citet*{gaillard2019uniform}, using \cref{prop:identityforlinearloss}. We exactly repeat the lines of the proof of \cref{cor:vaw} with the only change. Now, the matrix $A$ in \cref{prop:identityforlinearloss} is not equal to $\lambda I_d$, but is equal to $\lambda \sum\nolimits_{t = 1}^Tx_tx_t^{\top}$, which is assumed to be invertible. Finally, using an analog of \eqref{eq:sumoflev}, the total loss of the predictor is bounded by
\begin{align*}
&\inf\limits_{\theta \in \mathbb{R}^d}\left(\sum_{t = 1}^{T} (y_t - \langle x_t, \theta \rangle)^2 + \lambda\sum\limits_{t = 1}^T (\langle x_t , \theta\rangle)^2\right) + m^2\log\left(\frac{\det\left(\lambda \sum\nolimits_{t = 1}^Tx_tx_t^{\top}+ \sum\nolimits_{t = 1}^Tx_tx_t^{\top}\right)}{\det\left(\lambda \sum\nolimits_{t = 1}^Tx_tx_t^{\top}\right)}\right)
\\
&=\inf\limits_{\theta \in \mathbb{R}^d}\left\{\sum_{t = 1}^{T} (y_t - \langle x_t, \theta \rangle)^2 + \lambda\sum\limits_{t = 1}^T (\langle x_t , \theta\rangle)^2\right\} + dm^2\log\left(1 + \frac{1}{\lambda}\right)
\\
&\le \inf\limits_{\theta \in \mathbb{R}^d}\left\{\sum_{t = 1}^{T} (y_t - \langle x_t, \theta \rangle)^2 \right\}+ \lambda T m^2 + dm^2\log\left(1 + \frac{1}{\lambda}\right),
\end{align*}
where the last line follows from exactly the same computation as in the proof of Theorem 7 in \citep{gaillard2019uniform}. 

Similarly, repeating these lines but instead using the matrix $A = \lambda\sum\nolimits_{t = 1}^Tx_tx_t^{\top}$ in the proof of \cref{cor:clipped}, we obtain the following result.
\begin{corollary}
    \label{cor:fixedclipped}
    Fix $\lambda > 0$ and consider a sequence of estimators 
\[
\widetilde{\theta}_{t} = \arg\min_{\theta \in \mathbb{R}^d} \left\{\sum_{i = 1}^{t - 1} (y_i - \langle x_i, \theta \rangle)^2 + \lambda\sum\limits_{i = 1}^T(\langle x_i, \theta\rangle)^2\right\}.
\]
Assume that $|y_t| \le m$ for $t = 1, \ldots, T$. Then, the following regret bound holds
\[
\sum\limits_{t = 1}^T\left(y_t - \clip{\langle x_t, \widetilde{\theta}_{t}\rangle}\right)^2 \le \inf \limits_{\theta \in \mathbb{R}^d}\left\{\sum_{t = 1}^{T} (y_t - \langle x_t, \theta \rangle)^2\right\} + \lambda T m^2 + 4dm^2\log\left(1 + \frac{1}{\lambda}\right).
\]
\end{corollary}

\section{Achieving Similar Regret Bounds Using the $\varepsilon$-nets}
\label{app:covering}

We note that an approach based on $\veps$-nets typically works for the problems we consider. The idea behind this approach is to build a covering for the corresponding classes using the information on $x_{1}, \ldots, x_T$, and then run the exponential weights algorithm over the cover. This can be used to obtain the realizable bounds for the binary loss, as implied by the inequality \eqref{eq:perceptron} in \cref{cor:inductivehinge} when reduced to the binary loss. Indeed, one can discretize the space of linear predictors at the level $\gamma$ and then obtain an $O(d\log(1/\gamma))$-type regret bound using the Halving algorithm. However, such a procedure is impractical even in the realizable case, as the discretization would require us to aggregate over $O\left(1/\gamma^d\right)$ experts. Similarly, for linear regression, one can construct an $\varepsilon$-net for the clipped class of linear functions as in \citep[Proof of Theorem 3]{mourtada2022distribution} and aggregate over this net using the exponential weights algorithm with the uniform prior on this $\varepsilon$-net. Similar $\varepsilon$-net based constructions are also popular in the analysis of density estimators and usually lead to the information-theoretically optimal regret bound, but do not result in implementable algorithms; see \citep*{yang1999information, bilodeau2023minimax} for related results. We quote \citet{vovk2006metric} for a general comment on the $\varepsilon$-net based methods: \say{Another disadvantage of the metric entropy method is that it is not clear how to implement it efficiently, whereas many other methods are computationally very efficient. Therefore, the results obtained by this method are only a first step, and we should be looking for other prediction strategies, both computationally more efficient and having better performance guarantees.} \looseness=-1

Despite that, for completeness, in what follows, we choose to present the $\veps$-nets construction for the logistic 
case for assumption-free regret bounds. In our construction of $\veps$-nets, we first clip the logistic function to obtain bounded functions. That is, for any $\theta, x\in \bR^d$ and $y\in \set{\pm 1}$, we define the clipped logistic function as
\begin{align}
    \label{eq:clipped-logistic}
    f_\theta(x,y) = - \log \prn*{ \frac{1}{2T} + \prn*{1- \frac{1}{T}} \sigma(y\inner{x,\theta})  }   =- \log \prn*{ \frac{1}{2T} + \prn*{1- \frac{1}{T}} \frac{1}{1+e^{-y\inner{x,\theta}}}  }.
\end{align}
Our goal is to cover the function class $\cF_X=\set{f_\theta(\cdot,\cdot):X\times \set{\pm 1}\to [0,1] \mid \theta\in \bR^d}$ in the $L^\infty$ norm, which is defined as $\norm{f}_\infty = \sup_{x,y} |f(x,y)|$. Note that working with a clipped class for the log loss causes only a small overhead \citep[Lemma 9.5]{cesa2006prediction} (see also \cite{bilodeau2023minimax}). Specifically, running the exponential weights algorithm with a uniform prior over the $\varepsilon$-net for a properly chosen scale $\varepsilon$ provides an $O(d\log(T))$ regret bound.

\begin{lemma}
    \label{lem:lipschitzness-of-logistic}
    For any $u_1,u_2\in [0,1]$, we have
    \begin{align*}
        \abs*{ \log \prn*{ \frac{1}{2T} + \prn*{1- \frac{1}{T}} u_1 } - \log \prn*{ \frac{1}{2T} + \prn*{1- \frac{1}{T}} u_2 }} \leq 2T\abs*{u_1-u_2}.
    \end{align*}
\end{lemma}

\begin{proof}[\pfref{lem:lipschitzness-of-logistic}]
    Let $h(u) = - \log (1/(2T)+(1-1/T)u)$ for $u\in [0,1]$. Then
    \begin{align*}
        |h'(u)| = \frac{(1-1/T)}{1/(2T)+(1-1/T)u} \leq 2T.
    \end{align*}
    The claim follows.
\end{proof}

\begin{lemma}[Lemma 6 of \citet{drmota2024unbounded}]
\label{lem:drmota}
    For any covariate set $X=\set{x_t}_{t\in [T]}$, there exists an $\veps$-covering $\cG_{X,\veps}$ of the function class $\cG_X=\set{\sigma(\inner{\cdot,\theta}):X\to [0,1] \mid{}\theta\in \bR^d}$ in $L^\infty$ norm with at most $O((T/\veps)^d)$ elements.
\end{lemma}

One may alternatively prove Lemma \ref{lem:drmota} using the bound on the covering numbers of VC-subgraph classes (see \citet[Theorem 9.4]{gyorfi2002distribution}).

\begin{lemma}
    \label{lem:covering-for-transductive-logisitc}
    For any covariate set $X=\set{x_t}_{t\in [T]}$, there exists an $\veps$-covering $\cF_{X,\veps}$ of the function class $\cF_X=\set{f_\theta(\cdot,\cdot):X\times \set{\pm 1}\to [0,1] \mid{}\theta\in \bR^d}$ (see \eqref{eq:clipped-logistic}) in $L^\infty$ norm with at most $O((2T/\veps)^{2d})$ elements.
\end{lemma}

\begin{proof}[\pfref{lem:covering-for-transductive-logisitc}]
    For any $\veps>0$, let $\veps' = \veps/(2T)$. Let $\cG_{X,\veps'}$ be an $\veps'$-covering of the function class $\cG_X$ in $L^\infty$ norm with at most $O((2T/\veps)^{2d})$ elements. Then consider 
    \begin{align*}
        \cF_{X,\veps} = \set*{f_{\thetatil}(\cdot, \cdot)   \mid{} \sigma(\inner{\cdot,\thetatil}) \in \cG_{X,\veps} }.
    \end{align*}
    Then, for any $f_\theta\in \cF_X$, there exists $\sigma(\inner{\cdot,\thetatil}) \in \cG_{X,\veps}$ such that $\norm{ \sigma(\inner{\cdot,\theta})-  \sigma(\inner{\cdot,\thetatil}) }_\infty \leq \veps$. 
    By \cref{lem:lipschitzness-of-logistic}, we note that for any $x$,
    \begin{align*}
        \abs*{ f_\theta(x,1) - f_{\thetatil}(x,1) } \leq 2T \abs*{\sigma(\inner{x,\theta}) - \sigma(\inner{x,\thetatil}) } \leq 2T\veps' = \veps
    \end{align*}
    and 
    \begin{align*}
        \abs*{ f_\theta(x,-1) - f_{\thetatil}(x,-1) } \leq 2T \abs*{\sigma(-\inner{x,\theta}) - \sigma(-\inner{x,\thetatil}) } = 2T \abs*{\sigma(\inner{x,\theta}) - \sigma(\inner{x,\thetatil}) }   \leq   \veps,
    \end{align*}
    where the first equality is due to $\sigma(-u) = 1- \sigma(u)$ for any $u\in \bR$. This implies that 
    \[
        \norm*{f_\theta(\cdot,\cdot) - f_{\thetatil}(\cdot,\cdot)}_\infty = \sup_{x,y} \abs*{ f_\theta(x,y) - f_{\thetatil}(x,y) } \leq \veps.
    \]
    Thus, $\cF_{X,\veps}$ is an $\veps$-covering of $\cF_X$. The claim follows.
\end{proof}

\end{document}